\documentclass{article}

\usepackage[letterpaper,top=2cm,bottom=2cm,left=3cm,right=3cm,marginparwidth=1.75cm]{geometry}
\usepackage{natbib}
\usepackage{parskip}
\usepackage{subfigure}
\usepackage[utf8]{inputenc} % allow utf-8 input
\usepackage[T1]{fontenc}    % use 8-bit T1 fonts
\usepackage[colorlinks]{hyperref}       % hyperlinks
\usepackage{url}            % simple URL typesetting
\usepackage{booktabs}       % professional-quality tables
\usepackage{amsfonts}       % blackboard math symbols
\usepackage{nicefrac}       % compact symbols for 1/2, etc.
\usepackage{microtype}      % microtypography
\usepackage[dvipsnames, svgnames, x11names]{xcolor}         % colors

\hypersetup{
    citecolor = DeepSkyBlue3
}

\usepackage{booktabs}
\usepackage{multirow}
\usepackage{algorithm,algorithmic}
\usepackage{times}

\usepackage{amsmath,amsthm,amsfonts,bm}
\usepackage{hyperref}
\usepackage{url}
\usepackage[capitalize,noabbrev]{cleveref}
\usepackage{apptools}
\usepackage{enumitem}
\usepackage{graphicx,subfigure}
\usepackage{wrapfig,caption}
\usepackage{pifont}
\usepackage[T1]{fontenc}
\usepackage{fancybox,colortbl}
\usepackage{tablefootnote}

\AtAppendix{\counterwithin{lem}{section}}
\AtAppendix{%
  \crefalias{section}{appendix}
  \crefalias{subsection}{appendix}
  \crefalias{subsubsection}{appendix}
}

\newtheorem{thm}{Theorem}
\newtheorem{prop}{Proposition}
\newtheorem{lem}{Lemma}
\newtheorem{cor}[thm]{Corollary}
\newtheorem{ass}{Assumption}
\newtheorem{defn}{Definition}
\newtheorem{remark}{Remark}

\def \y {\mathbf{y}}

\def \E {\mathbb{E}}
\def \x {\mathbf{x}}

\def \v {\mathbf{v}}

\def \g {\mathbf{g}}

\def \O {\mathcal{O}}

\def \z {\mathbf{z}}

\def \u {\mathbf{u}}

\def \w {\mathbf{w}}

\def \r {\mathbf{r}}
\def \R {\mathbb{R}}

\def \uh {\widehat{\u}}
\def \P {\mathcal{P}}

\def \F {\mathcal{F}}

\def \A {\mathcal{A}}
\def \Ab {\mathbf{A}}

\let\oldc\c
\renewcommand{\c}{\ifmmode\mathbf{c}\else\expandafter\oldc\fi}

\def \diff {\mathrm{d}}

\def \xt {\widetilde{\x}}

\def \SS {\mathcal{S}}

\def \Gb {\mathbf{G}}

\def \Lb {\mathbf{L}}

\def \m {\mathbf{m}}

\def \xb {\overline{\x}}

\def \wt {\widetilde{\w}}

\def \n {\mathbf{n}}

\def \wb {\overline{\w}}

\def \e {\mathbf{e}}
\def \EB {\mathbf{E}}

\def \X {\mathcal{X}}

\def \N {\mathbb{N}}

\newcommand{\inner}[2]{\left\langle #1, #2 \right\rangle}

\newcommand \indicator[1]{\mathbb{I}[#1]}

\DeclareMathOperator*{\argmin}{argmin}
\DeclareMathOperator*{\argmax}{argmax}
\DeclareMathOperator*{\sgn}{sign}

\DeclareMathOperator*{\unif}{Unif}
\DeclareMathOperator*{\expdist}{Exp}

\DeclareMathOperator*{\supp}{supp}

\newcommand{\norm}[1]{\left\Vert#1\right\Vert}

\newcommand{\abs}[1]{\left|#1\right|}

\newcommand{\diag}[1]{\textrm{diag}\left(#1\right)}
\newcommand{\sign}[1]{\sgn\left(#1\right)}

\def \Pr    {\mathbb{P}}
\def \eps {\bm{\epsilon}}
\def \bDelta {\bm{\Delta}}
\def \bsigma {\bm{\sigma}}

\newcommand{\sqbrac}[1]{\left[#1\right]}
\newcommand{\brac}[1]{\left(#1\right)}
\newcommand{\cbrac}[1]{\left\{#1\right\}}

\crefname{ass}{Assumption}{Assumptions}
\crefname{defn}{Definition}{Definitions}
\crefname{lem}{Lemma}{Lemmas}
\crefname{thm}{Theorem}{Theorems}
\crefname{appendix}{Appendix}{Appendices}
\crefname{subappendix}{Appendix}{Appendices}
\crefname{subsubappendix}{Appendix}{Appendices}
\crefname{subsubsubappendix}{Appendix}{Appendices}
\def\ceil#1{\left\lceil #1 \right\rceil}

\def\1{\bm{1}}
\def\0{\bm{0}}

\usepackage{etoolbox}

\allowdisplaybreaks[4]
\title{\textbf{StoSignSGD: Unbiased Structural Stochasticity Fixes SignSGD for Training Large Language Models}}

\author{
  Dingzhi Yu\textsuperscript{1,*} \quad
  Rui Pan\textsuperscript{1,*} \quad
  Yuxing Liu\textsuperscript{1,*} \quad
  Difan Zou\textsuperscript{2} \quad
  Tong Zhang\textsuperscript{1} \\
  \textsuperscript{1}University of Illinois Urbana-Champaign \\
  \textsuperscript{2}The University of Hong Kong \\
  \normalsize\texttt{dingzhiyu01@gmail.com, \{ruip4,yuxing6,tozhang\}@illinois.edu, dzou@cs.hku.hk}\\
  \textsuperscript{*}Equal Contribution
}

\date{}

\begin{document}

\maketitle

\begin{abstract}
    Sign-based optimization algorithms, such as SignSGD, have garnered significant attention for their remarkable performance in distributed learning and training large foundation models. Despite their empirical superiority, SignSGD is known to diverge on non-smooth objectives, which are ubiquitous in modern machine learning due to ReLUs, max-pools, and mixture-of-experts. To overcome this fundamental limitation, we propose \textbf{StoSignSGD}, an algorithm that injects structural stochasticity into the sign operator while maintaining an unbiased update step. In the regime of (online) convex optimization, our theoretical analysis shows that StoSignSGD rigorously resolves the non-convergence issues of SignSGD, achieving a sharp convergence rate matching the lower bound. For the more challenging non-convex non-smooth optimization, we introduce generalized stationary measures that encompass prior definitions, proving that StoSignSGD improves upon the best-known complexity bounds by dimensional factors. Empirically, StoSignSGD is stable and efficient across diverse large language model (LLM) training regimes. In aggressive low-precision pretraining, which spans both FP8 and the far more demanding FP4 regime where AdamW fails catastrophically, StoSignSGD stays stable and consistently performs the best. It attains a \textbf{1.44$\times$ to 2.14$\times$ speedup} over established baselines under FP8. Under 4-bit precision, it improves downstream accuracy on the largest OLMo2-370M model by \textbf{1.13} points over the strongest stable baseline, and this advantage grows as both the model size and the data scale up. When fine-tuning 7B LLMs on mathematical reasoning tasks, StoSignSGD also delivers clear gains over both AdamW and SignSGD. Finally, to explain why it works, we develop a sign conversion framework that turns any general optimizer into its unbiased, sign-based counterpart. Using this framework, we decompose the core components of StoSignSGD and run a comprehensive ablation study to validate our design choices.  
\end{abstract}

\section{Introduction}

The rapid development of large foundation models~\citep{brown2020gpt3,touvron2023llama,touvron2023llama2,achiam2023gpt4,team2023gemini,grattafiori2024llama,guo2025deepseek} relies heavily on efficient and scalable optimization algorithms, in which AdamW~\citep{kingma15adam,loshchilov2019adamw} has been prevailing for years. However, as large language models (LLMs) continue to scale, they are inevitably deployed in distributed training environments where the substantial memory footprint and communication overhead of AdamW become major bottlenecks~\citep{shoeybi2019megatron,rajbhandari2020zero,ren2021zero,tang2021_1bitadam,li2022_1bitlamb}. Meanwhile, low-precision training has become an important industrial paradigm for reducing training cost and improving hardware efficiency~\citep{micikevicius2018mixed,liu2024deepseek,yang2025qwen3}, where modern LLMs are increasingly trained in FP8, FP4, or even INT4 numerical formats~\citep{micikevicius2022fp8,peng2023fp8,xi2023int4,ICLR2025fp8scaling,castro2025quartet_fp4,abecassis2025pretraining_fp4,zeng2026glm,xu2026deepseekv4,kimiteam2026kimik3}. Among these, 4-bit formats such as FP4/NVFP4 sit at the current frontier of low-precision training. Since they are natively supported by the latest accelerators, they roughly halve the memory and bandwidth cost of FP8 and promise the largest end-to-end training speedups so far~\citep{chen2025tetrajetv1,castro2025quartet_fp4,abecassis2025pretraining_fp4,chen2026tetrajetv2}, yet they leave almost no numerical headroom and are the least forgiving. In these highly constrained settings, vanilla AdamW is very vulnerable to numerical underflow and instability, and often diverges during the early stage of training~\citep{wortsman2023stable,lee2024fp8,qiu2026why}. These challenges motivate the search for alternative optimizers that are both communication-efficient and numerically robust. 

Sign-based optimizers offer a compelling direction toward this goal. A representative example is SignSGD~\citep{bernstein2018signsgd}, which updates model parameters according to
$\x_{t+1}=\x_t-\eta_t\sign{\g_t}$. By relying merely on gradient signs, sign-based methods compress each coordinate to 1 bit, making them naturally appealing for distributed learning~\citep{bernstein2019signsgd,liu2024dlion}. Moreover, because they discard gradient magnitudes, they are intrinsically less sensitive to limited numerical precision~\citep{li2023qft}. These advantages have led to promising empirical results in LLM training and related large-scale settings~\citep{chen2023symbolic,zhaoICLR2025deconstructing,yuan25mars,narayan2025fp8unit,schlotthauer2025pre,semenov2025benchmarking}. Despite these practical advantages, a critical gap remains between empirical success and theoretical guarantees: \emph{SignSGD fails to converge even on simple (convex) non-smooth objective functions due to its intrinsically biased compression scheme}~\citep{karimireddy2019error}. Given that modern neural architectures critically rely on non-smooth components such as ReLUs~\citep{nair2010relu}, max-pooling~\citep{goodfellow2016deep}, mixture-of-experts~\citep{shazeer2017moe}, and gating~\citep{JMLR2022switchtransformer}, this lack of convergence poses a severe risk to the broader applicability of SignSGD. Furthermore, the biased nature of the gradient updates intrinsically limits the algorithm's final precision, convergence speed, and overall generalization capabilities~\citep{karimireddy2019error,ajalloeian2020bias,pan2025unbiased}.

To rigorously resolve the non-convergence issue of SignSGD, we propose \textbf{StoSignSGD}, a stochastic sign-based optimization algorithm that injects structural stochasticity into the sign operator so as to preserve an unbiased gradient step in expectation. The key idea is to replace the deterministic sign operator with a stochastic counterpart whose randomness adapts to the optimization trajectory through the tracking of maximum historical gradients. This design allows StoSignSGD to behave like a preconditioned SGD method in expectation while preserving the communication and numerical benefits of sign-based updates. Theoretically, we show that StoSignSGD resolves the divergence issue of SignSGD for convex objectives and enjoys tight convergence guarantees with matching lower bounds. In non-convex settings, we introduce refined stationary measures capable of accommodating diverse problem geometries via arbitrary norms. We further prove that StoSignSGD attains a sharp complexity that does not explicitly depend on the problem dimension $d$, outperforming previous state-of-the-art by a factor of $\sqrt{d}$. Beyond these theoretical advancements, we also demonstrate the empirical effectiveness of StoSignSGD in challenging LLM training regimes. We focus on pretraining LLMs in FP8 and FP4 low-precision formats, where AdamW and Muon fail catastrophically. While sign-based methods remain robust under such low-precision formats, StoSignSGD further improves training efficiency. Under FP8 it reaches the same validation loss with \textbf{30\%--53\%} fewer tokens than competitive baselines, giving a \textbf{1.44$\times$ to 2.14$\times$ speedup}. On the far more aggressive FP4 frontier, StoSignSGD trains stably across OLMo2 models up to \textbf{370M} parameters, where AdamW and Muon diverge to NaN at every scale, and it consistently outperforms the strong Lion baseline on downstream benchmarks. It improves average accuracy by \textbf{0.45}, \textbf{0.44}, and \textbf{1.13} points on the 136M, 267M, and 474M models trained on \textbf{4.19B}, \textbf{9.23B}, and \textbf{20.97B} tokens, respectively. The gain becomes much larger at the biggest scale, more than doubling that of the smaller models, which suggests that the benefit of StoSignSGD's numerical robustness is strongest when the precision budget is tightest and training is scaled up. We also evaluate StoSignSGD on 7B-scale LLM finetuning for mathematical reasoning, where it consistently delivers a non-trivial \textbf{3\%--5\%} accuracy improvement. To understand the source of these gains, we develop a general \emph{sign conversion framework} that transforms a standard optimizer into its unbiased sign-based counterpart. This framework offers a unified way to analyze sign-based optimization beyond StoSignSGD itself. From this view, we isolate the essential ingredients of our method and conduct comprehensive ablation studies to assess the necessity of the proposed structural stochasticity in the sign operator.

Our main contributions can be summarized as follows.
\begin{itemize}[itemsep=0.00em, topsep=0.00em]
    \item We devise a stochastic sign operator with unbiased structural stochasticity and propose StoSignSGD, which provably fixes the non-convergence issue of SignSGD, and tightens best-known complexities in convex and non-convex settings, respectively. 
    \item Extensive LLM pretraining and finetuning experiments validate the practical efficiency of StoSignSGD. Our method delivers a \textbf{1.44$\times$ to 2.14$\times$ speedup} in FP8 pretraining, improves FP4 downstream accuracy over Lion by \textbf{0.45}/\textbf{0.44}/\textbf{1.13} points on OLMo2-70M/150M/370M with the largest gain at the biggest model (while AdamW diverges), and yields a \textbf{3\%--5\%} improvement on math reasoning tasks.
    \item We propose an unbiased sign conversion framework that converts a general optimizer into a sign-based optimizer. Using this framework, we analyze the design of StoSignSGD and run an ablation study to justify the necessity of our proposed sign-based unbiased structural stochasticity.
\end{itemize}
Due to space limitations, we defer the comprehensive literature review to \cref{sec:related-work}.

\section{Our Algorithms and Theoretical Results}

\subsection{Preliminaries} 

\paragraph{Notations} We write $\1_d$ for all-ones vector in $\R^d$, $[T]$ for $\{1,2,\dots,T\}$, $\odot$ for the element-wise product, and $\max\{\x,\y\}$ for the element-wise maximum of two vectors. We denote the $q$-norm ball centered at $\x$ with radius $\delta$ by $B_q(\x,\delta):=\{\y\in\R^d|\norm{\y-\x}_q\le\delta\}$. For $\Gb\in\R_{+}^{d}$, let $\norm{\x}_{\Gb}^{2}:=\x^{\top}\diag{\Gb}\x$ and define the $\Gb$-weighted projection onto the convex set $\X$ as $\Pi_\X^\Gb[\x]=\argmin_{\x^\prime\in\X}\norm{\x^\prime-\x}_\Gb^2$, which is a convex quadratic program efficiently solvable~\citep[\S~4]{hazan2007logarithmic}. Throughout, we assume access to a stochastic first-order oracle $\O$: queried at $\x_t\in\X$, it returns $\g_t$ satisfying $\E_{\O_t}[\g_t]\in\partial f(\x_t)$. For $\cbrac{\g_t}_{t\in[T]}$ where $\g_t=(g_{t,1},\cdots,g_{t,d})^\top$, we define $\g_{1:t,i}=(g_{1,i},\cdots,g_{t,i})^\top\in\R^t$.

\paragraph{Problem setup and assumptions} We investigate the following stochastic non-smooth optimization problem: $\min_{\x\in\R^d} f(\x)$, where $f:\R^d\mapsto\R$ is a possibly non-smooth objective function, which we can only access in a noisy manner. A conspicuous example in machine learning is $f(\x)=\E\sqbrac{\ell(\x;\xi)}$, where $\x$ serves as model weights, $\xi$ represents a mini-batch of samples, and $\ell(\x;\xi)$ is the loss function. For theoretical analysis, we employ the following assumptions.
\begin{ass}[Convexity]\label{ass:convexity}
    The objective function $f$ is convex. For convex settings, we consider solutions within a non-empty closed convex set \( \X \subseteq \R^d \), where at least one optimal solution \( \x_* \in \X \) exists, and the set $\X\subseteq\R^d$ is bounded in the sense that $\sup_{\x,\x^\prime\in\X}\norm{\x-\x^\prime}_\infty\le D_\infty$. 
\end{ass}

\begin{ass}[Non-convexity]\label{ass:non-convexity}
    The objective function $f$ is non-convex, and bounded from below: $f^*:=\inf_{\x\in\R^d} f(\x)>-\infty$.
\end{ass}
\begin{ass}[Coordinate-wise Lipschitzness]\label{ass:coord-lipschitz}
    There exists a non-negative vector $\Lb=(L_1,\cdots,\allowbreak L_d)^\top\in\R^d_+$ such that $\E\sqbrac{\norm{\g_{1:T,i}}_\infty^2}\le L_i^2$ holds for all $i\in[d]$. 
\end{ass}

\begin{ass}[Well-behavedness]\label{ass:well-behavedness}
    $f$ is differentiable and for all $\x^\prime,\x\in\X$, it holds that $f(\x^\prime)-f(\x)=\int_0^1\inner{\nabla f(\x+t(\x^\prime-\x))}{\x^\prime-\x}\diff{t}$.

\end{ass}

\cref{ass:convexity,ass:non-convexity} are standard for convex and non-convex objectives~\citep{duchi2011adaptive,arjevani2023lower,liu2025adagrad}. \cref{ass:coord-lipschitz} is introduced to bound $\norm{\g_{1:T,i}}_\infty$ and is indispensable for our non-smooth analysis. While slightly stronger than the coordinate-wise second-moment bound $\E[g_{t,i}^2]\le L_i^2$ used in~\citet{cutkosky2023optimal}, it is strictly weaker than the almost-sure uniform gradient bound imposed in~\citet{chzhen2023signsvrg,sun2023momentum,NeurIPS:2024:Jiang,jiang2025improved}. As we will see in \cref{thm:oco-bound}, \cref{ass:coord-lipschitz} is also necessary to ensure that StoSignSGD is a no-regret learning algorithm. \cref{ass:well-behavedness} serves as a mild regularity condition widely adopted for non-convex, non-smooth problems~\citep{cutkosky2023optimal,zhang2024random,ICML:2024:Liu,ahn2024adamema,ahn2025general,jiangSTOC2025improved,liu2026old}.

\subsection{Unbiased Stochastic Sign Operator}

We now formally introduce the coordinate-wise stochastic sign operator.
\begin{defn}[Stochastic Sign Operator]\label{def:stochastic-sign}
    For any $\x=(x_1,\cdots,x_d)^\top\in\R^d$ and $\Gb=(G_1,\cdots,\allowbreak G_d)^\top\in\R^d_+$ satisfying $\abs{x_i}\le G_i$, define
    \begin{align*}
        \SS_\Gb(\x)=\sgn\brac{\x+\Gb\odot\n},\textnormal{ where }\n\sim\unif\brac{[-1,1]^d}.
    \end{align*}
\end{defn}

We inject coordinate-wise uniformly distributed noise into the sign operator. This deliberate corruption is key to preventing the divergence observed in vanilla SignSGD.

\begin{prop}\label{prop:unbiased-sign}
    For any $\Gb\succeq\x$, it holds that $\E_\SS\sqbrac{\SS_\Gb(\x)}=\x/\Gb=(x_1/G_1,\cdots,x_d/G_d)$.
\end{prop}

\begin{proof}
    Observe that $\Pr\brac{\SS_{G_i}(x_i)=\pm1}=\frac{G_i\pm x_i}{2G_i},\forall i\in[d]$. Taking expectations yields the result.
\end{proof}

\cref{prop:unbiased-sign} reveals that $\SS_\Gb(\cdot)$ secretly retains anisotropic gradient magnitudes even though only sign information is transmitted. Replacing $\sgn(\cdot)$ with $\SS_\Gb(\cdot)$ makes the expected update coincide with preconditioned SGD. Because the operator is \emph{unbiased only after coordinate-wise rescaling}, the choice of $\Gb$, which sets the per-coordinate normalization level, is decisive. As elaborated in later sections, our $\Gb$ adapts coordinate-wise along the optimization trajectory, yielding theoretical guarantees that contain no explicit dimensional factors.

\begin{algorithm}[t]
  \caption{Online Stochastic Sign (Sub)Gradient Descent}
  \label{alg:online-stosignsgd}
  \begin{algorithmic}[1]
    \STATE {\bfseries Input:} Non-empty closed convex set $\X\subseteq\R^d$, initialization $\x_1\in\X$, $\Gb_0=\0\in\R^d$.
    \FOR{$t=1$ {\bfseries to} $T$}
    \STATE Play $\x_t\in\X$
    \STATE Receive $f_t(\x_t)$ for an $f_t$ subdifferentiable in $\X$
    \STATE Get $\g_t\in\partial f_t(\x_t)$
    \STATE Update $\Gb_t=\max\cbrac{\Gb_{t-1},\abs{\g_t}}$
    \STATE Sample Uniform noise $\n_t\sim\unif\brac{[-1,1]^d}$
    \STATE Compute $\x_{t+1}=\Pi_\X^{\Gb_t}\sqbrac{\x_{t}-\eta_t\sign{\g_t+\Gb_t\odot\n_t}}$
    \ENDFOR
  \end{algorithmic}
\end{algorithm}

\subsection{Convex Theory\label{sec:convex-theory}}

In this section, we delve into online convex optimization (OCO), which is the canonical framework for non-smooth convex optimization. This setting covers stochastic convex optimization in \cref{sec:sco} and also serves as the basis for the non-convex regime in \cref{sec:non-convex-theory}.

Here, we briefly introduce the online learning setup, leaving a comprehensive overview to~\citep{cesa2006prediction,orabona2019intro,hazan2019introduction}. In each round $t=1,\dots,T$, the learner plays $\x_t\in\X$, after which an adversary reveals a convex loss $f_t:\X\to\R$. The learner then incurs loss $f_t(\x_t)$ and observes a subgradient $\g_t\in\partial f_t(\x_t)$. The performance of the learner is measured by the regret against any comparator $\x\in\X$: 
\begin{align*}
\textnormal{Regret}_T(\x):=\sum_{t=1}^T\brac{f_t(\x_t)-f_t(\x)}.
\end{align*}
\noindent Our goal is to design a sign-based no-regret learning algorithm that achieves sublinear regret for non-smooth losses $\cbrac{f_t}_{t\in[T]}$.

We formally introduce the online StoSignSGD in \cref{alg:online-stosignsgd}. At each round, the algorithm updates
\begin{align}\label{eq:online-stosignsgd}
    \Gb_t=\max\cbrac{\Gb_{t-1},\abs{\g_t}},\qquad
\x_{t+1}=\Pi_\X^{\Gb_t}\sqbrac{\x_t-\eta_t\cdot \SS_{\Gb_t}(\g_t)},
\end{align}
where $\Gb_t=(G_{t,1},\dots,G_{t,d})\in\R^d$ tracks the coordinate-wise maximum magnitude of past subgradients. Compared with vanilla SignSGD, the only change is that the deterministic sign operator $\sign{\cdot}$ is replaced by the stochastic sign operator $\SS_{\Gb_t}(\cdot)$, together with a generalized projection under the geometry induced by $\Gb_t$. For example, when $\X=[-D_\infty,D_\infty]^d$, the update in~\eqref{eq:online-stosignsgd} reduces to the coordinate-wise form $x_{t+1,i}=\Pi_{[-D_\infty,D_\infty]}\sqbrac{x_{t,i}-\eta_t\SS_{G_{t,i}}(g_{t,i})}$. Thus, StoSignSGD preserves the simplicity of sign-based updates while introducing a carefully structured stochasticity.

We now state the main regret guarantee, whose proof is postponed to \cref{app:sec:proof-online-upper-bound}.

\begin{thm}\label{thm:oco-bound}
    Under \cref{ass:convexity,ass:coord-lipschitz}, set $\eta_t=\frac{D_\infty}{\sqrt{2t}}$. \cref{alg:online-stosignsgd} ensures $\E\sqbrac{\max_{\x\in\X}\textnormal{Regret}_T(\x)}\le(2+\sqrt{2})D_\infty\norm{\Lb}_1\sqrt{T}.$

\end{thm}

\cref{thm:oco-bound} indicates that \cref{alg:online-stosignsgd} is a no-regret online learning algorithm for non-smooth convex objectives. The $\sqrt{T}$ order in the time horizon $T$ is minimax optimal up to constant factors~\citep{abernethy2008optimal}. Moreover, the bound depends on $\norm{\Lb}_1$ rather than explicitly on the ambient dimension $d$, which is particularly desirable in high-dimensional applications such as large-scale model training.

\paragraph{On the design of $\Gb_t$}
The choice $\Gb_t=\max\cbrac{\Gb_{t-1},\abs{\g_t}}$ plays a central role for two reasons. First, it is admissible by \cref{def:stochastic-sign}. Second, it guarantees the monotonicity property $\Gb_t\succeq \Gb_{t-1}$, which is crucial for the telescoping argument in the proof. This design mirrors the core idea behind adaptive methods such as AdaGrad~\citep{duchi2011adaptive} and AMSGrad~\citep{reddi2018convergence}, but here it is used to control the stochastic sign operator rather than to rescale full gradients. The resulting coordinate-wise adaptivity is what allows our guarantee to avoid the explicit dimensional dependence that appears in prior analyses of non-smooth sign-based methods~\citep{jin2020stochastic,safaryan2021stochastic,chzhen2023signsvrg,sun2023momentum,NeurIPS:2024:Jiang,jiang2025improved}. Last but least, our design of $\Gb_t$ is already ``optimal'' in the sense that no better convergence can be obtained by utilizing another $\Gb_t$, as supported by \cref{thm:lower-bound,app:sec:proof-convex-upper-bound}. 

\begin{remark}
    Unlike previous non-smooth sign-based optimizers, StoSignSGD does not require prior knowledge of problem parameters apart from $D_\infty$. For example,~\citet{chzhen2023signsvrg} fix the corruption vector $\Gb_t$ to a global constant $G$, which is both unknown in practice and blind to coordinate-wise variation, and their step-size schedule requires the time horizon $T$ in advance. In contrast, our method adapts $\Gb_t$ online in a coordinate-wise manner and uses the anytime schedule $\eta_t\propto t^{-1/2}$. This also resolves the open problem raised in~\citet[Remark~1]{chzhen2023signsvrg}.
\end{remark}

Next, we establish a matching lower bound showing that the above guarantee is tight.

\begin{thm}\label{thm:lower-bound}
    Under \cref{ass:convexity,ass:coord-lipschitz}, there exists a closed convex set $\X\subseteq\R^d$, a sequence of convex losses $\cbrac{f_t}_{t\in[T]}$, and a comparator $\x\in\X$ such that \cref{alg:online-stosignsgd} satisfies $\E\sqbrac{\max_{\x\in\X}\textnormal{Regret}_T(\x)}\ge D_\infty\norm{\Lb}_1\sqrt{T/8}$.

\end{thm}

The lower bound is built from the canonical one-dimensional hard instance for online convex optimization~\citep{orabona2019intro}, and then lifted to $d$ dimensions through the separable construction $f(\x)=\sum_{i=1}^d f_i(x_i)$. Since regret decomposes additively across coordinates for separable losses, the total regret is simply the sum of the one-dimensional contributions. Detailed analysis can be found in \cref{app:sec:proof-lower-bound}.

\begin{remark}\label{remark:lower-bound}
    \cref{thm:lower-bound} shows that StoSignSGD is worst-case optimal for non-smooth convex optimization in the online setting, justifying the tightness of \cref{thm:oco-bound}. More importantly, it demonstrates that our design of $\Gb_t$ has achieved the best possible rate in non-smooth regimes. Altering $\Gb_t$ to other modes, e.g., $G_{t,i}=\norm{\g_{1:t,i}}_2$ according to AdaGrad, offers no performance gains, in general.
\end{remark}

\subsection{Non-convex Theory\label{sec:non-convex-theory}}

\begin{algorithm}[t]
   \caption{Non-smooth Non-convex StoSignSGD}
   \label{alg:nonconvex-stosignsgd}
\begin{algorithmic}[1]
   \STATE {\bfseries Input:} Initialization $\x_0\in\R^d$, $K,N\in\N,D_\infty>0$, $\bDelta_0\in B_\infty(\0,D_\infty)\subseteq\R^d$, $\Gb_0=\0\in\R^d$.\\
   \STATE Set $T=KN$ and $\eta_0=0$.
   \FOR{$t=1$ {\bfseries to} $T$}

   \STATE Compute $\bDelta_t=\Pi_{B_\infty(\0,D_\infty)}^{\Gb_{t-1}}\sqbrac{\Delta_{t-1}-\eta_{t-1}\cdot\SS_{\Gb_{t-1}}(\g_{t-1})}$  
   \STATE Sample $s_t\sim\expdist(1)$
   \STATE Set $\x_t=\x_{t-1}+s_t\bDelta_t$
   \STATE Get stochastic gradient $\g_t$ from oracle $\O$ such that $\E_{\O_t}[\g_t]=\nabla f(\x_t)$
   \STATE Update $\Gb_t=\max\cbrac{\Gb_{t-1},\abs{\g_t}}$
   \ENDFOR
   \STATE Set $\xb^k=\frac{1}{N}\sum_{n=1}^N\x_{(k-1)N+n}$ for all $k\in[K]$.
   \STATE {\bfseries Output:} Return $\xb\sim\unif\brac{\cbrac{\xb^1,\cdots,\xb^K}}$.
\end{algorithmic}
\end{algorithm}

In this section, we instantiate StoSignSGD to solve non-smooth, non-convex stochastic optimization problems. Such problems are ubiquitous in practice, yet their theoretical analysis is notoriously challenging. For the usual $\epsilon$-stationary criterion $\norm{\nabla f(\x)}\le \epsilon$, first-order methods can neither guarantee to reach nor even to approach an $\epsilon$-neighborhood in finite time~\citep{zhang2020complexity,kornowski2022oracle}. Following the alternative tractable stationarity~\citep{zhang2020complexity}, we propose the stationary measure below.

\begin{defn}[$(\delta,\epsilon)$-$\ell_{p,q}$-stationary point]\label{def:stationary-point}
    Given $\epsilon,\delta> 0,1\le p,q\le\infty$ and a differentiable function $f:\R^d\mapsto\R$, $\x$ is a $(\delta,\epsilon)$-$\ell_{p,q}$-stationary point of $f$ if
    \begin{align*}
        \norm{\nabla f(\x)}_{p,q}^{[\delta]}:=\inf_{P\in\P(\R^d),\E_{\y\sim P}[\y]=\x}\cbrac{\norm{\E\sqbrac{\nabla f(\y)}}_p+\delta\E\sqbrac{\norm{\y-\x}_q^2}}\le\epsilon.
    \end{align*}
\end{defn}

Our \cref{def:stationary-point} encompasses previous definitions: covering~\citet[Definition~2.2]{zhang2024random} and~\citet[Definition~1]{ahn2025general} by $p=q=2$; covering~\citet[Definition~15]{ahn2024adamema} by $p=1,q=2$. Through arbitrary $(p,q)$ pairs, we obtain a more refined stationary measure that can adapt to various problem geometries. In the sequel, we focus on the dual pair $p=1,q=\infty$, which aligns naturally with StoSignSGD.

We present the non-convex StoSignSGD in \cref{alg:nonconvex-stosignsgd}. The algorithm instantiates the online-to-non-convex conversion introduced by~\citet{cutkosky2023optimal}. In each round $t$, the update direction $\bDelta_t$ is produced by the online StoSignSGD (\cref{alg:online-stosignsgd}) and is then scaled by an exponentially-distributed random variable $s_t$~\citep{zhang2024random}. This single random scaling removes the bias of first-order Taylor expansion~\citet{zhang2020complexity}, and obviates the auxiliary sequence $\x_{t-1}+\bDelta_t$ used in~\citet{cutkosky2023optimal}. This mechanism succeeds because the online learner is literally predicting the next update direction, precisely the scenario for which online algorithms are designed. Consequently, the optimal regret guarantee of \cref{alg:online-stosignsgd} immediately translates into a sharp complexity bound for non-smooth non-convex optimization. The formal theoretical arguments are stated below, with their proof deferring to \cref{app:sec:exponential}.

\begin{thm}\label{thm:nonconvex-bound}
    Under \cref{ass:coord-lipschitz,ass:non-convexity,ass:well-behavedness} and further assume that $f(\x_0)-f^*\le\Delta_f$. Set
    \begin{align*}
    K=\ceil{\frac{7\sqrt{14}\Delta_f\delta^{1/2}}{2\epsilon^{3/2}}},\ N=\ceil{751\norm{\Lb}_1^2\epsilon^{-2}},\  D_\infty=\frac{\sqrt{\epsilon}}{\sqrt{14\delta}N},\ \eta_t=\frac{\sqrt{2}D_\infty}{\sqrt{((t-1)\bmod N)+1}}.
\end{align*}
Then, \cref{alg:nonconvex-stosignsgd} finds $(\delta,\epsilon)$-$\ell_{1,\infty}$-stationary points within
\begin{align*}
    T=KN\le39326\Delta_f\norm{\Lb}_1^2\delta^{1/2}\epsilon^{-7/2}
\end{align*}
stochastic gradient evaluations.
\end{thm}
The complexity matches the $\ell_{2,2}$ lower bound~\citep[Corollary~5.1]{zhang2024random} in $\Delta_f,\delta,\epsilon$ and improves the best-known $\ell_{1,2}$ upper bound $O\brac{\max\cbrac{\Delta_f\norm{\Lb}_1^2\sqrt{d}\delta^{1/2}\epsilon^{-7/2},\norm{\Lb}_1^3\epsilon^{-3}}}$ for Adam~\citep{ahn2024adamema} by a factor of $\sqrt{d}$.

As indicated by \cref{thm:nonconvex-bound}, \cref{alg:nonconvex-stosignsgd} achieves state-of-the-art complexity for finding stationary points in \cref{def:stationary-point}. Recent studies~\citep{cutkosky2023optimal} also consider stationary points based on Goldstein stationarity~\citep{goldstein1977optimization}. In \cref{sec:goldstein}, we further propose a new stationary measure (\cref{def:stationary-point-as}) encompassing previous works and present a variant of StoSignSGD (\cref{alg:nonconvex-stosignsgd-uniform}) that efficiently identifies it.

\section{Empirical Study\label{sec:experiments}}

\begin{algorithm}[t]
    \caption{Practical Implementation of StoSignSGD}
    \label{alg:practical-stosignsgd}
    \begin{algorithmic}[1]
    \STATE {\bfseries Input:} Start point $\x_1\in\R^d$, momentum $\beta_1\in[0,1)$, learning rate $\cbrac{\eta_t}_{t=1}^T$, weight decay $\lambda\ge0$.\\
    \FOR{$t=1$ {\bfseries to} $T$}
    \STATE Get stochastic gradient $\g_t$
    \STATE Update momentum $\m_t=\beta_1\m_{t-1}+(1-\beta_1)\g_t$\hfill\COMMENT{\textcolor{gray}{$\m_1=\g_1$}}
    \STATE Update max-buffer $\Gb_t=\max\cbrac{\Gb_{t-1},\abs{\m_t}}$\hfill\COMMENT{\textcolor{gray}{$\Gb_1=\abs{\m_1}$}}
    \STATE Sample Uniform noise $\n_t\sim\unif\brac{[-1,1]^d}$
    \STATE Compute $\x_{t+1} = \x_t - \eta_t \sign{\m_t+\Gb_t\odot\n_t}-\eta_t\lambda\x_t$
    \ENDFOR
\end{algorithmic}
\end{algorithm}

\subsection{Practical Efficiency of StoSignSGD}
\begin{wrapfigure}[8]{r}{0.48\textwidth}
    \centering
    \vspace{-25pt}
    \includegraphics[width=0.9\linewidth]{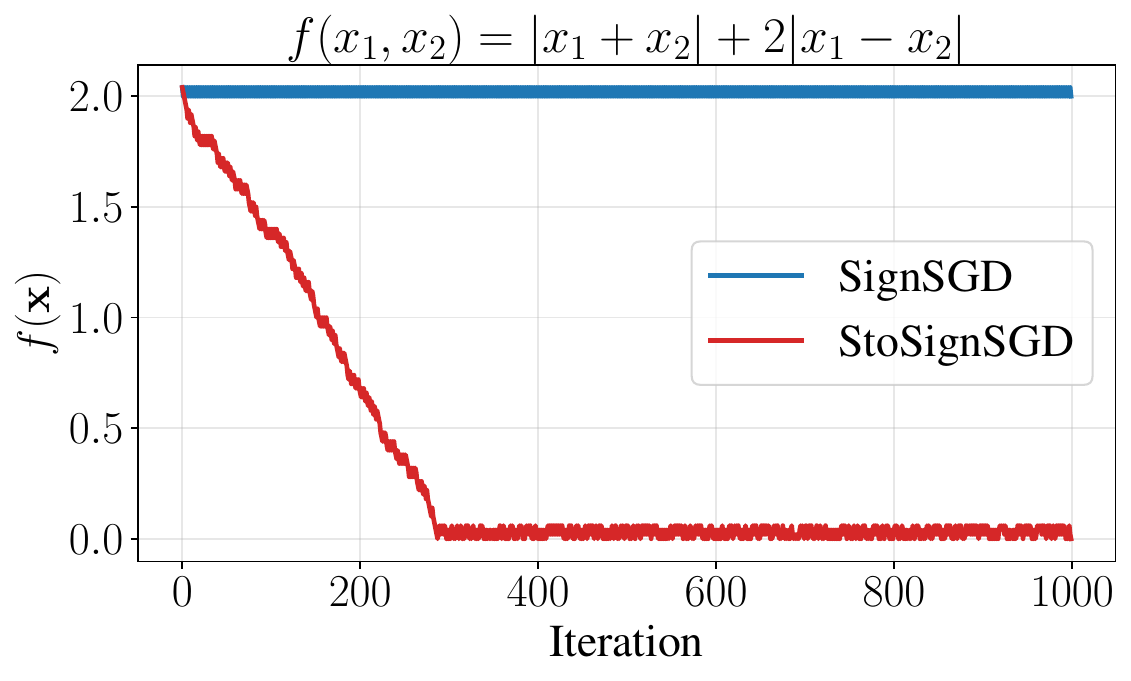}
    \vspace{-8pt}
    \caption{Convex Non-smooth Numerical Example}
    \label{fig:convex}
    \vspace{-8pt}
\end{wrapfigure}
In this section, we examine StoSignSGD (\cref{alg:practical-stosignsgd}), SignSGD (\cref{alg:practical-signsgd}), and AdamW (\cref{alg:adamw}) in convex optimization, and real-world applications of LLM finetuning as well as pretraining.

\subsubsection{Convex Objectives}

When the objective function is non-smooth, SignSGD is known to diverge even without gradient noise~\citep{karimireddy2019error,xiao2023stochastic}. To support our theory in \cref{sec:convex-theory}, we conduct numerical experiments on a convex non-smooth function $f(x_1,x_2)=\abs{x_1+x_2}+2\abs{x_1-x_2}$ with exact gradients (noiseless setting). The distinct convergence behavior in \cref{fig:convex} shows that StoSignSGD effectively fixes the non-convergence issue of SignSGD in the context of non-smooth optimization.

\subsubsection{Mathematical Reasoning\label{sec:math-reasoning}}

To test the empirical effectiveness of StoSignSGD in practice, we conduct supervised finetuning (SFT) of LLMs to enhance their mathematical reasoning ability. Following common practice~\citep{pan-etal-2025-scalebio,pan2025unbiased,xiong2025stepwiser,NeurIPS:2025:Wang:B,NeurIPS:2025:Wang:A}, we finetune Qwen2.5-7B model~\citep{yang2024qwen2.5} on NuminaMath-CoT dataset~\citep{numina_math_datasets}, which consists of 860k Chain-of-Thought (CoT)~\citep{wei2022cot} mannered high school and international mathematics Olympiad competition problems. We randomly sample 20k problem-solution pairs and finetune the model for 1 epoch using \texttt{LMFlow}~\citep{diao2024lmflow} framework. After that, post-trained models are evaluated on the GSM8k~\citep{cobbe2021gsm8k} and MATH~\citep{hendrycks2021math} datasets by \texttt{lm-evaluation-harness}~\citep{eval-harness}. The detailed training and evaluation configurations are postponed to \cref{sec:math-reasoning-cont}. 

\begin{wraptable}{r}{0.5\textwidth}
\vspace{-10pt} 
\caption{Evaluation accuracy (\%) for Qwen2.5-7B\label{tab:numinamath-results}}
\vspace{-5pt} 
\centering
\begin{tabular}{lcc}
\toprule
\textbf{Optimizer} & \textbf{GSM8k} & \textbf{MATH} \\
\midrule
AdamW   & 74.37 \(\pm\) 1.20 & 48.86 \(\pm\) 0.66  \\
SignSGD &  71.95 \(\pm\) 1.24&  47.66 \(\pm\) 0.66   \\
\rowcolor{yellow!20}StoSignSGD & \textbf{77.33 \(\pm\) 1.15} & \textbf{48.88 \(\pm\) 0.65}  \\
\bottomrule
\end{tabular}
\vspace{-8pt} 
\end{wraptable}

\cref{tab:numinamath-results} shows the zero-shot accuracy of Qwen2.5-7B finetuned with AdamW, SignSGD, and StoSignSGD. On the GSM8k benchmark, our optimizer yields a non-trivial $3\%$ improvement over AdamW, underscoring the practical advantages of StoSignSGD for SFT. Meanwhile, StoSignSGD maintains performance comparable to AdamW on the more challenging problems within the MATH dataset. In \cref{sec:math-reasoning-cont}, we present additional empirical evidence, including LoRA~\citep{hu2022lora} experiments, suggesting that StoSignSGD effectively enhances the model's proficiency on foundational mathematical tasks.

We also test our method on other math reasoning benchmarks and incorporate more baselines, where StoSignSGD remains competitive in \cref{tab:gsm8k_mathqa}. Due to space limitations, experiments underlying \cref{tab:gsm8k_mathqa} will be further discussed in detail in \cref{sec:demystify-structural}.

\begin{table}[htbp]
\centering

\caption{Zero-shot accuracy (\%) on the GSM8K and MathQA datasets.\label{tab:gsm8k_mathqa}}
\vspace{-5pt}
\resizebox{\linewidth}{!}{
\begin{tabular}{lcccccc}
\toprule
& \multicolumn{3}{c}{\textbf{GSM8K}} & \multicolumn{3}{c}{\textbf{MathQA}} \\
\cmidrule(lr){2-4} \cmidrule(lr){5-7}
\textbf{Optimizer} & \textbf{Qwen2.5-7B} & \textbf{Llama-3.1-8B} & \textbf{Mistral-7B-v0.1} & \textbf{Qwen2.5-7B} & \textbf{Llama-3.1-8B} & \textbf{Mistral-7B-v0.1} \\
\midrule
SignSGD 
& 72.10$\pm$1.24 & 48.75$\pm$1.38 & 45.03$\pm$1.37 
& 41.07$\pm$0.90 & 38.99$\pm$0.89 & 31.39$\pm$0.85 \\

\textcolor{gray}{AdamW}
& 74.82$\pm$1.20 & 53.30$\pm$1.37 & 48.22$\pm$1.38 
& 40.77$\pm$0.90 & 40.10$\pm$0.90 & 32.60$\pm$0.86 \\

\textcolor{gray}{AdaMax}
& 74.75$\pm$1.20 & 52.84$\pm$1.38 & 49.58$\pm$1.38 
& 43.52$\pm$0.91 & 40.54$\pm$0.90 & 33.47$\pm$0.86 \\

\textcolor{gray}{IE-StoSignSGD}
& 74.67$\pm$1.20 & 54.81$\pm$1.37 & 47.61$\pm$1.38 
& 44.86$\pm$0.91 & 40.54$\pm$0.90 & 33.93$\pm$0.87 \\

\textcolor{blue}{SignAdamW}
& 75.28$\pm$1.19 & 54.97$\pm$1.37 & 49.43$\pm$1.38 
& 45.43$\pm$0.91 & 40.80$\pm$0.90 & 34.07$\pm$0.87 \\

\textcolor{blue}{SignAdaMax}
& 74.82$\pm$1.20 & 53.30$\pm$1.37 & 47.38$\pm$1.38 
& 44.92$\pm$0.91 & 40.67$\pm$0.90 & 34.40$\pm$0.87 \\

\rowcolor{yellow!20}\textcolor{blue}{\textbf{StoSignSGD}}
& \textbf{75.81$\pm$1.18} & \textbf{55.50$\pm$1.37} & \textbf{50.19$\pm$1.38}
& \textbf{46.47$\pm$0.91} & \textbf{41.01$\pm$0.90} & \textbf{35.08$\pm$0.87} \\
\bottomrule
\end{tabular}
}
\vspace{-10pt}
\end{table}

\subsubsection{Low-precision Pretraining in FP8\label{sec:fp8}}
Low-precision numerical formats, such as FP8~\citep{micikevicius2022fp8}, have drawn much attention for scaling up LLM training~\citep{liu2024deepseek,yang2025qwen3,ICLR2025fp8scaling}, since they greatly reduce memory footprint and computation overhead. However, directly applying AdamW is often infeasible~\citep{wortsman2023stable,lee2024fp8,qiu2026why}, and a rich line of work has tried to fix the failure of AdamW in various ways~\citep{perez2023training,peng2023fp8,balancca2024scalify,ICLR2025coat}. Different from previous works, we offer a simple and effective solution by switching the optimizer to StoSignSGD \emph{without any other stabilization techniques}. Following~\citet{qiu2026why}, we pretrain GPT-2~\citep{radford2019language} on the OpenWebText corpus~\citep{Gokaslan2019OpenWeb} under FP8 precision. In addition to SignSGD and AdamW, we include Lion~\citep{chen2023symbolic} and Muon~\citep{jordan2024muon,liu2025muon} as baselines. The complete experimental setup is deferred to \cref{sec:fp8-cont}. We note that our FP8 setting is much more aggressive than common industrial practice such as DeepSeek-V3~\citep{liu2024deepseek}. Beyond low-precision linear operations, we also store both gradients and optimizer states in FP8. This greatly reduces the memory footprint and memory-bandwidth cost of backpropagation and optimizer updates, but it also imposes much harsher numerical constraints. Stable training under this setting therefore gives direct evidence that StoSignSGD can turn its numerical robustness into real systems-level benefits.
\begin{wraptable}{r}{0.48\textwidth}
    \centering
    \caption{FP8 pretraining results for GPT-2.\label{tab:fp8}}
    \vspace{-8pt}
    \resizebox{\linewidth}{!}{%
    \begin{tabular}{lcccc}
        \toprule
        \multirow{2}{*}{\textbf{Optimizer}} & \multicolumn{2}{c}{\textbf{Loss}} & \multicolumn{2}{c}{\textbf{Perplexity}} \\
        \cmidrule(lr){2-3} \cmidrule(lr){4-5}
        & \textbf{Train} & \textbf{Val} & \textbf{Train} & \textbf{Val} \\
        \midrule
        AdamW      & NaN   & NaN   & NaN   & NaN   \\
        Muon       & NaN   & NaN   & NaN   & NaN   \\
        SignSGD    & 3.539 & 3.546 & 34.43 & 34.67 \\
        Lion       & 3.452 & 3.460 & 31.55 & 31.82 \\
        \rowcolor{yellow!20}StoSignSGD & \textbf{3.410} & \textbf{3.418} & \textbf{30.26} & \textbf{30.51} \\
        \bottomrule
    \end{tabular}%
    }
    \vspace{-8pt} 
\end{wraptable}

\begin{figure}[t]

    \centering

    \subfigure[FP8 pretraining loss curves for GPT-2.]{%
        \label{fig:gpt2fp8}

        \includegraphics[width=0.65\linewidth]{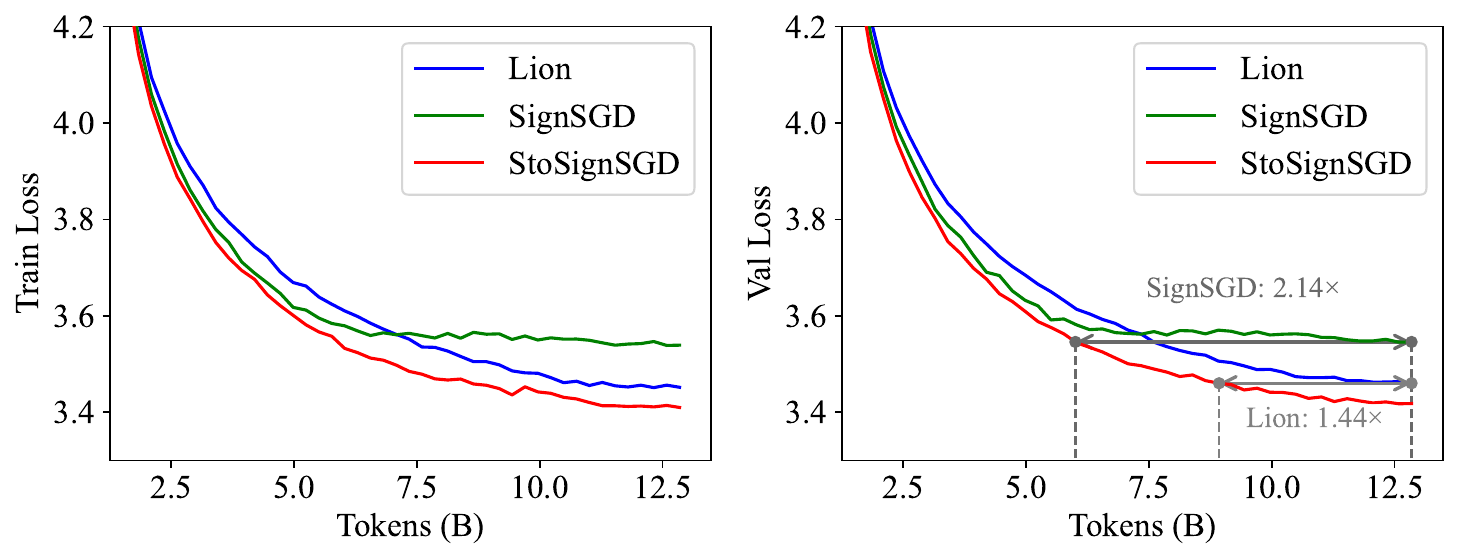}%
    }
    \hfill
    \subfigure[Failure diagnosis of AdamW.]{%
        \label{fig:adamw_fp8_collapse}
        \vspace{-20pt}
        \includegraphics[width=0.33\linewidth]{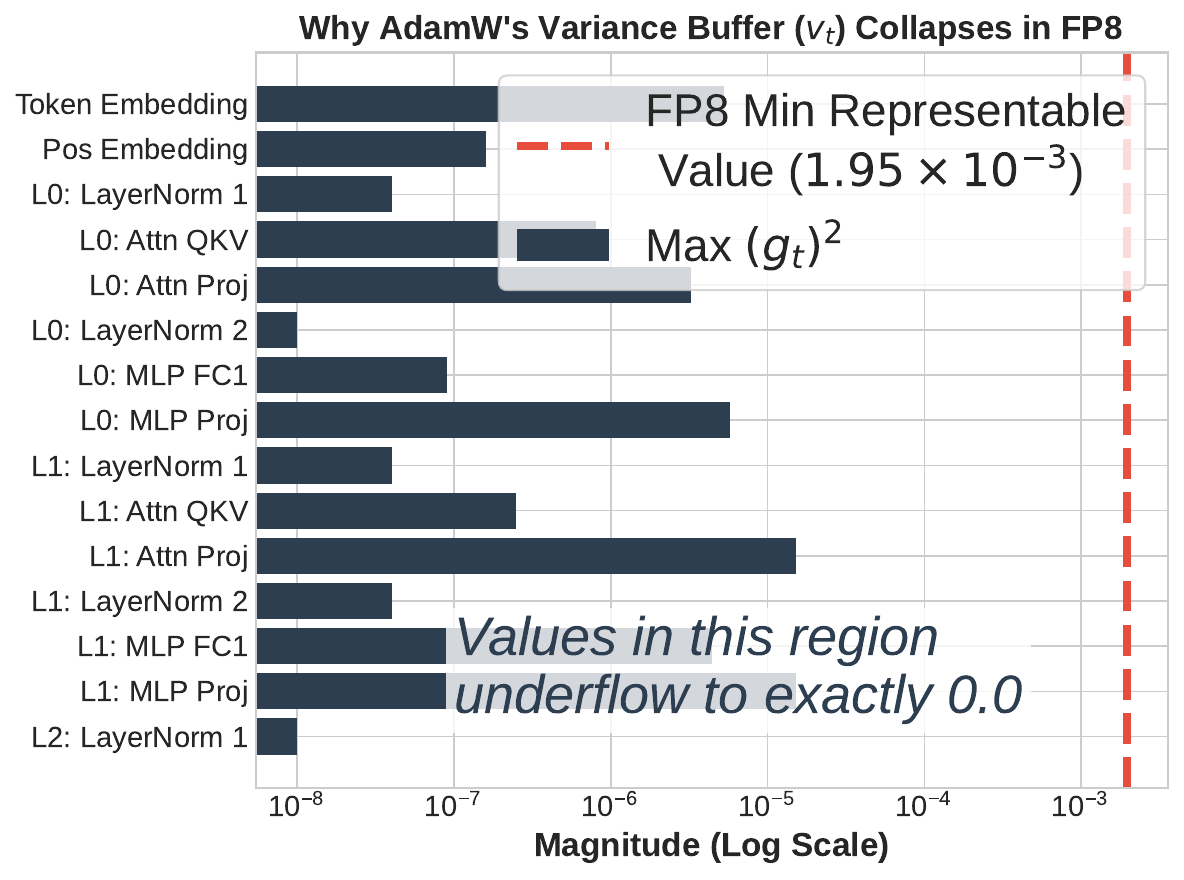}%
    }
    \vspace{-10pt}
    \caption{Quantitative analysis of pretraining GPT-2 on OpenWebText in FP8 precision.\label{fig:fp8_combined_analysis}}

\end{figure}

The loss curves and perplexity scores are presented in \cref{fig:gpt2fp8,tab:fp8}. Our key findings are twofold. \textbf{(i) Both AdamW and Muon diverge during the very early stages of training, whereas sign-based methods train the model stably throughout the entire trajectory}. As shown in \cref{fig:adamw_fp8_collapse}, for AdamW the root cause is the numerical underflow of the variance buffer $\v_t$. Muon fails for a related reason, since it is coupled with Adam and its Newton--Schulz orthogonalization becomes fragile once optimizer states are stored in low precision. In contrast, StoSignSGD uses only the sign information, which naturally fits low-precision formats. A detailed empirical analysis is given in \cref{sec:fp8-cont}. \textbf{(ii) StoSignSGD reaches the target validation loss using 53\% and 30\% fewer tokens than SignSGD and Lion, giving a speedup of 1.44$\times$ and 2.14$\times$, respectively}. Together with its ability to avoid AdamW's divergence, this token efficiency makes StoSignSGD a simple and robust optimizer for low-precision LLM pretraining.

\subsubsection{Low-precision Pretraining in FP4\label{sec:fp4}}

Building on the FP8 study, we further push StoSignSGD to the FP4 frontier, which is more practical for deployment but much less forgiving numerically. Our FP4 training pipeline follows the NVFP4 recipe of~\citet{chen2026tetrajetv2}, and, as in the FP8 regime, we store both gradients and optimizer states in FP4, the most aggressive configuration, while replacing AdamW with StoSignSGD without any stabilization techniques. The downstream evaluation follows common practice~\citep{pan2025unbiased,chen2026tetrajetv2}, and we assess StoSignSGD against Lion, AdamW, and Muon on the OLMo2-70M, -150M, and -370M models, which have \textbf{136M}, \textbf{267M}, and \textbf{474M} parameters in total~\citep{olmo2024olmo2}. Full details and results are deferred to \cref{sec:fp4-cont}. As reported in \cref{tab:fp4,tab:fp4-train}, the FP8 pattern re-emerges at FP4 and becomes even sharper. AdamW and Muon both diverge to NaN at every scale. AdamW fails from variance-buffer underflow (cf. \cref{fig:adamw_fp8_collapse}) and Muon from its Adam coupling and precision-sensitive orthogonalization, whereas StoSignSGD trains stably and consistently beats Lion. This advantage is substantial, and it is largest at the biggest model and the longest run. StoSignSGD improves average downstream accuracy over Lion by \textbf{0.45}, \textbf{0.44}, and \textbf{1.13} points on the 136M, 267M, and 474M models, trained on \textbf{4.19B}, \textbf{9.23B}, and \textbf{20.97B} tokens, respectively, and it also reaches lower training loss and perplexity. The gain at the largest, longest run more than \emph{doubles} that of the smaller scales, which shows that StoSignSGD's numerical robustness matters most when the precision budget is tightest and training is scaled up. In short, StoSignSGD stays stable and performs the best in the most aggressive 4-bit regime, where AdamW is unusable and even Lion is clearly outperformed.

\paragraph{Why does StoSignSGD do well in low precision?}
A natural question is why StoSignSGD is not only stable but also better than Lion, the strongest baseline that survives FP4. We believe one key reason is that \emph{StoSignSGD performs fewer and numerically safer floating-point operations in its state update}. Both methods keep a first-moment buffer $\m_t=\beta_1\m_{t-1}+(1-\beta_1)\g_t$. On top of this, Lion forms a second interpolation $\c_t=\beta_2\m_{t-1}+(1-\beta_2)\g_t$ and then takes its sign, which needs an extra round of multiplications and additions on values that can be tiny in low precision. StoSignSGD instead updates a max-buffer $\Gb_t=\max\{\Gb_{t-1},|\m_t|\}$, so it replaces that add-and-multiply step with a $\max$. The $\max$ is far less sensitive to rounding than repeated fused multiply-adds, because it does not accumulate small errors and never cancels two nearby numbers. Under FP4, where every extra multiply-add loses precision and can underflow, this simpler and more robust update gives StoSignSGD a small but consistent edge over Lion, and the edge grows as the models and the token budget get larger.

\begin{table}[t]
    \centering
    \caption{Radical FP4 pretraining results for the OLMo2 model series.
    All downstream scores are accuracies (\%). Avg.\ is the average over
    the seven reported benchmarks. A dash (--) indicates that the
    corresponding pretraining run encountered \textbf{NaN} loss/perplexity, and
    therefore downstream accuracies are not reported.\label{tab:fp4}}
    \vspace{-5pt}
    \resizebox{\linewidth}{!}{%
    \begin{tabular}{llccccccc>{\columncolor{blue!8}}c}
        \toprule
        \textbf{Model}
        & \textbf{Method}
        & \textbf{ARC-E}
        & \textbf{ARC-C}
        & \textbf{OBQA}
        & \textbf{HellaSwag}
        & \textbf{PIQA}
        & \textbf{SIQA}
        & \textbf{BoolQA}
        & \textbf{Avg.} \\
        \midrule

        \multirow{4}{*}{OLMo2-70M}
        & AdamW
        & -- & -- & -- & -- & -- & -- & --
        & -- \\

        & Muon
        & -- & -- & -- & -- & -- & -- & --
        & -- \\

        & Lion
        & 33.86
        & 21.40
        & 25.40
        & 26.09
        & 55.28
        & 38.79
        & 52.26
        & 36.15 \\

        & \textbf{StoSignSGD}
        & 33.33
        & 20.74
        & 26.40
        & 26.08
        & 55.60
        & 39.25
        & 54.83
        & \textbf{36.60} \\

        \midrule

        \multirow{4}{*}{OLMo2-150M}
        & AdamW
        & -- & -- & -- & -- & -- & -- & --
        & -- \\

        & Muon
        & -- & -- & -- & -- & -- & -- & --
        & -- \\

        & Lion
        & 39.47
        & 19.73
        & 27.60
        & 26.85
        & 59.41
        & 39.66
        & 61.35
        & 39.15 \\

        & \textbf{StoSignSGD}
        & 41.93
        & 21.40
        & 26.60
        & 27.06
        & 58.81
        & 40.69
        & 60.67
        & \textbf{39.59} \\

        \midrule

        \multirow{4}{*}{OLMo2-370M}
        & AdamW
        & -- & -- & -- & -- & -- & -- & --
        & -- \\

        & Muon
        & -- & -- & -- & -- & -- & -- & --
        & -- \\

        & Lion
        & 51.58
        & 26.76
        & 29.00
        & 33.54
        & 63.66
        & 42.84
        & 55.78
        & 43.31 \\

        & \textbf{StoSignSGD}
        & 52.11
        & 25.42
        & 32.20
        & 33.80
        & 64.74
        & 41.81
        & 61.01
        & \textbf{44.44} \\

        \bottomrule
    \end{tabular}%
    }
    \vspace{-10pt}
\end{table}

\subsection{Demystify Structural Stochasticity in Sign\label{sec:demystify-structural}}

\subsubsection{Unbiased Sign Conversion Framework\label{sec:sign-conversion-framework}}

\begin{algorithm}[htbp]
    \caption{Unbiased Sign Conversion Framework}
    \label{alg:sign-conversion}
    \begin{algorithmic}[1]
        \STATE {\bfseries Input:} Initialization $\x_1\in\R^d$, learning rate $\cbrac{\eta_t}_{t=1}^T$, weight decay $\lambda\ge0$.
        \FOR{$t=1$ {\bfseries to} $T$}
        \STATE Get stochastic gradient $\g_t$
        \STATE Update optimizer states $\m_t\leftarrow\m\brac{\m_{t-1},\g_t}, \bsigma_t\leftarrow\bsigma\brac{\v_{t-1},\g_t}$
        \STATE \textcolor{gray}{\itshape \# Choice I: General Optimizer Update}
        \STATE \textcolor{gray}{$\x_{t+1} = \x_t - \eta_t \m_t/\bsigma_t-\eta_t\lambda\x_t$}
        \STATE \textcolor{blue}{\itshape \# Choice II: Sign Conversion Update}
        \STATE \textcolor{blue}{Sample Uniform noise $\n_t\sim\unif\brac{[-1,1]^d}$}
        \STATE \textcolor{blue}{$\x_{t+1} = \x_t - \eta_t \sign{\m_t+\bsigma_t\odot\n_t}-\eta_t\lambda\x_t$}
        \ENDFOR
    \end{algorithmic}
\end{algorithm}

We formally introduce the unbiased sign conversion framework in \cref{alg:sign-conversion}, which systematically converts any general optimizer into an unbiased sign-based counterpart (unbiasedness is established in \cref{prop:unbiased-sign-conversion}). The \textcolor{gray}{gray-colored} text (lines 5--6) represents the update rule of a standard optimizer, such as AdamW or AdaMax (incorporating decoupled weight decay). Conversely, the \textcolor{blue}{blue-colored} text (lines 7--9) denotes the converted sign-based optimizer. This conversion process first samples uniformly distributed noise $\n_t$, and then incorporates structural noise $\bsigma_t\odot\n_t$ into the sign operator to update the model parameters. Within this framework, we dissect the core components of StoSignSGD to elucidate the mechanisms driving its effectiveness.

\begin{table}[t]
    \centering
    \caption{Summary of optimizers and tricks.\label{tab:tricks}}
    \vspace{-5pt}
    \begin{tabular}{@{}lcccc@{}}
        \toprule
        & \textbf{Trick1} & \textbf{Trick2} & \textbf{Trick3} & \textbf{Result} \\
        \textbf{Optimizer} & \textbf{structural noise} & \textbf{$\bsigma_t$ depend} & \textbf{Inf-norm} & \textbf{RMS-norm} \\
        & \textbf{$\sign{\bsigma_t\odot\n_t}$} & \textbf{on $\m_t$} & \textbf{on $\bsigma_t$} & \textbf{$\|\m_t/\bsigma_t\|_{\text{RMS}}$} \\
        \midrule
        SignSGD & \ding{55} & \ding{55} & \ding{55} & 1 \\
        \textcolor{gray}{AdamW} & \ding{55} & \ding{55} & \ding{55} & $\approx$0.2 \\
        \textcolor{gray}{AdaMax} & \ding{55} & \ding{55} & \ding{51} & $\approx$0.1 \\
        \textcolor{gray}{IE-StoSignSGD} & \ding{55} & \ding{51} & \ding{51} & $\approx$0.4 \\
        \textcolor{blue}{SignAdamW} & \ding{51} & \ding{55} & \ding{55} & 1 \\
        \textcolor{blue}{SignAdaMax} & \ding{51} & \ding{55} & \ding{51} & 1 \\
        \rowcolor{yellow!20}\textcolor{blue}{\textbf{StoSignSGD}} & \ding{51} & \ding{51} & \ding{51} & 1 \\
        \bottomrule
    \end{tabular}
    \vspace{-1pt}
\end{table}
\cref{tab:tricks} summarizes three key tricks that constitute StoSignSGD. The first technique is the injection of structural noise $\bsigma_t\odot\n_t$ into the sign operator $\sign{\cdot}$, which essentially defines whether a base method has been converted into a sign-based optimizer via \cref{alg:sign-conversion}. The second trick introduces a direct dependency between the optimizer states $\bsigma_t$ and $\m_t$, specifically coupling the derivation of $\bsigma_t$ to $\m_t$. The third trick dictates whether $\bsigma_t$ tracks the infinity norm of historic gradients, fundamentally altering the optimization geometry. Utilizing the infinity norm requires computing $\bsigma_t$ via a $\max$ operator, whereas omitting it defaults to an element-wise squared operation, as used in AdamW. Toggling these design choices yields various distinct algorithms, the full procedures for which are detailed in \cref{sec:supporting-algorithms}.

\subsubsection{Ablation Experiments\label{sec:ablation}}

To gain a deeper understanding of the three design choices outlined in \cref{tab:tricks}, we evaluate a wide range of model architectures on instruction following and mathematical reasoning tasks. Due to space constraints, the full experimental details are deferred to \cref{sec:dessect-sign-noise}. For the instruction-following task, we report the training loss of six models finetuned on the Alpaca dataset~\citep{taori2023alpaca}. For the mathematical reasoning task, we finetune three 7B+ parameter models on the training sets of GSM8k~\citep{cobbe2021gsm8k} and MathQA~\citep{amini-etal-2019-mathqa}. The results are summarized in \cref{tab:alpaca,tab:gsm8k_mathqa}.

\begin{table}[htbp]
\centering

\caption{Training loss on the Alpaca dataset. Results are reported as mean $\pm$ error.\label{tab:alpaca}}
\vspace{-5pt}
\resizebox{\linewidth}{!}{
\begin{tabular}{lcccccc}
\toprule
\textbf{Optimizer} & \textbf{Qwen2.5-0.5B} & \textbf{TinyLLaMA} & \textbf{Llama-3.2-1B} & \textbf{GPT-2} & \textbf{Gemma-3-1B} & \textbf{Phi-2} \\
\midrule
SignSGD 
& 1.282$\pm$.002 
& 0.959$\pm$.002 
& 1.159$\pm$.001 
& 1.896$\pm$.002 
& 1.158$\pm$.000 
& 1.012$\pm$.003 \\

\textcolor{gray}{AdamW} 
& 1.290$\pm$.000 
& 0.958$\pm$.002 
& 1.168$\pm$.001 
& 1.891$\pm$.000 
& 1.132$\pm$.000 
& 1.026$\pm$.001 \\

\textcolor{gray}{AdaMax} 
& 1.279$\pm$.001 
& 0.954$\pm$.002 
& 1.165$\pm$.002 
& 1.899$\pm$.002 
& 1.134$\pm$.001 
& 0.993$\pm$.001 \\

\textcolor{gray}{IE-StoSignSGD} 
& 1.235$\pm$.001 
& 0.933$\pm$.002 
& \textbf{1.092$\pm$.002} 
& \textbf{1.858$\pm$.001} 
& \textbf{1.105$\pm$.001} 
& 1.018$\pm$.002 \\

\textcolor{blue}{SignAdamW} 
& 1.272$\pm$.001 
& 0.949$\pm$.002 
& 1.112$\pm$.001 
& 1.875$\pm$.002 
& 1.134$\pm$.001 
& \textbf{0.985$\pm$.001} \\

\textcolor{blue}{SignAdaMax} 
& 1.301$\pm$.002 
& 0.964$\pm$.003 
& 1.127$\pm$.002 
& 1.925$\pm$.003 
& 1.152$\pm$.001 
& 0.991$\pm$.002 \\

\rowcolor{yellow!20}
\textcolor{blue}{\textbf{StoSignSGD}} 
& \textbf{1.234$\pm$.000} 
& \textbf{0.932$\pm$.002} 
& \textbf{1.092$\pm$.002} 
& \textbf{1.858$\pm$.001} 
& 1.106$\pm$.001 
& 0.989$\pm$.001 \\
\bottomrule
\end{tabular}
}
\vspace{-10pt}
\end{table}

Based on the empirical evidence in \cref{tab:alpaca,tab:gsm8k_mathqa}, we draw the following conclusions. \textbf{(i) Trick 1 is highly effective as a core mechanism.} This is supported by the performance gains of \textcolor{blue}{StoSignSGD} over \textcolor{gray}{IE-StoSignSGD}, \textcolor{blue}{SignAdamW} over \textcolor{gray}{AdamW}, and \textcolor{blue}{SignAdaMax} performing better than or on par with \textcolor{gray}{AdaMax}. \textbf{(ii) Trick 2 provides consistent improvements, independent of Trick 1.} This is evident from the superiority of \textcolor{gray}{IE-StoSignSGD} over \textcolor{gray}{AdaMax}, and \textcolor{blue}{StoSignSGD} over \textcolor{blue}{SignAdaMax}. \textbf{(iii) Trick 3 is synergistic with Trick 2, but may not be effective in isolation with only Trick 1.} This is illustrated by \textcolor{gray}{AdaMax} performing similarly to \textcolor{gray}{AdamW}, and \textcolor{blue}{SignAdamW} performing similarly to \textcolor{blue}{SignAdaMax}, while their combination in \textcolor{blue}{StoSignSGD} yields the best overall results. Together, these findings confirm that all three design choices in \cref{tab:tricks} are critical components. To fully understand why this specific combination is optimal, we examine the underlying optimization dynamics.

\paragraph{Further Theoretical and Empirical Justifications}
We provide further theoretical and empirical reasoning to justify the design of StoSignSGD and explain why alternative sign conversions, such as SignAdamW and SignAdaMax, fall short. From the perspective of non-smooth optimization theory, our design of the preconditioner $\Gb_t$, i.e., Trick 3 plus Trick, already gives the best possible convergence for sign-based optimizers (see \cref{remark:lower-bound}), whose tightness is also justified by \cref{thm:lower-bound}. 
When it comes to smooth optimization, managing the variance of the stochastic updates is critical for convergence~\citep{Bottou18optimization}. To understand this, we analyze the variance of the three sign-converted methods. Let $\SS_{\bsigma_t}(\m_t) := \sign{\m_t+\bsigma_t\odot\n_t}$ and define the signal-to-noise ratio $\r_t := \m_t/\bsigma_t$. The (conditional) variance for coordinate $i$ is exactly $\mathrm{Var}([\SS_{\bsigma_t}(\m_t)]_i) = 1-\r_{t,i}^2$. Consequently, the total variance across all coordinates is $\mathrm{Var}(\SS_{\bsigma_t}(\m_t)) = d(1-\|\r_t\|^2_{\text{RMS}})$. This equation illustrates a direct mathematical relationship: a smaller relative magnitude $\r_t$ strictly inflates the per-step variance. When we evaluate the RMS-norms of $\r_t$ for StoSignSGD, SignAdamW, and SignAdaMax along their optimization trajectories (see \cref{fig:rms-norm} and the last column of \cref{tab:tricks}), a clear pattern emerges: StoSignSGD sustains a significantly larger $\r_t$ compared to the alternatives. By dynamically coupling $\bsigma_t$ and $\m_t$ (Trick 2), StoSignSGD ensures that the structural noise does not overwhelm the gradient signal. This allows the algorithm to smartly leverage the sign-plus-noise mechanism without incurring excessive stochasticity. Rather than suffering from unchecked variance, StoSignSGD achieves an optimal balance: it injects just enough structural stochasticity to benefit from its exploratory properties, while preserving a sufficiently large $\r_t$ to keep the update variance strictly under control for stable convergence.

\section{Conclusion}

In this paper, we propose StoSignSGD, a variant of SignSGD that adds unbiased structural stochasticity to the vanilla sign operator. Our theoretical analysis shows that StoSignSGD fixes the non-convergence issue of standard SignSGD and improves prior best-known complexity bounds by dimensional factors. Empirically, extensive LLM pretraining and fine-tuning experiments show clear practical efficiency gains. StoSignSGD gives consistent speedups under FP8 pretraining, and in the far more demanding FP4 regime it stays stable across OLMo2 models up to 370M parameters, where AdamW collapses, while improving downstream accuracy over Lion by up to 1.13 points, with the largest gain at the biggest model and longest run. To explain these gains, we develop an unbiased sign conversion framework for general optimizers, which clarifies the role of structural stochasticity in StoSignSGD and offers theoretical support for our design.

\bibliography{ref}

@book{goodfellow2016deep,
title={Deep learning},
author={Goodfellow, Ian and Bengio, Yoshua and Courville, Aaron and Bengio, Yoshua},
volume={1},
year={2016},
publisher={MIT Press}
}

@article{duchi2011adaptive,
  title={Adaptive subgradient methods for online learning and stochastic optimization.},
  author={Duchi, John and Hazan, Elad and Singer, Yoram},
  journal={Journal of Machine Learning Research (JMLR)},
  volume={12},
  number={7},
  year={2011}
}

@InProceedings{sun2023momentum,
  title = 	 {Momentum Ensures Convergence of {SIGNSGD} under Weaker Assumptions},
  author =       {Sun, Tao and Wang, Qingsong and Li, Dongsheng and Wang, Bao},
  booktitle = 	 {Proceedings of the 40th International Conference on Machine Learning (ICML)},
  pages = 	 {33077--33099},
  year = 	 {2023},
}

@misc{jordan2024muon,
  author       = {Keller Jordan and Yuchen Jin and Vlado Boza and Jiacheng You and Franz Cesista and Laker Newhouse and Jeremy Bernstein},
  title        = {{M}uon: An optimizer for hidden layers in neural networks},
  year         = {2024},
  url          = {https://kellerjordan.github.io/posts/muon/}
}

@inproceedings{zhang2020complexity,
  title={Complexity of finding stationary points of nonconvex nonsmooth functions},
  author={Zhang, Jingzhao and Lin, Hongzhou and Jegelka, Stefanie and Sra, Suvrit and Jadbabaie, Ali},
  booktitle={Proceedings of the 37th International Conference on Machine Learning (ICML)},
  pages={11173--11182},
  year={2020},
}

@inproceedings{cutkosky2023optimal,
  title={Optimal stochastic non-smooth non-convex optimization through online-to-non-convex conversion},
  author={Cutkosky, Ashok and Mehta, Harsh and Orabona, Francesco},
  booktitle={Proceedings of the 40th International Conference on Machine Learning (ICML)},
  pages={6643--6670},
  year={2023},
}

@inproceedings{zhang2024random,
  title={Random Scaling and Momentum for Non-smooth Non-convex Optimization},
  author={Zhang, Qinzi and Cutkosky, Ashok},
  booktitle={Proceedings of the 41st International Conference on Machine Learning (ICML)},
  pages={58780--58799},
  year={2024},
}

@inproceedings{ahn2024adamema,
 author = {Ahn, Kwangjun and Cutkosky, Ashok},
 booktitle = {Advances in Neural Information Processing Systems 37 (NeurIPS)},
 pages = {94909--94933},
 title = {{A}dam with model exponential moving average is effective for nonconvex optimization},
 year = {2024}
}

@inproceedings{NeurIPS:2024:Jiang,
    author = {Wei Jiang and Sifan Yang and Wenhao Yang and Lijun Zhang},    
    title = {Efficient Sign-Based Optimization: Accelerating Convergence via Variance Reduction},
    booktitle = {Advances in Neural Information Processing Systems 37 (NeurIPS)},    
    pages = {33891--33932},    
    year = {2024},
}

@inproceedings{crawshaw2022robustness,
  title={Robustness to unbounded smoothness of generalized {SignSGD}},
  author={Crawshaw, Michael and Liu, Mingrui and Orabona, Francesco and Zhang, Wei and Zhuang, Zhenxun},
  booktitle={Advances in Neural Information Processing Systems 35 (NeurIPS)},
  pages={9955--9968},
  year={2022}
}

@inproceedings{liu2025adagrad,
    title={{AdaGrad} under Anisotropic Smoothness},
    author={Yuxing Liu and Rui Pan and Tong Zhang},
    booktitle={International Conference on Learning Representations (ICLR)},
    pages = {19574--19608},
    year={2025},
}

@article{orabona2019intro,
  author       = {Francesco Orabona},
  title        = {A Modern Introduction to Online Learning},
  journal      = {arXiv preprint arXiv:1912.13213v8},
  year         = {2019},
}

@article{bubeck2015convex,
  title={Convex optimization: Algorithms and complexity},
  author={Bubeck, S{\'e}bastien and others},
  journal={Foundations and Trends{\textregistered} in Machine Learning},
  volume={8},
  number={3-4},
  pages={231--357},
  year={2015},
  publisher={Now Publishers, Inc.}
}

@book{cesa2006prediction,
  title={Prediction, learning, and games},
  author={Cesa-Bianchi, Nicolo and Lugosi, G{\'a}bor},
  year={2006},
  publisher={Cambridge University Press}
}

@inproceedings{NeurIPS:2023:Zhang,
author = {Lijun Zhang and Peng Zhao and Zhen-Hua Zhuang and Tianbao Yang and Zhi-Hua Zhou},
title = {Stochastic Approximation Approaches to Group Distributionally Robust Optimization},
booktitle = {Advances in Neural Information Processing Systems 36 (NeurIPS)},
pages = {52490--52522},
year = {2023},
}

@ARTICLE{zhang2026gdro,
  author={Zhang, Lijun and Bai, Haomin and Zhao, Peng and Zhou, Zhi-Hua},
  journal={IEEE Transactions on Pattern Analysis and Machine Intelligence}, 
  title={Stochastic Approximation Approaches to Group Distributionally Robust Optimization and Beyond}, 
  year={2026},
  volume={},
  number={},
  pages={1-12},
}

@inproceedings{ICML:2024:Liu,
    author = {Langqi Liu and Yibo Wang and Lijun Zhang},
    title = {High-Probability Bound for Non-Smooth Non-Convex Stochastic Optimization with Heavy Tails},
    booktitle = {Proceedings of the 41st International Conference on Machine Learning (ICML)},
    pages = {32122--32138},
    year = {2024},
}

@article{nemirovskij1983problem,
  title={Problem complexity and method efficiency in optimization},
  author={Arkadi Semen Nemirovski and David Berkovich Yudin},
  year={1983},
  journal={Wiley-Interscience},
  publisher={Wiley-Interscience}
}

@article{nemirovski2009robust,
  title={Robust stochastic approximation approach to stochastic programming},
  author={Nemirovski, Arkadi and Juditsky, Anatoli and Lan, Guanghui and Shapiro, Alexander},
  journal={SIAM Journal on Optimization},
  volume={19},
  number={4},
  pages={1574--1609},
  year={2009},
  publisher={SIAM}
}

@article{Bottou18optimization,
    author = {Bottou, L\'{e}on and Curtis, Frank E. and Nocedal, Jorge},
    title = {Optimization Methods for Large-Scale Machine Learning},
    journal = {SIAM Review},
    volume = {60},
    number = {2},
    pages = {223-311},
    year = {2018},
}

@article{arjevani2023lower,
  title={Lower bounds for non-convex stochastic optimization},
  author={Arjevani, Yossi and Carmon, Yair and Duchi, John C. and Foster, Dylan J and Srebro, Nathan and Woodworth, Blake},
  journal={Mathematical Programming},
  volume={199},
  number={1},
  pages={165--214},
  year={2023},
  publisher={Springer}
}

@inproceedings{yu24egdro,
  title = 	 {Efficient Algorithms for Empirical Group Distributionally Robust Optimization and Beyond},
  author =       {Yu, Dingzhi and Cai, Yunuo and Jiang, Wei and Zhang, Lijun},
  booktitle = 	 {Proceedings of the 41st International Conference on Machine Learning (ICML)},
  pages = 	 {57384--57414},
  year = 	 {2024},
}

@inproceedings{he2016resnet,
  author={He, Kaiming and Zhang, Xiangyu and Ren, Shaoqing and Sun, Jian},
  booktitle={Proceedings of the IEEE/CVF Conference on Computer Vision and Pattern Recognition (CVPR)}, 
  title={Deep Residual Learning for Image Recognition}, 
  year={2016},
  pages={770-778}
}

@inproceedings{kingma15adam,
  author={Diederik P. Kingma and Jimmy Ba},
  title={{A}dam: A Method for Stochastic Optimization},
  year={2015},
  booktitle={International Conference on Learning Representations (ICLR)},
}

@article{chzhen2023signsvrg,
  title={{SignSVRG}: fixing {SignSGD} via variance reduction},
  author={Chzhen, Evgenii and Schechtman, Sholom},
  journal={arXiv preprint arXiv:2305.13187},
  year={2023}
}

@inproceedings{NEURIPS2019PYTORCH,
 author = {Paszke, Adam and Gross, Sam and Massa, Francisco and Lerer, Adam and Bradbury, James and Chanan, Gregory and Killeen, Trevor and Lin, Zeming and Gimelshein, Natalia and Antiga, Luca and Desmaison, Alban and Kopf, Andreas and Yang, Edward and DeVito, Zachary and Raison, Martin and Tejani, Alykhan and Chilamkurthy, Sasank and Steiner, Benoit and Fang, Lu and Bai, Junjie and Chintala, Soumith},
 booktitle = {Advances in Neural Information Processing Systems 32 (NeurIPS)},
 pages = {8026-8037},
 title = {{PyTorch: An Imperative Style, High-Performance Deep Learning Library}},
 year = {2019}
}

@article{hazan2007logarithmic,
  title={Logarithmic regret algorithms for online convex optimization},
  author={Hazan, Elad and Agarwal, Amit and Kale, Satyen},
  journal={Machine Learning},
  volume={69},
  number={2},
  pages={169--192},
  year={2007},
  publisher={Springer}
}

@inproceedings{McMahanS10adagrad,
  author       = {Brendan McMahan and
                  Matthew Streeter},
  title        = {Adaptive Bound Optimization for Online Convex Optimization},
  booktitle    = {Proceedings of the 23rd Conference on Learning Theory (COLT)},
  pages        = {244--256},
  year         = {2010},
}

@inproceedings{ahn2025general,
  title={General framework for online-to-nonconvex conversion: Schedule-free {SGD} is also effective for nonconvex optimization},
  author={Ahn, Kwangjun and Magakyan, Gagik and Cutkosky, Ashok},
  booktitle={Proceedings of the 42nd International Conference on Machine Learning (ICML)},
  year={2025},
  pages={772--795}
}

@article{bubeck2012regret,
  title={Regret analysis of stochastic and nonstochastic multi-armed bandit problems},
  author={Bubeck, S{\'e}bastien and Cesa-Bianchi, Nicolo},
  journal={Foundations and Trends{\textregistered} in Machine Learning},
  volume={5},
  number={1},
  pages={1--122},
  year={2012},
  publisher={Now Publishers, Inc.}
}

@book{lattimore2020bandit,
  title={Bandit algorithms},
  author={Lattimore, Tor and Szepesv{\'a}ri, Csaba},
  year={2020},
  publisher={Cambridge University Press}
}

@article{bai2025group,
  title={Group Distributionally Robust Optimization with Flexible Sample Queries},
  author={Bai, Haomin and Yu, Dingzhi and Li, Shuai and Luo, Haipeng and Zhang, Lijun},
  journal={arXiv preprint arXiv:2505.15212},
  year={2025}
}

@inproceedings{zinkevich2003oco,
author = {Zinkevich, Martin},
title = {Online convex programming and generalized infinitesimal gradient ascent},
year = {2003},
booktitle = {Proceedings of the 20th International Conference on Machine Learning (ICML)},
pages = {928–935},
}

@inproceedings{reddi2018convergence,
  title={On the Convergence of {A}dam and Beyond},
  author={Reddi, Sashank J and Kale, Satyen and Kumar, Sanjiv},
  booktitle={International Conference on Learning Representations (ICLR)},
  year={2018}
}

@article{agarwal2012information,
  title={Information-Theoretic Lower Bounds on the Oracle Complexity of Stochastic Convex Optimization},
  author={Agarwal, Alekh and Bartlett, Peter L and Ravikumar, Pradeep and Wainwright, Martin J},
  journal={IEEE Transactions on Information Theory},
  volume={58},
  number={5},
  pages={3235--3249},
  year={2012},
}

@inproceedings{bernstein2018signsgd,
  title={{S}ign{SGD}: Compressed optimisation for non-convex problems},
  author={Bernstein, Jeremy and Wang, Yu-Xiang and Azizzadenesheli, Kamyar and Anandkumar, Animashree},
  booktitle={Proceedings of the 35th International Conference on Machine Learning (ICML)},
  pages={560--569},
  year={2018},
}

@InProceedings{karimireddy2019error,
  title = 	 {Error Feedback Fixes {S}ign{SGD} and other Gradient Compression Schemes},
  author =       {Karimireddy, Sai Praneeth and Rebjock, Quentin and Stich, Sebastian and Jaggi, Martin},
  booktitle = 	 {Proceedings of the 36th International Conference on Machine Learning (ICML)},
  pages = 	 {3252--3261},
  year = 	 {2019},
}

@article{xiao2023stochastic,
  title={Stochastic Subgradient Methods with Guaranteed Global Stability in Nonsmooth Nonconvex Optimization},
  author={Xiao, Nachuan and Hu, Xiaoyin and Toh, Kim-Chuan},
  journal={arXiv preprint arXiv:2307.10053},
  year={2023}
}

@article{hazan2019lecture,
  title={Lecture notes: Optimization for machine learning},
  author={Hazan, Elad},
  journal={arXiv preprint arXiv:1909.03550},
  year={2019}
}

@article{hazan2019introduction,
  title={Introduction to Online Convex Optimization},
  author={Hazan, Elad},
  journal={arXiv preprint arXiv:1909.05207v3},
  year={2019}
}

@inproceedings{NIPS:2013:Zhang,
author = {Lijun Zhang and Mehrdad Mahdavi and Rong Jin},
title = {Linear Convergence with Condition Number Independent Access of Full Gradients},
booktitle = {Advance in Neural Information Processing Systems 26 (NIPS)},
pages = {980--988},
year = {2013},
}

@inproceedings{NIPS2013svrg,
 author = {Johnson, Rie and Zhang, Tong},
 booktitle = {Advances in Neural Information Processing Systems 26 (NIPS)},
 pages = {315--323},
 title = {Accelerating Stochastic Gradient Descent using Predictive Variance Reduction},
 year = {2013},
}

@inproceedings{bernstein2019signsgd,
  title={{S}ign{SGD} with Majority Vote is Communication Efficient and Fault Tolerant},
  author={Bernstein, Jeremy and Zhao, Jiawei and Azizzadenesheli, Kamyar and Anandkumar, Anima},
  booktitle={International Conference on Learning Representations (ICLR)},
  year={2019},
}

@misc{krizhevsky2009learning,
  title={Learning multiple layers of features from tiny images.(2009)},
  author={Krizhevsky, Alex and Hinton, Geoffrey and others},
  year={2009}
}

@inproceedings{safaryan2021stochastic,
  title={Stochastic sign descent methods: New algorithms and better theory},
  author={Safaryan, Mher and Richt{\'a}rik, Peter},
  booktitle={Proceedings of the 38th International Conference on Machine Learning (ICML)},
  pages={9224--9234},
  year={2021},
}

@inproceedings{chen2023symbolic,
  title={Symbolic discovery of optimization algorithms},
  author={Chen, Xiangning and Liang, Chen and Huang, Da and Real, Esteban and Wang, Kaiyuan and Pham, Hieu and Dong, Xuanyi and Luong, Thang and Hsieh, Cho-Jui and Lu, Yifeng and others},
  booktitle={Advances in Neural Information Processing Systems 36 (NeurIPS)},
  pages={49205--49233},
  year={2023}
}

@article{dong2024convergence,
  title={Convergence Rate Analysis of {LION}},
  author={Dong, Yiming and Li, Huan and Lin, Zhouchen},
  journal={arXiv preprint arXiv:2411.07724},
  year={2024}
}

@inproceedings{balles2018dissecting,
  title={Dissecting {A}dam: The sign, magnitude and variance of stochastic gradients},
  author={Balles, Lukas and Hennig, Philipp},
  booktitle={Proceedings of the 35th International Conference on Machine Learning (ICML)},
  pages={404--413},
  year={2018},
}

@article{jin2020stochastic,
  title={Stochastic-{S}ign {SGD} for Federated Learning with Theoretical Guarantees},
  author={Richeng Jin and Yufan Huang and Xiaofan He and Huaiyu Dai and Tianfu Wu},
  journal={arXiv preprint arXiv:2002.10940},
  year={2020}
}

@article{jiang2025improved,
  title={Improved Analysis for Sign-based Methods with Momentum Updates},
  author={Wei Jiang and Dingzhi Yu and Sifan Yang and Wenhao Yang and Lijun Zhang},
  journal={arXiv preprint arXiv:2507.12091},
  year={2025}
}

@article{kornowski2022oracle,
  author  = {Guy Kornowski and Ohad Shamir},
  title   = {Oracle Complexity in Nonsmooth Nonconvex Optimization},
  journal = {Journal of Machine Learning Research (JMLR)},
  year    = {2022},
  volume  = {23},
  number  = {314},
  pages   = {1--44},
}

@InProceedings{an2025asgo,
  title = 	 {{ASGO}: Adaptive Structured Gradient Optimization},
  author =       {Kang An and Yuxing Liu and Rui Pan and Yi Ren and Shiqian Ma and Donald Goldfarb and Tong Zhang},
  booktitle = 	 {Advances in Neural Information Processing Systems 38 (NeurIPS)},
  pages = 	 {126775--126814},
  year = 	 {2025},
}

@article{radford2019language,
  title={Language models are unsupervised multitask learners},
  author={Radford, Alec and Wu, Jeffrey and Child, Rewon and Luan, David and Amodei, Dario and Sutskever, Ilya and others},
  journal={OpenAI blog},
  volume={1},
  number={8},
  pages={9},
  year={2019}
}

@article{touvron2023llama,
  title={{Llama}: Open and efficient foundation language models},
  author={Touvron, Hugo and Lavril, Thibaut and Izacard, Gautier and Martinet, Xavier and Lachaux, Marie-Anne and Lacroix, Timoth{\'e}e and Rozi{\`e}re, Baptiste and Goyal, Naman and Hambro, Eric and Azhar, Faisal and others},
  journal={arXiv preprint arXiv:2302.13971},
  year={2023}
}

@article{yang2025qwen3,
  title={Qwen3 technical report},
  author={Yang, An and Li, Anfeng and Yang, Baosong and Zhang, Beichen and Hui, Binyuan and Zheng, Bo and Yu, Bowen and Gao, Chang and Huang, Chengen and Lv, Chenxu and others},
  journal={arXiv preprint arXiv:2505.09388},
  year={2025}
}

@misc{taori2023alpaca,
  title={{Stanford Alpaca}: An instruction-following {Llama} model},
  author={Taori, Rohan and Gulrajani, Ishaan and Zhang, Tianyi and Dubois, Yann and Li, Xuechen and Guestrin, Carlos and Liang, Percy and Hashimoto, Tatsunori B},
  year={2023},
  publisher={Stanford, CA, USA},
  url={https://github.com/tatsu-lab/stanford_alpaca}
}

@InProceedings{yuan25mars,
  title = 	 {{MARS}: Unleashing the Power of Variance Reduction for Training Large Models},
  author =       {Yuan, Huizhuo and Liu, Yifeng and Wu, Shuang and Xun, Zhou and Gu, Quanquan},
  booktitle = 	 {Proceedings of the 42nd International Conference on Machine Learning (ICML)},
  pages = 	 {73553--73587},
  year = 	 {2025},
}

@inproceedings{diao2024lmflow,
  title={{LMFlow}: An extensible toolkit for finetuning and inference of large foundation models},
  author={Diao, Shizhe and Pan, Rui and Dong, Hanze and Shum, KaShun and Zhang, Jipeng and Xiong, Wei and Zhang, Tong},
  booktitle={Proceedings of the 2024 Conference of the North American Chapter of the Association for Computational Linguistics: Human Language Technologies (Volume 3: System Demonstrations)},
  pages={116--127},
  year={2024},
  url={https://github.com/OptimalScale/LMFlow}
}

@article{liu2025muon,
  title={{M}uon is scalable for {LLM} training},
  author={Liu, Jingyuan and Su, Jianlin and Yao, Xingcheng and Jiang, Zhejun and Lai, Guokun and Du, Yulun and Qin, Yidao and Xu, Weixin and Lu, Enzhe and Yan, Junjie and others},
  journal={arXiv preprint arXiv:2502.16982},
  year={2025}
}

@inproceedings{brown2020gpt3,
 author = {Brown, Tom and Mann, Benjamin and Ryder, Nick and Subbiah, Melanie and Kaplan, Jared D and Dhariwal, Prafulla and Neelakantan, Arvind and Shyam, Pranav and Sastry, Girish and Askell, Amanda and Agarwal, Sandhini and Herbert-Voss, Ariel and Krueger, Gretchen and Henighan, Tom and Child, Rewon and Ramesh, Aditya and Ziegler, Daniel and Wu, Jeffrey and Winter, Clemens and Hesse, Chris and Chen, Mark and Sigler, Eric and Litwin, Mateusz and Gray, Scott and Chess, Benjamin and Clark, Jack and Berner, Christopher and McCandlish, Sam and Radford, Alec and Sutskever, Ilya and Amodei, Dario},
 booktitle = {Advances in Neural Information Processing Systems 33 (NeurIPS)},
 pages = {1877--1901},
 title = {Language Models are Few-Shot Learners},
 year = {2020}
}

@article{touvron2023llama2,
  title={{Llama} 2: Open foundation and fine-tuned chat models},
  author={Touvron, Hugo and Martin, Louis and Stone, Kevin and Albert, Peter and Almahairi, Amjad and Babaei, Yasmine and Bashlykov, Nikolay and Batra, Soumya and Bhargava, Prajjwal and Bhosale, Shruti and others},
  journal={arXiv preprint arXiv:2307.09288},
  year={2023}
}

@article{achiam2023gpt4,
  title={{GPT}-4 technical report},
  author={Achiam, Josh and Adler, Steven and Agarwal, Sandhini and Ahmad, Lama and Akkaya, Ilge and Aleman, Florencia Leoni and Almeida, Diogo and Altenschmidt, Janko and Altman, Sam and Anadkat, Shyamal and others},
  journal={arXiv preprint arXiv:2303.08774},
  year={2023}
}

@inproceedings{loshchilov2019adamw,
    title={Decoupled Weight Decay Regularization},
    author={Ilya Loshchilov and Frank Hutter},
    booktitle={International Conference on Learning Representations (ICLR)},
    year={2019},
}

@article{guo2025deepseek,
  title={{DeepSeek-R1}: Incentivizing reasoning capability in {LLM}s via reinforcement learning},
  author={Guo, Daya and Yang, Dejian and Zhang, Haowei and Song, Junxiao and Zhang, Ruoyu and Xu, Runxin and Zhu, Qihao and Ma, Shirong and Wang, Peiyi and Bi, Xiao and others},
  journal={arXiv preprint arXiv:2501.12948},
  year={2025}
}

@article{liu2024deepseek,
  title={{DeepSeek-V3} technical report},
  author={Liu, Aixin and Feng, Bei and Xue, Bing and Wang, Bingxuan and Wu, Bochao and Lu, Chengda and Zhao, Chenggang and Deng, Chengqi and Zhang, Chenyu and Ruan, Chong and others},
  journal={arXiv preprint arXiv:2412.19437},
  year={2024}
}

@article{semenov2025benchmarking,
  title={Benchmarking Optimizers for Large Language Model Pretraining},
  author={Semenov, Andrei and Pagliardini, Matteo and Jaggi, Martin},
  journal={arXiv preprint arXiv:2509.01440},
  year={2025}
}

@inproceedings{wen2026fantastic,
    title={Fantastic Pretraining Optimizers and Where to Find Them},
    author={Kaiyue Wen and David Leo Wright Hall and Tengyu Ma and Percy Liang},
    booktitle={International Conference on Learning Representations (ICLR)},
    year={2026},
    pages = {144731--144838}
}

@inproceedings{kunstner2023noise,
    title={Noise Is Not the Main Factor Behind the Gap Between {SGD} and {Adam} on Transformers, But Sign Descent Might Be},
    author={Frederik Kunstner and Jacques Chen and Jonathan Wilder Lavington and Mark Schmidt},
    booktitle={International Conference on Learning Representations (ICLR)},
    year={2023},
}

@inproceedings{jiang2026lion,
    title={Convergence Analysis of the Lion Optimizer in Centralized and Distributed Settings},
    author={Wei Jiang and Mao Xu and Wenhao Yang and Yibo Wang and Zechao Li and Lijun Zhang},
    booktitle={Proceedings of the 43rd International Conference on Machine Learning (ICML)},
    year={2026},
    pages={to appear}
}

@article{pan2025unbiased,
  title={Unbiased Gradient Low-Rank Projection},
  author={Pan, Rui and Luo, Yang and Liu, Yuxing and You, Yang and Zhang, Tong},
  journal={arXiv preprint arXiv:2510.17802},
  year={2025}
}

@InProceedings{pmlr-v235-kosson24a,
  title = 	 {Rotational Equilibrium: How Weight Decay Balances Learning Across Neural Networks},
  author =       {Kosson, Atli and Messmer, Bettina and Jaggi, Martin},
  booktitle = 	 {Proceedings of the 41st International Conference on Machine Learning (ICML)},
  pages = 	 {25333--25369},
  year = 	 {2024},
}

@inproceedings{zhaoICLR2025deconstructing,
 author = {Zhao, Rosie and Morwani, Depen and Brandfonbrener, David and Vyas, Nikhil and Kakade, Sham},
 booktitle = {International Conference on Learning Representations (ICLR)},
 pages = {2830--2850},
 title = {Deconstructing What Makes a Good Optimizer for Autoregressive Language Models},
 year = {2025}
}

@article{team2023gemini,
  title={{Gemini}: a family of highly capable multimodal models},
  author={Team, Gemini and Anil, Rohan and Borgeaud, Sebastian and Alayrac, Jean-Baptiste and Yu, Jiahui and Soricut, Radu and Schalkwyk, Johan and Dai, Andrew M and Hauth, Anja and Millican, Katie and others},
  journal={arXiv preprint arXiv:2312.11805},
  year={2023}
}

@inproceedings{hoffmann2022chinchilla,
      title={Training Compute-Optimal Large Language Models}, 
      author={Jordan Hoffmann and Sebastian Borgeaud and Arthur Mensch and Elena Buchatskaya and Trevor Cai and Eliza Rutherford and Diego de Las Casas and Lisa Anne Hendricks and Johannes Welbl and Aidan Clark and Tom Hennigan and Eric Noland and Katie Millican and George van den Driessche and Bogdan Damoc and Aurelia Guy and Simon Osindero and Karen Simonyan and Erich Elsen and Jack W. Rae and Oriol Vinyals and Laurent Sifre},
      year={2022},
      pages = {30016--30030},
      booktitle = {Advances in Neural Information Processing Systems 35 (NeurIPS)},
}

@misc{nanogpt,
  author = {Andrej Karpathy},
  title = {{NanoGPT}},
  year = {2022},
  publisher = {GitHub},
  url = {https://github.com/karpathy/nanoGPT}
}

@inproceedings{liu2026old,
    author = {Zijian Liu},
    title = {Online Convex Optimization with Heavy Tails: Old Algorithms, New Regrets, and Applications},
    booktitle = {Proceedings of the 37th International Conference on Algorithmic Learning Theory (ALT)},
    year = {2026},
    pages={1--47},
}

@article{yang2026distributed,
  title={Distributed Online Convex Optimization with Compressed Communication: Optimal Regret and Applications},
  author={Yang, Sifan and Li, Dan-Yue and Zhang, Lijun},
  journal={arXiv preprint arXiv:2604.09276},
  year={2026}
}

@article{zeng2026glm,
  title={{GLM-5}: from Vibe Coding to Agentic Engineering},
  author={Zeng, Aohan and Lv, Xin and Hou, Zhenyu and Du, Zhengxiao and Zheng, Qinkai and Chen, Bin and Yin, Da and Ge, Chendi and Xie, Chengxing and Wang, Cunxiang and others},
  journal={arXiv preprint arXiv:2602.15763},
  year={2026}
}

@inproceedings{pan-etal-2025-scalebio,
    title = "{S}cale{B}i{O}: Scalable Bilevel Optimization for {LLM} Data Reweighting",
    author = "Pan, Rui  and
      Zhang, Dylan  and
      Zhang, Hanning  and
      Pan, Xingyuan  and
      Xu, Minrui  and
      Zhang, Jipeng  and
      Pi, Renjie  and
      Wang, Xiaoyu  and
      Zhang, Tong",
    booktitle = "Proceedings of the 63rd Annual Meeting of the Association for Computational Linguistics (Volume 1: Long Papers)",
    year = "2025",
    pages = "31959--31982",
}

@article{xiong2025stepwiser,
  title={Stepwiser: Stepwise generative judges for wiser reasoning},
  author={Xiong, Wei and Zhao, Wenting and Yuan, Weizhe and Golovneva, Olga and Zhang, Tong and Weston, Jason and Sukhbaatar, Sainbayar},
  journal={arXiv preprint arXiv:2508.19229},
  year={2025}
}

@misc{numina_math_datasets,
  author = {Jia LI and Edward Beeching and Lewis Tunstall and Ben Lipkin and Roman Soletskyi and Shengyi Costa Huang and Kashif Rasul and Longhui Yu and Albert Jiang and Ziju Shen and Zihan Qin and Bin Dong and Li Zhou and Yann Fleureau and Guillaume Lample and Stanislas Polu},
  title = {{NuminaMath}},
  year = {2024},
  publisher = {Numina},
  journal = {Hugging Face repository},
  howpublished = {\url{[https://huggingface.co/AI-MO/NuminaMath-CoT](https://github.com/project-numina/aimo-progress-prize/blob/main/report/numina_dataset.pdf)}}
}

@article{cobbe2021gsm8k,
  title={Training verifiers to solve math word problems},
  author={Cobbe, Karl and Kosaraju, Vineet and Bavarian, Mohammad and Chen, Mark and Jun, Heewoo and Kaiser, Lukasz and Plappert, Matthias and Tworek, Jerry and Hilton, Jacob and Nakano, Reiichiro and others},
  journal={arXiv preprint arXiv:2110.14168},
  year={2021}
}

@inproceedings{hendrycks2021math,
 author = {Hendrycks, Dan and Burns, Collin and Kadavath, Saurav and Arora, Akul and Basart, Steven and Tang, Eric and Song, Dawn and Steinhardt, Jacob},
 booktitle = {Proceedings of the Neural Information Processing Systems Track on Datasets and Benchmarks},
 title = {Measuring Mathematical Problem Solving With the {MATH} Dataset},
 url = {https://github.com/hendrycks/math/},
 year = {2021}
}

@article{yang2024qwen2.5,
  title={Qwen2.5 Technical Report},
  author={Yang, An and Yang, Baosong and Zhang, Beichen and Hui, Binyuan and Zheng, Bo and Yu, Bowen and Li, Chengyuan and Liu, Dayiheng and Huang, Fei and Wei, Haoran and others},
  journal={arXiv preprint arXiv:2412.15115},
  year={2024}
}

@article{grattafiori2024llama,
  title={The {Llama 3} herd of models},
  author={Grattafiori, Aaron and Dubey, Abhimanyu and Jauhri, Abhinav and Pandey, Abhinav and Kadian, Abhishek and Al-Dahle, Ahmad and Letman, Aiesha and Mathur, Akhil and Schelten, Alan and Vaughan, Alex and others},
  journal={arXiv preprint arXiv:2407.21783},
  year={2024}
}

@article{jiang2023mistral,
  title={{Mistral 7B}},
  author={Jiang, Albert Q and Sablayrolles, Alexandre and Mensch, Arthur and Bamford, Chris and Chaplot, Devendra Singh and Casas, Diego de las and Bressand, Florian and Lengyel, Gianna and Lample, Guillaume and Saulnier, Lucile and others},
  journal={arXiv preprint arXiv:2310.06825},
  year={2023}
}

@inproceedings{amini-etal-2019-mathqa,
    title = "{M}ath{QA}: Towards Interpretable Math Word Problem Solving with Operation-Based Formalisms",
    author = "Amini, Aida  and
      Gabriel, Saadia  and
      Lin, Shanchuan  and
      Koncel-Kedziorski, Rik  and
      Choi, Yejin  and
      Hajishirzi, Hannaneh",
    booktitle = "Proceedings of the 2019 Conference of the North {A}merican Chapter of the Association for Computational Linguistics: Human Language Technologies, Volume 1 (Long and Short Papers)",
    year = "2019",
    pages = "2357--2367",
}

@inproceedings{ICLR2025coat,
 author = {Xi, Haocheng and Cai, Han and Zhu, Ligeng and Lu, Yao and Keutzer, Kurt and Chen, Jianfei and Han, Song},
 booktitle = {International Conference on Learning Representations (ICLR)},
 pages = {42989--43011},
 title = {{COAT}: Compressing Optimizer states and Activations for Memory-Efficient {FP8} Training},
 year = {2025}
}

@misc{eval-harness,
  author       = {Gao, Leo and Tow, Jonathan and Abbasi, Baber and Biderman, Stella and Black, Sid and DiPofi, Anthony and Foster, Charles and Golding, Laurence and Hsu, Jeffrey and Le Noac'h, Alain and Li, Haonan and McDonell, Kyle and Muennighoff, Niklas and Ociepa, Chris and Phang, Jason and Reynolds, Laria and Schoelkopf, Hailey and Skowron, Aviya and Sutawika, Lintang and Tang, Eric and Thite, Anish and Wang, Ben and Wang, Kevin and Zou, Andy},
  title        = {The Language Model Evaluation Harness},
  month        = 07,
  year         = 2024,
  publisher    = {Zenodo},
  version      = {v0.4.3},
  doi          = {10.5281/zenodo.12608602},
  url          = {https://zenodo.org/records/12608602}
}

@misc{Gokaslan2019OpenWeb,
    title={{OpenWebText Corpus}},
    author={Gokaslan, Aaron and Cohen, Vanya and Pavlick, Ellie and Tellex, Stefanie},
    howpublished={\url{http://Skylion007.github.io/OpenWebTextCorpus}},
    year={2019}
}

@inproceedings{qiu2026why,
    title={{Why Low-Precision Transformer Training Fails: An Analysis on Flash Attention}},
    author={Haiquan Qiu and Quanming Yao},
    booktitle={International Conference on Learning Representations (ICLR)},
    year={2026},
    pages = {70649--70672}
}

@inproceedings{wei2022cot,
 author = {Wei, Jason and Wang, Xuezhi and Schuurmans, Dale and Bosma, Maarten and ichter, brian and Xia, Fei and Chi, Ed and Le, Quoc V and Zhou, Denny},
 booktitle = {Advances in Neural Information Processing Systems 35 (NeurIPS)},
 pages = {24824--24837},
 title = {Chain-of-Thought Prompting Elicits Reasoning in Large Language Models},
 year = {2022}
}

@inproceedings{hu2022lora,
    title={Lo{RA}: Low-Rank Adaptation of Large Language Models},
    author={Edward J Hu and Yelong Shen and Phillip Wallis and Zeyuan Allen-Zhu and Yuanzhi Li and Shean Wang and Lu Wang and Weizhu Chen},
    booktitle={International Conference on Learning Representations (ICLR)},
    year={2022},
}

@article{yu2026signheavytails,
  title={Sign-Based Optimizers Are Effective Under Heavy-Tailed Noise},
  author={Yu, Dingzhi and Tao, Hongyi and Wan, Yuanyu and Luo, Luo and Zhang, Lijun},
  journal={arXiv preprint arXiv:2602.07425},
  year={2026}
}

@inproceedings{loshchilov2017sgdr,
    title={{SGDR}: Stochastic Gradient Descent with Warm Restarts},
    author={Ilya Loshchilov and Frank Hutter},
    booktitle={International Conference on Learning Representations (ICLR)},
    year={2017},
}

@inproceedings{wolf-etal-2020-transformers,
    title = "Transformers: State-of-the-Art Natural Language Processing",
    author = "Thomas Wolf and Lysandre Debut and Victor Sanh and Julien Chaumond and Clement Delangue and Anthony Moi and Pierric Cistac and Tim Rault and Rémi Louf and Morgan Funtowicz and Joe Davison and Sam Shleifer and Patrick von Platen and Clara Ma and Yacine Jernite and Julien Plu and Canwen Xu and Teven Le Scao and Sylvain Gugger and Mariama Drame and Quentin Lhoest and Alexander M. Rush",
    booktitle = "Proceedings of the 2020 Conference on Empirical Methods in Natural Language Processing: System Demonstrations",
    year = "2020",
    pages = "38--45"
}

@article{or2025torchao,
  title={{TorchAO: PyTorch-Native Training-to-Serving Model Optimization}},
  author={Or, Andrew and Jain, Apurva and Vega-Myhre, Daniel and Cai, Jesse and Hernandez, Charles David and Zheng, Zhenrui and Guessous, Driss and Kuznetsov, Vasiliy and Puhrsch, Christian and Saroufim, Mark and others},
  journal={arXiv preprint arXiv:2507.16099},
  year={2025}
}

@inproceedings{dao2022flashattention,
 author = {Dao, Tri and Fu, Dan and Ermon, Stefano and Rudra, Atri and R\'{e}, Christopher},
 booktitle = {Advances in Neural Information Processing Systems 35 (NeurIPS)},
 pages = {16344--16359},
 title = {{FlashAttention: Fast and Memory-Efficient Exact Attention with IO-Awareness}},
 year = {2022}
}

@inproceedings{ICLR2025fp8scaling,
 author = {Fishman, Maxim and Chmiel, Brian and Banner, Ron and Soudry, Daniel},
 booktitle = {International Conference on Learning Representations (ICLR)},
 pages = {98631--98644},
 title = {Scaling {FP8} training to trillion-token {LLMs}},
 year = {2025}
}

@inproceedings{daoICLR2024flashattention2,
 author = {Dao, Tri},
 booktitle = {International Conference on Learning Representations (ICLR)},
 pages = {35549--35562},
 title = {{FlashAttention-2: Faster Attention with Better Parallelism and Work Partitioning}},
 year = {2024}
}

@inproceedings{shah2024flashattention3,
 author = {Shah, Jay and Bikshandi, Ganesh and Zhang, Ying and Thakkar, Vijay and Ramani, Pradeep and Dao, Tri},
 booktitle = {Advances in Neural Information Processing Systems 37 (NeurIPS)},
 pages = {68658--68685},
 title = {{FlashAttention-3: Fast and Accurate Attention with Asynchrony and Low-precision}},
 year = {2024}
}

@InProceedings{alacaoglu2020newregretanalysis,
  title = 	 {A new regret analysis for {A}dam-type algorithms},
  author =       {Alacaoglu, Ahmet and Malitsky, Yura and Mertikopoulos, Panayotis and Cevher, Volkan},
  booktitle = 	 {Proceedings of the 37th International Conference on Machine Learning (ICML)},
  pages = 	 {202--210},
  year = 	 {2020},
}

@inproceedings{panICLR2024momentum,
 author = {Pan, Rui and Liu, Yuxing and Wang, Xiaoyu and Zhang, Tong},
 booktitle = {International Conference on Learning Representations (ICLR)},
 pages = {55989--56028},
 title = {Accelerated Convergence of Stochastic Heavy Ball Method under Anisotropic Gradient Noise},
 year = {2024}
}

@inproceedings{kidambi2018on,
    title={On the insufficiency of existing momentum schemes for Stochastic Optimization},
    author={Rahul Kidambi and Praneeth Netrapalli and Prateek Jain and Sham M. Kakade},
    booktitle={International Conference on Learning Representations (ICLR)},
    year={2018},
}

@article{micikevicius2022fp8,
  title={{FP8} formats for deep learning},
  author={Micikevicius, Paulius and Stosic, Dusan and Burgess, Neil and Cornea, Marius and Dubey, Pradeep and Grisenthwaite, Richard and Ha, Sangwon and Heinecke, Alexander and Judd, Patrick and Kamalu, John and others},
  journal={arXiv preprint arXiv:2209.05433},
  year={2022}
}

@article{peng2023fp8,
  title={{FP8-LM}: Training {FP8} large language models},
  author={Peng, Houwen and Wu, Kan and Wei, Yixuan and Zhao, Guoshuai and Yang, Yuxiang and Liu, Ze and Xiong, Yifan and Yang, Ziyue and Ni, Bolin and Hu, Jingcheng and others},
  journal={arXiv preprint arXiv:2310.18313},
  year={2023}
}

@article{perez2023training,
  title={Training and inference of large language models using 8-bit floating point},
  author={Perez, Sergio P and Zhang, Yan and Briggs, James and Blake, Charlie and Levy-Kramer, Josh and Balanca, Paul and Luschi, Carlo and Barlow, Stephen and Fitzgibbon, Andrew William},
  journal={arXiv preprint arXiv:2309.17224},
  year={2023}
}

@article{lee2024fp8,
  title={To {FP8} and back again: Quantifying the effects of reducing precision on {LLM} training stability},
  author={Lee, Joonhyung and Bae, Jeongin and Kim, Byeongwook and Kwon, Se Jung and Lee, Dongsoo},
  journal={arXiv preprint arXiv:2405.18710},
  year={2024}
}

@article{balancca2024scalify,
  title={Scalify: scale propagation for efficient low-precision LLM training},
  author={Balan{\c{c}}a, Paul and Hosegood, Sam and Luschi, Carlo and Fitzgibbon, Andrew},
  journal={arXiv preprint arXiv:2407.17353},
  year={2024}
}

@inproceedings{wortsman2023stable,
 author = {Wortsman, Mitchell and Dettmers, Tim and Zettlemoyer, Luke and Morcos, Ari and Farhadi, Ali and Schmidt, Ludwig},
 booktitle = {Advances in Neural Information Processing Systems 36 (NeurIPS)},
 pages = {10271--10298},
 title = {Stable and low-precision training for large-scale vision-language models},
 year = {2023}
}

@article{JMLR:v15:srivastava14dropout,
  author  = {Nitish Srivastava and Geoffrey Hinton and Alex Krizhevsky and Ilya Sutskever and Ruslan Salakhutdinov},
  title   = {Dropout: A Simple Way to Prevent Neural Networks from Overfitting},
  journal = {Journal of Machine Learning Research (JMLR)},
  year    = {2014},
  volume  = {15},
  number  = {56},
  pages   = {1929--1958},
}

@article{hendrycks2016gelu,
  title={Gaussian error linear units ({GeLUs})},
  author={Hendrycks, Dan and Gimpel, Kevin},
  journal={arXiv preprint arXiv:1606.08415},
  year={2016}
}

@inproceedings{NEURIPS2019hybrid_fp8,
 author = {Sun, Xiao and Choi, Jungwook and Chen, Chia-Yu and Wang, Naigang and Venkataramani, Swagath and Srinivasan, Vijayalakshmi (Viji) and Cui, Xiaodong and Zhang, Wei and Gopalakrishnan, Kailash},
 booktitle = {Advances in Neural Information Processing Systems 32 (NeurIPS)},
 pages = {4900--4909},
 title = {{Hybrid 8-bit Floating Point (HFP8) Training and Inference for Deep Neural Networks}},
 year = {2019}
}

@article{noune20228bit,
  title={8-bit numerical formats for deep neural networks},
  author={Noune, Badreddine and Jones, Philip and Justus, Daniel and Masters, Dominic and Luschi, Carlo},
  journal={arXiv preprint arXiv:2206.02915},
  year={2022}
}

@inproceedings{zheng-etal-2024-llamafactory,
    title = "{L}lama{F}actory: Unified Efficient Fine-Tuning of 100+ Language Models",
    author = "Zheng, Yaowei  and
      Zhang, Richong  and
      Zhang, Junhao  and
      Ye, Yanhan  and
      Luo, Zheyan",
    booktitle = "Proceedings of the 62nd Annual Meeting of the Association for Computational Linguistics (Volume 3: System Demonstrations)",
    year = "2024",
    url = "https://github.com/hiyouga/LLaMAFactory",
    pages = "400--410",
}

@inproceedings{jiangSTOC2025improved,
    author = {Jiang, Ruichen and Mokhtari, Aryan and Patitucci, Francisco},
    title = {Improved Complexity for Smooth Nonconvex Optimization: A Two-Level Online Learning Approach with Quasi-Newton Methods},
    year = {2025},
    booktitle = {Proceedings of the 57th Annual ACM Symposium on Theory of Computing (STOC)},
    pages = {2225–2236},
}

@inproceedings{patitucci2026improving,
    title={Improving Online-to-Nonconvex Conversion for Smooth Optimization via Double Optimism},
    author={Francisco Patitucci and Ruichen Jiang and Aryan Mokhtari},
    booktitle={International Conference on Learning Representations (ICLR)},
    year={2026},
    pages={26553--26581}
}

@inproceedings{liang2026cautious,
    title={Cautious Optimizers: Improving Training with One Line of Code},
    author={Kaizhao Liang and Lizhang Chen and Bo Liu and Qiang Liu},
    booktitle={International Conference on Learning Representations (ICLR)},
    year={2026},
    pages = {106538--106563}
}

@article{goldstein1977optimization,
  title={Optimization of Lipschitz continuous functions},
  author={Goldstein, Allen A},
  journal={Mathematical Programming},
  volume={13},
  number={1},
  pages={14--22},
  year={1977},
  publisher={Springer}
}

@InProceedings{sutskever2013momentum,
  title = 	 {On the importance of initialization and momentum in deep learning},
  author = 	 {Sutskever, Ilya and Martens, James and Dahl, George and Hinton, Geoffrey},
  booktitle = 	 {Proceedings of the 30th International Conference on Machine Learning (ICML)},
  pages = 	 {1139--1147},
  year = 	 {2013},
}

@InProceedings{arpit2017closer,
  title = 	 {A Closer Look at Memorization in Deep Networks},
  author =       {Devansh Arpit and Stanis{\l}aw Jastrz{\k{e}}bski and Nicolas Ballas and David Krueger and Emmanuel Bengio and Maxinder S. Kanwal and Tegan Maharaj and Asja Fischer and Aaron Courville and Yoshua Bengio and Simon Lacoste-Julien},
  booktitle = 	 {Proceedings of the 34th International Conference on Machine Learning (ICML)},
  pages = 	 {233--242},
  year = 	 {2017},
}

@inproceedings{shah2020pitfalls,
 author = {Shah, Harshay and Tamuly, Kaustav and Raghunathan, Aditi and Jain, Prateek and Netrapalli, Praneeth},
 booktitle = {Advances in Neural Information Processing Systems 33 (NeurIPS)},
 pages = {9573--9585},
 title = {The Pitfalls of Simplicity Bias in Neural Networks},
 year = {2020}
}

@article{zhang2024tinyllama,
  title={{TinyLlama}: An open-source small language model},
  author={Zhang, Peiyuan and Zeng, Guangtao and Wang, Tianduo and Lu, Wei},
  journal={arXiv preprint arXiv:2401.02385},
  year={2024}
}

@article{javaheripi2023phi2,
  title={Phi-2: The surprising power of small language models},
  author={Javaheripi, Mojan and Bubeck, S{\'e}bastien and Abdin, Marah and Aneja, Jyoti and Bubeck, Sebastien and Mendes, Caio C{\'e}sar Teodoro and Chen, Weizhu and Del Giorno, Allie and Eldan, Ronen and Gopi, Sivakanth and others},
  journal={Microsoft Research Blog},
  volume={1},
  number={3},
  pages={3},
  year={2023}
}

@misc{meta2024llama3.2_1b,
  title        = {Llama-3.2-1B},
  author       = {{Meta Llama}},
  year         = {2024},
  howpublished = {\url{https://huggingface.co/meta-llama/Llama-3.2-1B}},
}

@article{kamath2025gemma3,
  title={Gemma 3 technical report},
  author={Kamath, Aishwarya and Ferret, Johan and Pathak, Shreya and Vieillard, Nino and Merhej, Ramona and Perrin, Sarah and Matejovicova, Tatiana and Ram{\'e}, Alexandre and Rivi{\`e}re, Morgane and Rouillard, Louis and others},
  journal={arXiv preprint arXiv:2503.19786},
  year={2025},
}

@inproceedings{pan2024lisa,
 author = {Pan, Rui and Liu, Xiang and Diao, Shizhe and Pi, Renjie and Zhang, Jipeng and Han, Chi and Zhang, Tong},
 booktitle = {Advances in Neural Information Processing Systems 37 (NeurIPS)},
 pages = {57018--57049},
 title = {{LISA}: Layerwise Importance Sampling for Memory-Efficient Large Language Model Fine-Tuning},
 year = {2024}
}

@inproceedings{rajbhandari2020zero,
  title={{ZeRO}: memory optimizations toward training trillion parameter models},
  author={Rajbhandari, Samyam and Rasley, Jeff and Ruwase, Olatunji and He, Yuxiong},
  booktitle={Proceedings of the International Conference for High Performance Computing, Networking, Storage and Analysis},
  pages={1--16},
  year={2020}
}

@inproceedings{ren2021zero,
  title={{ZeRO-Offload}: Democratizing billion-scale model training},
  author={Ren, Jie and Rajbhandari, Samyam and Aminabadi, Reza Yazdani and Ruwase, Olatunji and Yang, Shuangyan and Zhang, Minjia and Li, Dong and He, Yuxiong},
  booktitle={Proceedings of the 2021 USENIX Annual Technical Conference (USENIX ATC 2021)},
  pages={551--564},
  year={2021}
}

@InProceedings{tang2021_1bitadam,
  title = 	 {{1-bit Adam: Communication Efficient Large-Scale Training with Adam’s Convergence Speed}},
  author =       {Tang, Hanlin and Gan, Shaoduo and Awan, Ammar Ahmad and Rajbhandari, Samyam and Li, Conglong and Lian, Xiangru and Liu, Ji and Zhang, Ce and He, Yuxiong},
  booktitle = 	 {Proceedings of the 38th International Conference on Machine Learning (ICML)},
  pages = 	 {10118--10129},
  year = 	 {2021},
}

@inproceedings{li2022_1bitlamb,
  title={1-bit {LAMB}: Communication efficient large-scale large-batch training with {LAMB}’s convergence speed},
  author={Li, Conglong and Awan, Ammar Ahmad and Tang, Hanlin and Rajbhandari, Samyam and He, Yuxiong},
  booktitle={2022 IEEE 29th International Conference on High Performance Computing, Data, and Analytics (HiPC)},
  pages={272--281},
  year={2022},
}

@article{shoeybi2019megatron,
  title={{Megatron-LM}: Training multi-billion parameter language models using model parallelism},
  author={Shoeybi, Mohammad and Patwary, Mostofa and Puri, Raul and LeGresley, Patrick and Casper, Jared and Catanzaro, Bryan},
  journal={arXiv preprint arXiv:1909.08053},
  year={2019}
}

@inproceedings{micikevicius2018mixed,
    title={Mixed Precision Training},
    author={Paulius Micikevicius and Sharan Narang and Jonah Alben and Gregory Diamos and Erich Elsen and David Garcia and Boris Ginsburg and Michael Houston and Oleksii Kuchaiev and Ganesh Venkatesh and Hao Wu},
    booktitle={International Conference on Learning Representations (ICLR)},
    year={2018},
}

@inproceedings{castro2025quartet_fp4,
    title={Quartet: Native {FP}4 Training Can Be Optimal for Large Language Models},
    author = {Castro, Roberto and Panferov, Andrei and Tabesh, Rush and Sieberling, Oliver and Chen, Jiale and Nikdan, Mahdi and Ashkboos, Saleh and Alistarh, Dan},
    booktitle={Advances in Neural Information Processing Systems 38 (NeurIPS)},
    year={2025},
    pages = {43552--43572}
}

@article{abecassis2025pretraining_fp4,
  title={Pretraining large language models with {NVFP4}},
  author={Abecassis, Felix and Agrusa, Anjulie and Ahn, Dong and Alben, Jonah and Alborghetti, Stefania and Andersch, Michael and Arayandi, Sivakumar and Bjorlin, Alexis and Blakeman, Aaron and Briones, Evan and others},
  journal={arXiv preprint arXiv:2509.25149},
  year={2025}
}

@article{xi2023int4,
  title={Training transformers with 4-bit integers},
  author={Xi, Haocheng and Li, Changhao and Chen, Jianfei and Zhu, Jun},
  journal={Advances in Neural Information Processing Systems 36 (NeurIPS)},
  pages={49146--49168},
  year={2023}
}

@inproceedings{liu2024dlion,
 author = {Liu, Bo and Wu, Lemeng and Chen, Lizhang and Liang, Kaizhao and Zhu, Jiaxu and Liang, Chen and Krishnamoorthi, Raghuraman and Liu, Qiang},
 booktitle = {Advances in Neural Information Processing Systems 37 (NeurIPS)},
 pages = {18388--18415},
 title = {Communication Efficient Distributed Training with Distributed {Lion}},
 year = {2024}
}

@article{li2023qft,
  title={{QFT}: Quantized full-parameter tuning of {LLMs} with affordable resources},
  author={Li, Zhikai and Liu, Xiaoxuan and Zhu, Banghua and Dong, Zhen and Gu, Qingyi and Keutzer, Kurt},
  journal={arXiv preprint arXiv:2310.07147},
  year={2023}
}

@inproceedings{narayan2025fp8unit,
  title = 	 {{$\mathrm{\mu}$}nit Scaling: Simple and Scalable {FP8} LLM Training},
  author =       {Narayan, Saaketh and Gupta, Abhay and Paul, Mansheej and Blalock, Davis},
  booktitle = 	 {Proceedings of the 42nd International Conference on Machine Learning (ICML)},
  pages = 	 {45720--45736},
  year = 	 {2025},
}

@article{schlotthauer2025pre,
  title={Pre-Training {LLMs} on a budget: A comparison of three optimizers},
  author={Schlotthauer, Joel and Kroos, Christian and Hinze, Chris and Hangya, Viktor and Hahn, Luzian and K{\"u}ch, Fabian},
  journal={arXiv preprint arXiv:2507.08472},
  year={2025}
}

@inproceedings{nair2010relu,
  title={Rectified linear units improve restricted boltzmann machines},
  author={Nair, Vinod and Hinton, Geoffrey E},
  booktitle={Proceedings of the 27th International Conference on Machine Learning (ICML)},
  pages={807--814},
  year={2010}
}

@article{JMLR2022switchtransformer,
  author  = {William Fedus and Barret Zoph and Noam Shazeer},
  title   = {Switch Transformers: Scaling to Trillion Parameter Models with Simple and Efficient Sparsity},
  journal = {Journal of Machine Learning Research (JMLR)},
  year    = {2022},
  volume  = {23},
  number  = {120},
  pages   = {1--39},
}

@inproceedings{shazeer2017moe,
    title={Outrageously Large Neural Networks: The Sparsely-Gated Mixture-of-Experts Layer},
    author={Noam Shazeer and *Azalia Mirhoseini and *Krzysztof Maziarz and Andy Davis and Quoc Le and Geoffrey Hinton and Jeff Dean},
    booktitle={International Conference on Learning Representations (ICLR)},
    year={2017},
}

@article{ajalloeian2020bias,
  title={On the convergence of {SGD} with biased gradients},
  author={Ajalloeian, Ahmad and Stich, Sebastian U},
  journal={arXiv preprint arXiv:2008.00051},
  year={2020}
}

@inproceedings{cutkosky2019storm,
 author = {Cutkosky, Ashok and Orabona, Francesco},
 booktitle = {Advances in Neural Information Processing Systems 32 (NeurIPS)},
 pages = {15236--15245},
 title = {Momentum-Based Variance Reduction in Non-Convex {SGD}},
 year = {2019}
}

@inproceedings{riedmiller1993rprop,
  title={A direct adaptive method for faster backpropagation learning: The {RPROP} algorithm},
  author={Riedmiller, Martin and Braun, Heinrich},
  booktitle={IEEE International Conference on Neural Networks},
  pages={586--591},
  year={1993},
  organization={IEEE}
}

@InProceedings{jordan2023deterministic,
  title = 	 {Deterministic Nonsmooth Nonconvex Optimization},
  author =       {Jordan, Michael and Kornowski, Guy and Lin, Tianyi and Shamir, Ohad and Zampetakis, Manolis},
  booktitle = 	 {Proceedings of 36th Conference on Learning Theory (COLT)},
  pages = 	 {4570--4597},
  year = 	 {2023},
}

@InProceedings{chen2023faster,
  title = 	 {Faster Gradient-Free Algorithms for Nonsmooth Nonconvex Stochastic Optimization},
  author =       {Chen, Lesi and Xu, Jing and Luo, Luo},
  booktitle = 	 {Proceedings of the 40th International Conference on Machine Learning (ICML)},
  pages = 	 {5219--5233},
  year = 	 {2023},
}

@inproceedings{lin2022gradient,
 author = {Lin, Tianyi and Zheng, Zeyu and Jordan, Michael},
 booktitle = {Advances in Neural Information Processing Systems 35 (NeurIPS)},
 pages = {26160--26175},
 title = {Gradient-Free Methods for Deterministic and Stochastic Nonsmooth Nonconvex Optimization},
 year = {2022}
}

@InProceedings{liu2024zeroth,
  title = 	 {Zeroth-Order Methods for Constrained Nonconvex Nonsmooth Stochastic Optimization},
  author =       {Liu, Zhuanghua and Chen, Cheng and Luo, Luo and Low, Bryan Kian Hsiang},
  booktitle = 	 {Proceedings of the 41st International Conference on Machine Learning (ICML)},
  pages = 	 {30842--30872},
  year = 	 {2024},
}

@inproceedings{liu2024gradient,
 author = {Liu, Zhuanghua and Luo, Luo and Low, Bryan Kian Hsiang},
 booktitle = {Advances in Neural Information Processing Systems 37 (NeurIPS)},
 pages = {45438--45461},
 title = {Gradient-Free Methods for Nonconvex Nonsmooth Stochastic Compositional Optimization},
 year = {2024}
}

@InProceedings{ahn2024adam,
  title = 	 {Understanding {A}dam Optimizer via Online Learning of Updates: {A}dam is {FTRL} in Disguise},
  author =       {Ahn, Kwangjun and Zhang, Zhiyu and Kook, Yunbum and Dai, Yan},
  booktitle = 	 {Proceedings of the 41st International Conference on Machine Learning (ICML)},
  pages = 	 {619--640},
  year = 	 {2024},
}

@inproceedings{dettmers2022_8bit,
    title={8-bit Optimizers via Block-wise Quantization},
    author={Tim Dettmers and Mike Lewis and Sam Shleifer and Luke Zettlemoyer},
    booktitle={International Conference on Learning Representations (ICLR)},
    year={2022},
}

@inproceedings{hernandez-cano2025fp8gemm,
    title={Towards Fully {FP}8 {GEMM} {LLM} Training at Scale},
    author={Alejandro Hern{\'a}ndez-Cano and Dhia Garbaya and Imanol Schlag and Martin Jaggi},
    booktitle={Advances in Neural Information Processing Systems (NeurIPS)},
    year={2025},
    pages = {56676--56701}
}

@article{hao2025lowprecisionsurvey,
  title={Low-Precision Training of Large Language Models: Methods, Challenges, and Opportunities},
  author={Hao, Zhiwei and Guo, Jianyuan and Shen, Li and Luo, Yong and Hu, Han and Wang, Guoxia and Yu, Dianhai and Wen, Yonggang and Tao, Dacheng},
  journal={arXiv preprint arXiv:2505.01043},
  year={2025}
}

@article{mishra2025mxfp8,
  title={Recipes for Pre-training LLMs with {MXFP8}},
  author={Asit Mishra and Dusan Stosic and Simon Layton and Paulius Micikevicius},
  journal={arXiv preprint arXiv:2506.08027},
  year={2025}
}

@article{zhang2026random,
  title={Random Scaling and Momentum for Non-smooth Non-convex Optimization}, 
  author={Qinzi Zhang and Ashok Cutkosky},
  year={2026},
  journal={arXiv preprint arXiv:2405.09742v2},
}

@inproceedings{abernethy2008optimal,
  author       = {Jacob D. Abernethy and
                  Peter L. Bartlett and
                  Alexander Rakhlin and
                  Ambuj Tewari},
  title        = {Optimal Strategies and Minimax Lower Bounds for Online Convex Games},
  booktitle    = {Proceedings of the 21st Annual Conference on Learning Theory (COLT)},
  pages        = {415--424},
  year         = {2008},
}

@misc{yu2026finite,
    title = {Finite-Time Nonsmooth Convergence of Signum and Muon},
    author = {Yu, Dingzhi and Liu, Yuxing and Pan, Rui and Zhang, Lijun and You, Yang and Zou, Difan},
    year = {2026},
    url = {https://www.lamda.nju.edu.cn/yudz/pdf/Finite-Time_Nonsmooth_Convergence_of_Signum_and_Muon.pdf}
}

@article{tao2026when,
  title={{When and Why SignSGD Outperforms SGD: A Theoretical Study Based on $\ell_1$-norm Lower Bounds}},
  author={Tao, Hongyi and Yu, Dingzhi and Zhang, Lijun},
  journal={arXiv preprint arXiv:2605.06615},
  year={2026}
}

@inproceedings{NeurIPS:2025:Wang:A,
    author = {Yibo Wang and Hai-Long Sun and Guangda Huzhang and Qing-Guo Chen and Zhao Xu and Weihua Luo and Kaifu Zhang and Lijun Zhang},
    title = {Triplets Better Than Pairs: Towards Stable and Effective Self-Play Fine-Tuning for LLMs},
    booktitle = {Advances in Neural Information Processing Systems 38 (NeurIPS)},
    pages = {40980--41009},
    year = {2025},
}

@inproceedings{NeurIPS:2025:Wang:B,
    author = {Yibo Wang and Guangda Huzhang and Qing-Guo Chen and Zhao Xu and Weihua Luo and Kaifu Zhang and Lijun Zhang},
    title = {SPACE: Noise Contrastive Estimation Stabilizes Self-Play Fine-Tuning for Large Language Models},
    booktitle = {Advances in Neural Information Processing Systems 38 (NeurIPS)},
    pages = {60042--60072},
    year = {2025},
}

@inproceedings{li2026the,
    title={The Implicit Bias of Steepest Descent with Mini-batch Stochastic Gradient},
    author={Jichu Li and Xuan Tang and Difan Zou},
    booktitle={Proceedings of the 43rd International Conference on Machine Learning (ICML)},
    year={2026},
    pages={to appear}
}

@inproceedings{zhang2026scaling,
    title={Scaling Laws for Precision in High-Dimensional Linear Regression},
    author={Dechen Zhang and Xuan Tang and Yingyu Liang and Difan Zou},
    booktitle={Proceedings of the 43rd International Conference on Machine Learning (ICML)},
    year={2026},
    pages={to appear}
}

@inproceedings{tang2026floating,
    title={A Convergence Analysis of Adaptive Optimizers under Floating-point Quantization},
    author={Xuan Tang and Jichu Li and Difan Zou},
    booktitle={International Conference on Learning Representations (ICLR)},
    year={2026},
    pages={4069--4127}
}

@inproceedings{zhang2026learning,
    title={Learning under Quantization for High-Dimensional Linear Regression},
    author={Dechen Zhang and Junwei Su and Difan Zou},
    booktitle={International Conference on Learning Representations (ICLR)},
    year={2026},
    pages={26307--26374}
}

@InProceedings{chen2025tetrajetv1,
  title = 	 {Oscillation-Reduced {MXFP}4 Training for Vision Transformers},
  author =       {Chen, Yuxiang and Xi, Haocheng and Zhu, Jun and Chen, Jianfei},
  booktitle = 	 {Proceedings of the 42nd International Conference on Machine Learning (ICML)},
  pages = 	 {9400--9414},
  year = 	 {2025},
}

@inproceedings{chen2026tetrajetv2,
    title={TetraJet-v2: Accurate {NVFP}4 Training for Large Language Models with Oscillation Suppression and Outlier Control},
    author={Yuxiang Chen and Yifan Liu and Xiaoming Xu and Pengle Zhang and Michael Beyer and Martin Rapp and Jun Zhu and Jianfei Chen},
    booktitle={Proceedings of the 43rd International Conference on Machine Learning (ICML)},
    year={2026},
    pages={to appear}
}

@article{olmo2024olmo2,
  title={2 OLMo 2 Furious},
  author={OLMo, Team and Walsh, Pete and Soldaini, Luca and Groeneveld, Dirk and Lo, Kyle and Arora, Shane and Bhagia, Akshita and Gu, Yuling and Huang, Shengyi and Jordan, Matt and others},
  journal={arXiv preprint arXiv:2501.00656},
  year={2024}
}

@article{clark2018arc,
  title={Think you have solved question answering? try {ARC}, the {AI2} reasoning challenge},
  author={Clark, Peter and Cowhey, Isaac and Etzioni, Oren and Khot, Tushar and Sabharwal, Ashish and Schoenick, Carissa and Tafjord, Oyvind},
  journal={arXiv preprint arXiv:1803.05457},
  year={2018}
}

@inproceedings{mihaylov2018openbookqa,
  title={Can a Suit of Armor Conduct Electricity? A New Dataset for Open Book Question Answering},
  author={Mihaylov, Todor and Clark, Peter and Khot, Tushar and Sabharwal, Ashish},
  booktitle={Proceedings of the 2018 Conference on Empirical Methods in Natural Language Processing (EMNLP)},
  pages={2381--2391},
  year={2018}
}

@inproceedings{zellers2019hellaswag,
  title={{HellaSwag}: Can a Machine Really Finish Your Sentence?},
  author={Zellers, Rowan and Holtzman, Ari and Bisk, Yonatan and Farhadi, Ali and Choi, Yejin},
  booktitle={Proceedings of the 57th Annual Meeting of the Association for Computational Linguistics (ACL)},
  pages={4791--4800},
  year={2019}
}

@inproceedings{bisk2020piqa,
  title={{PIQA}: Reasoning about Physical Commonsense in Natural Language},
  author={Bisk, Yonatan and Zellers, Rowan and Gao, Jianfeng and Choi, Yejin and others},
  booktitle={Proceedings of the AAAI Conference on Artificial Intelligence (AAAI)},
  volume={34},
  pages={7432--7439},
  year={2020}
}

@inproceedings{sap2019socialiqa,
  title={{Social IQa}: Commonsense Reasoning about Social Interactions},
  author={Sap, Maarten and Rashkin, Hannah and Chen, Derek and Le Bras, Ronan and Choi, Yejin},
  booktitle={Proceedings of the 2019 Conference on Empirical Methods in Natural Language Processing (EMNLP)},
  pages={4463--4473},
  year={2019}
}

@inproceedings{clark2019boolq,
  title={{BoolQ}: Exploring the Surprising Difficulty of Natural Yes/No Questions},
  author={Clark, Christopher and Lee, Kenton and Chang, Ming-Wei and Kwiatkowski, Tom and Collins, Michael and Toutanova, Kristina},
  booktitle={Proceedings of the 2019 Conference of the North American Chapter of the Association for Computational Linguistics: Human Language Technologies (NAACL-HLT)},
  pages={2924--2936},
  year={2019}
}

@article{xu2026deepseekv4,
  title={{DeepSeek-V4}: Towards highly efficient million-token context intelligence},
  author={Xu, Anyi and Lin, Bangcai and Xue, Bing and Wang, Bingxuan and Xu, Bingzheng and Wu, Bochao and Zhang, Bowei and Lin, Chaofan and Dong, Chen and Ling, Chenchen and others},
  journal={arXiv preprint arXiv:2606.19348},
  year={2026}
}

@article{kimiteam2026kimik3,
  title={{Kimi K3}: Open Frontier Intelligence}, 
  author={Kimi Team and Tongtong Bai and Yifan Bai and Yiping Bao and M. C. and Jianfeng Cai and Xinyuan Cai and Peizhou Cao and Yuxuan Cao and Ziwei Chai and others},
  year={2026},
  journal={arXiv preprint arXiv:2607.24653},
}
\bibliographystyle{plainnat}
\newpage
\appendix

\section{Related Work\label{sec:related-work}}

\paragraph{Sign-based optimization}
SignSGD (and its momentum variant Signum) was first introduced by~\citet{bernstein2018signsgd}, who established an $O(1/T+1/B)$ convergence rate for smooth non-convex objectives, where $T$ denotes the number of iterations and $B$ the batch size. This result suggests that the convergence of SignSGD relies heavily on large-batch training regimes. In non-smooth settings,~\citet{karimireddy2019error,xiao2023stochastic} develop convex and deterministic counterexamples where SignSGD diverges.~\citet{yu2026finite} recently refine previous deterministic counterexamples and also construct stochastic hard instances to argue that even with momentum (Signum and Lion), SignSGD still fails to achieve finite-time convergence. Owing to its communication efficiency, SignSGD subsequently attracted extensive attention in distributed learning settings~\citep{bernstein2019signsgd,jin2020stochastic,safaryan2021stochastic}. The general idea of injecting noise into the sign operator appears at~\citep{jin2020stochastic,safaryan2021stochastic,sun2023momentum,chzhen2023signsvrg,NeurIPS:2024:Jiang,jiang2025improved,jiang2026lion}. However, most of them neither consider the coordinate-wise noise nor do they allow the noise level to adapt to the optimization trajectory. These drawbacks lead to dimension-dependent convergence rates as well as stronger assumptions, such as bounded gradients. 

Several works have further extended or reinterpreted sign-based methods. For example,~\citet{crawshaw2022robustness} proposed a generalized version of SignSGD that is closely related to Adam, complementing another line of work that seeks to explain Adam's effectiveness from the perspective of sign descent~\citep{balles2018dissecting,kunstner2023noise}. Relatedly,~\citet{li2026the} studied the implicit bias of mini-batch stochastic steepest descent, covering SignSGD as the $\ell_\infty$-instance, and characterized how batch size, momentum, and variance reduction affect its max-margin behavior. On the theoretical side,~\citet{sun2023momentum} derived a tighter $O(1/\sqrt{BT})$ convergence for Signum, albeit with less favorable dimension dependence. This analysis was later improved by~\citet{jiang2025improved}, who obtained better dimension factors under weaker assumptions. Beyond vanilla sign descent,~\citet{chzhen2023signsvrg} proposed SignSVRG, which combines sign-based updates with variance reduction techniques~\citep{NIPS2013svrg,NIPS:2013:Zhang}. This direction was further advanced by~\citet{NeurIPS:2024:Jiang}, who leveraged the STORM estimator~\citep{cutkosky2019storm} to achieve the optimal convergence rate in $T$, and extended the method to distributed settings with significantly improved rates over prior works.~\citet{yang2026distributed} recently initiated the study of distributed OCO with compressed communication, establishing matching lower and upper bounds and thus achieving minimax-optimal regret rates for (strongly) convex objectives.

Empirically, the origins of sign-based optimization can be traced back to RProp~\citep{riedmiller1993rprop}. A notable recent breakthrough in this domain is the Lion optimizer~\citep{chen2023symbolic}, discovered through symbolic search. Built upon Signum, Lion maintains two momentum buffers with decoupled roles, assigning greater importance to the current gradient. Through a range of large-scale training tasks,~\citet{chen2023symbolic} showed that Lion can converge substantially faster than AdamW while achieving comparable or better generalization. On the theoretical side,~\citet{dong2024convergence} provided the first general convergence analysis of Lion, which was later sharpened by~\citet{jiang2026lion} under weaker assumptions and further extended to distributed optimization. More recently,~\citet{yu2026signheavytails} is the first to theoretically justify the potential benefits of Signum and Lion over AdamW, through the lens of heavy-tailed stochasticity in LLM training. Their analysis establishes convergence under generalized heavy-tailed noise and removes the dimensional dependence present in prior works~\citep{dong2024convergence,jiang2026lion}. More recently,~\citet{tao2026when} investigate when and why SignSGD will outperform SGD based on a strong theoretical comparison through geometry-aware optimal lower bounds.

\paragraph{Non-smooth optimization}
In convex settings, non-smooth optimization arises naturally in online convex optimization~\citep{zinkevich2003oco} and has been studied extensively over the years~\citep{hazan2019introduction,orabona2019intro}. In the offline setting, there is likewise a rich literature on non-smooth convex optimization~\citep{bubeck2015convex,hazan2019lecture}. Transitioning to non-convex settings, the seminal work of~\citet{zhang2020complexity} provided the first finite-time analysis of stochastic gradient methods based on their newly proposed $(\delta,\epsilon)$-Goldstein stationary points. Such measures are essential, since~\citet{kornowski2022oracle} established hardness results for finding even near approximately-stationary points. Moreover,~\citet{jordan2023deterministic} showed that no deterministic algorithm can identify $(\delta,\epsilon)$-Goldstein points in dimension-free time. Beyond first-order methods, gradient-free approaches have also been proposed and analyzed for this challenging setting~\citep{lin2022gradient,chen2023faster,liu2024zeroth,liu2024gradient}.

A recent major breakthrough is the online-to-non-convex conversion (O2NC) framework proposed by~\citet{cutkosky2023optimal}. The key idea is to use an online learning algorithm to predict and determine the next update step within the conversion procedure. Remarkably,~\citet{cutkosky2023optimal} showed that an optimal regret bound for online learning immediately yields an optimal complexity bound for non-smooth non-convex optimization. Building on this framework,~\citet{zhang2024random} introduced exponential random scaling into O2NC and showed that SGD with momentum can efficiently find $(\delta,\epsilon)$-stationary points.~\citet{ICML:2024:Liu} elegantly incorporated gradient clipping into this framework and established the first rate-optimal high-probability guarantees under heavy-tailed gradients. These results were subsequently refined by~\citet{liu2026old}, who developed an elegant analysis of OGD itself under heavy-tailed noise. More broadly, the O2NC framework has also been utilized to justify the theoretical benefits of Adam~\citep{ahn2024adam,ahn2024adamema} and schedule-free SGD~\citep{ahn2025general}, and has recently enabled a deeper understanding of quasi-Newton and optimistic methods~\citep{jiangSTOC2025improved,patitucci2026improving}.

\paragraph{Low-precision training}
Low-precision training has progressed from mixed-precision recipes that keep selected tensors in higher precision to increasingly aggressive native low-bit regimes. Early work showed that stable training can be achieved by combining low-precision arithmetic with loss scaling and higher-precision master parameters~\citep{micikevicius2018mixed}, while later studies developed hardware-friendly low-bit formats, including hybrid FP8 designs and more systematic 8-bit floating-point specifications~\citep{NEURIPS2019hybrid_fp8,noune20228bit,micikevicius2022fp8}. A recent survey provides a comprehensive overview of methods, challenges, and opportunities in low-precision LLM training~\citep{hao2025lowprecisionsurvey}. Complementing these empirical developments, recent theoretical work has studied learning and precision scaling under quantization in high-dimensional linear regression~\citep{zhang2026learning,zhang2026scaling}, as well as the convergence of adaptive optimizers under floating-point quantization~\citep{tang2026floating}.

Building on these foundations, recent works have demonstrated the practical viability of low-precision training for LLMs.~\citet{perez2023training} studied dynamic scaling strategies for FP8 training and inference of GPT- and Llama-style models, while~\citet{peng2023fp8} proposed an FP8 automatic mixed-precision framework in which gradients and optimizer states can also be quantized effectively.~\citet{ICLR2025fp8scaling} further scaled FP8 training to trillion-token LLMs and revealed long-horizon instabilities that do not appear in shorter runs. Meanwhile, optimizer-state compression has become an important complementary direction:~\citet{dettmers2022_8bit} developed 8-bit optimizers via block-wise quantization, and~\citet{ICLR2025coat} further compressed optimizer states and activations to improve the memory efficiency of FP8 training. Beyond FP8, low-precision training has also expanded to more aggressive formats:~\citet{xi2023int4} developed hardware-compatible INT4 transformer training,~\citet{castro2025quartet_fp4,abecassis2025pretraining_fp4} showed that end-to-end FP4/NVFP4 training can be competitive for billion-scale LLM pretraining,~\citet{mishra2025mxfp8} presented practical recipes for pre-training LLMs with MXFP8, and~\citet{hernandez-cano2025fp8gemm} pushed toward fully FP8 GEMM LLM training at scale. Along a complementary line, the TetraJet series tackles the weight oscillation induced by low-bit quantization, first reducing it for MXFP4 training of vision transformers~\citep{chen2025tetrajetv1} and then extending oscillation suppression and outlier control to NVFP4 LLM pretraining~\citep{chen2026tetrajetv2}, whose recipe we adopt for our FP4 experiments. Industrial technical reports such as DeepSeek-V3 and Qwen3 further indicate that low-precision training is becoming an important systems direction for modern LLM development~\citep{liu2024deepseek,yang2025qwen3}.

At the same time, another line of work studies why low-precision training fails and how to stabilize it.~\citet{wortsman2023stable} identified loss spikes associated with under-estimated AdamW second moments in large-scale low-precision training.~\citet{lee2024fp8} systematically quantified the degradation in LLM training stability caused by aggressive precision reduction, while~\citet{balancca2024scalify} formalized scale propagation as an end-to-end recipe for low-precision LLM training. Most recently,~\citet{qiu2026why} traced an important failure mode of low-precision Transformer training to FlashAttention~\citep{dao2022flashattention,daoICLR2024flashattention2,shah2024flashattention3}. In contrast to these works, our goal is not to design yet another stabilization recipe for AdamW, but to show that replacing AdamW with a sign-based optimizer can itself provide a simple and robust alternative in highly constrained low-precision regimes.

\section{Experimental Details}

\subsection{Mathematical Reasoning\label{sec:math-reasoning-cont}}

We present the omitted details in \cref{sec:math-reasoning}, including the training and evaluation setup for full-parameter finetuning as well as the accuracies of subcategories in the MATH dataset. We also supplement with LoRA finetuning results to support the main discovery from SFT experiments.

\subsubsection{Full Finetuning: Training and Evaluation Configurations\label{sec:sft-configuration}}

For SFT, we adopt a global batch size of 16, with 16 gradient accumulation steps. The max sequence length is set to 2048 according to the problem distributions in~\citet{numina_math_datasets}. We employ the standard cosine decay learning rate schedule~\citep{loshchilov2017sgdr}, in which the learning rate decays to 0 at the end of training. The optimizer configurations for \cref{tab:numinamath-results} are listed in \cref{tab:optimizer-hp-qwen}. All other settings remain at their \texttt{transformers==4.52.4}~\citep{wolf-etal-2020-transformers} defaults. To cut memory usage and accelerate training, we utilize the LMFlow toolbox~\citep{diao2024lmflow}.

\begin{table}[htbp]
\centering
\caption{Optimizer hyperparameters for finetuning Qwen2.5-7B\label{tab:optimizer-hp-qwen}}
\centering
\begin{tabular}{lccccc}
\toprule
\textbf{Optimizer} & $\eta$ & $\beta_1$ & $\beta_2$ & $\lambda$ & seed\\
\midrule
AdamW   & 1e-5 & 0.9 & 0.999 & 0.1 & 71\\
SignSGD & 3e-6 & 0.9 & -- & 0.1  & 71\\
StoSignSGD & 4e-6 & 0.9 & -- & 0.1 & 71\\
\bottomrule
\end{tabular}
\end{table}
For evaluation, we leverage the \texttt{lm-evaluation-harness} framework~\citep{eval-harness} and follow mostly from the configurations in~\citet{pan-etal-2025-scalebio}. Specifically, we set \texttt{max\_new\_tokens} to 2048 (the same as sequence length), temperature to 0.0, and \texttt{top\_p} to 1.0. The rest of the parameters remain at \texttt{lm\_eval==0.4.9.1} defaults. Inferences are run under a batch size of 32, and we report zero-shot accuracies.

Below, we also report the evaluation accuracy of seven problem subclasses of the MATH dataset in \cref{tab:math-subcategory-results}, which serves as a supplement to \cref{tab:numinamath-results}. \cref{tab:math-subcategory-results} shows that the improvement of StoSignSGD primarily stems from the improvement on relatively easy data, including algebra, prealgebra, and number\_theory, the top-3 categories of MATH in terms of accuracy. This crucial discovery aligns with the result in \cref{tab:numinamath-results}, where StoSignSGD exhibits a significant edge over SignSGD and AdamW on GSM8k, which consists of grad-school-level math questions. On the contrary, the MATH dataset consists of more challenging mathematics competition problems\footnote{According to~\citet{hendrycks2021math}, a computer science PhD student who does not especially like mathematics attained
approximately 40\% on MATH, while a three-time IMO gold medalist attained 90\%.}. In the sequel, we conduct low-rank adaptation finetuning of LLMs to support our discovery with more empirical evidence.

\begin{table}[htbp]
    \centering
    \caption{Zero-shot accuracy (\%) on GSM8k and MATH. Extended version of \cref{tab:numinamath-results}. ``$\pm$'' represents standard error automatically calculated by \texttt{lm\_eval}.}
    \label{tab:math-subcategory-results}
    \begin{tabular}{lccc}
        \toprule
        \textbf{Optimizer} & \textbf{AdamW} &\textbf{SignSGD} &  \textbf{StoSignSGD} \\
        \midrule
        GSM8k & 74.37 \(\pm\) 1.20 & 71.95 \(\pm\) 1.24 & \textbf{77.33 \(\pm\) 1.15} \\\midrule
        MATH (overall)    &48.86 \(\pm\) 0.66   & 47.66 \(\pm\) 0.66&\textbf{48.88 \(\pm\) 0.65}  \\\midrule
        algebra         & 70.60 $\pm$ 1.32 & 68.91 $\pm$ 1.34 & \textbf{71.19 $\pm$ 1.32}  \\
        prealgebra  & 65.33 $\pm$ 1.61 & 63.49 $\pm$ 1.63 & \textbf{65.44 $\pm$ 1.61}  \\
        num\_theory & 54.81 $\pm$ 2.14 & 54.44 $\pm$ 2.15 & \textbf{57.59 $\pm$ 2.13} \\
        counting\_and\_prob     & \textbf{40.51 $\pm$ 2.26} & 40.08 $\pm$ 2.25 & 40.08 $\pm$ 2.25 \\
        math\_geometry    & \textbf{35.70 $\pm$ 2.19} & 32.78 $\pm$ 2.15 & 35.49 $\pm$ 2.19 \\
        intermediate\_algebra    & \textbf{29.13 $\pm$ 1.51} & 28.02 $\pm$ 1.50 & 27.35 $\pm$ 1.48 \\
        precalc & 20.88 $\pm$ 1.74 & \textbf{21.67 $\pm$ 1.76}& 20.33 $\pm$ 1.72\\
        \bottomrule
    \end{tabular}
\end{table}

\subsubsection{LoRA Finetuning: StoSignSGD Improves Sample Efficiency on Basic Problems\label{sec:lora}}

As a typical representative of parameter-efficient fine-tuning (PEFT) methods, low-rank adaptation (LoRA)~\citep{hu2022lora} has become a standard practice in finetuning LLMs~\citep{zheng-etal-2024-llamafactory,diao2024lmflow}. Following the full-parameter finetuning experiments in \cref{sec:math-reasoning}, we conduct the corresponding LoRA finetuning experiments to broaden the applicability of StoSignSGD. In the meantime, we present more empirical evidence to support that StoSignSGD can better enhance the model's mathematical reasoning ability for foundational problems.

We conducted the LoRA experiments on a single A100 GPU, largely retaining the training, evaluation, and dataset configurations from the SFT setup detailed in \cref{sec:sft-configuration}. Due to hardware constraints, the global batch size was restricted to 1, with gradient accumulation steps also set to 1 to accelerate training. Notably, this intentionally establishes a challenging, small-batch regime that is known to be unfavorable for sign-based optimizers like SignSGD and StoSignSGD~\citep{bernstein2018signsgd,wen2026fantastic}. For evaluation, we utilized a batch size of 16 and a maximum generation length of 1024 tokens. Regarding the LoRA-specific configurations, we configured the low-rank matrices with a rank ($r$) of 256, a scaling factor ($\alpha$) of 256, and a dropout probability of 0.1. The final hyperparameters for each optimizer are listed in \cref{tab:optimizer-hp-qwen-lora}, which were determined via a grid search over the learning rate $\eta \in \{10^{-4}, 10^{-5}, 10^{-6}, 10^{-7}\}$, and momentum parameters $\beta_1 \in \{0.9, 0.95, 0.99\}$ and $\beta_2 \in \{0.95, 0.99, 0.999\}$.

\begin{table}[htbp]
\centering
\caption{Optimizer hyperparameters for LoRA finetuning Qwen2.5-7B\label{tab:optimizer-hp-qwen-lora}}
\centering
\begin{tabular}{lccccc}
\toprule
\textbf{Optimizer} & $\eta$ & $\beta_1$ & $\beta_2$ & $\lambda$ & seed\\
\midrule
AdamW   & 1e-5 & 0.9 & 0.999 & 0.1 & 43\\
SignSGD & 1e-6 & 0.9 & -- & 0.1  & 43\\
StoSignSGD & 1e-5 & 0.9 & -- & 0.1 & 43\\
\bottomrule
\end{tabular}
\end{table}

We present the zero-shot evaluation results in \cref{tab:lora} and highlight three key takeaways. First, \textbf{StoSignSGD yields performance improvements of 2.5\% on GSM8k and 0.5\% on MATH}, confirming its continued efficacy in parameter-efficient fine-tuning. Second, \textbf{StoSignSGD exhibits strong resilience to small batch sizes}, a setting where standard SignSGD suffers significant degradation. This confirms that StoSignSGD effectively mitigates a critical vulnerability of traditional sign-based methods. Finally, \textbf{the performance boosts are most pronounced in foundational mathematical tasks} (including GSM8k and the algebra, prealgebra, number theory, and counting/probability subsets of MATH), suggesting that StoSignSGD facilitates the learning of fundamental mathematical concepts with greater sample efficiency.

\begin{table}[htbp]
    \centering
    \caption{Zero-shot accuracy (\%) on GSM8k and MATH of Qwen2.5-7B trained by LoRA.}
    \label{tab:lora}
    \begin{tabular}{lccc}
        \toprule
        \textbf{Optimizer} & \textbf{AdamW} &\textbf{SignSGD} &  \textbf{StoSignSGD} \\
        \midrule      
        GSM8k & 60.05 \(\pm\) 1.35 & 48.07 \(\pm\) 1.38 & \textbf{62.55 \(\pm\) 1.33} \\\midrule
        MATH (overall)         & 48.08 $\pm$ 0.66 & 47.44 $\pm$ 0.66 & \textbf{48.60 $\pm$ 0.66} \\
        \midrule
        algebra                & 69.84 $\pm$ 1.33 & 69.17 $\pm$ 1.34 & \textbf{70.85 $\pm$ 1.32} \\
        prealgebra             & 63.15 $\pm$ 1.64 & 61.77 $\pm$ 1.65 & \textbf{63.38 $\pm$ 1.63} \\
        num\_theory            & 54.26 $\pm$ 2.15 & 56.11 $\pm$ 2.14 & \textbf{56.67 $\pm$ 2.13} \\
        counting\_and\_prob    & 39.66 $\pm$ 2.25 & 36.50 $\pm$ 2.21 & \textbf{40.51 $\pm$ 2.26} \\
        math\_geometry         & 34.45 $\pm$ 2.17 & \textbf{34.66 $\pm$ 2.18} & 33.40 $\pm$ 2.16 \\
        intermediate\_algebra  & \textbf{28.79 $\pm$ 1.51} & 27.35 $\pm$ 1.48 & \textbf{28.79 $\pm$ 1.51} \\
        precalc                & 21.79 $\pm$ 1.77 & \textbf{22.71 $\pm$ 1.79} & 21.79 $\pm$ 1.77 \\
        \bottomrule
    \end{tabular}
\end{table}

\subsubsection{A Closer Look into Momentum}
To back up our findings in \cref{sec:lora} and provide a mechanistic explanation for why StoSignSGD reduces the sample complexity required to learn fundamental concepts, we perform a simple empirical study on momentum and conclude that it is the \textbf{amplified momentum tails that reflect the acquisition of ubiquitous foundational skills}. 

In the context of deep learning, fundamental knowledge and core concepts correspond to structural patterns that appear with high frequency and ubiquity across the training data distribution~\citep{arpit2017closer, shah2020pitfalls}. Because these highly recurrent patterns are present in nearly every mini-batch, their corresponding stochastic gradients consistently point in the same optimization direction. The momentum term, acting as an exponential moving average, constructively accumulates these consistent gradient signals while destructively canceling out the fluctuating, zero-mean noise of batch-specific or rare features~\citep{sutskever2013momentum}. Therefore, the successful acquisition of highly recurrent, foundational "skills" will naturally manifest as a significant amplification of momentum magnitudes along specific coordinate axes.

\begin{figure}[htbp]
    \centering
    \vspace{-10pt}
    \includegraphics[width=\linewidth]{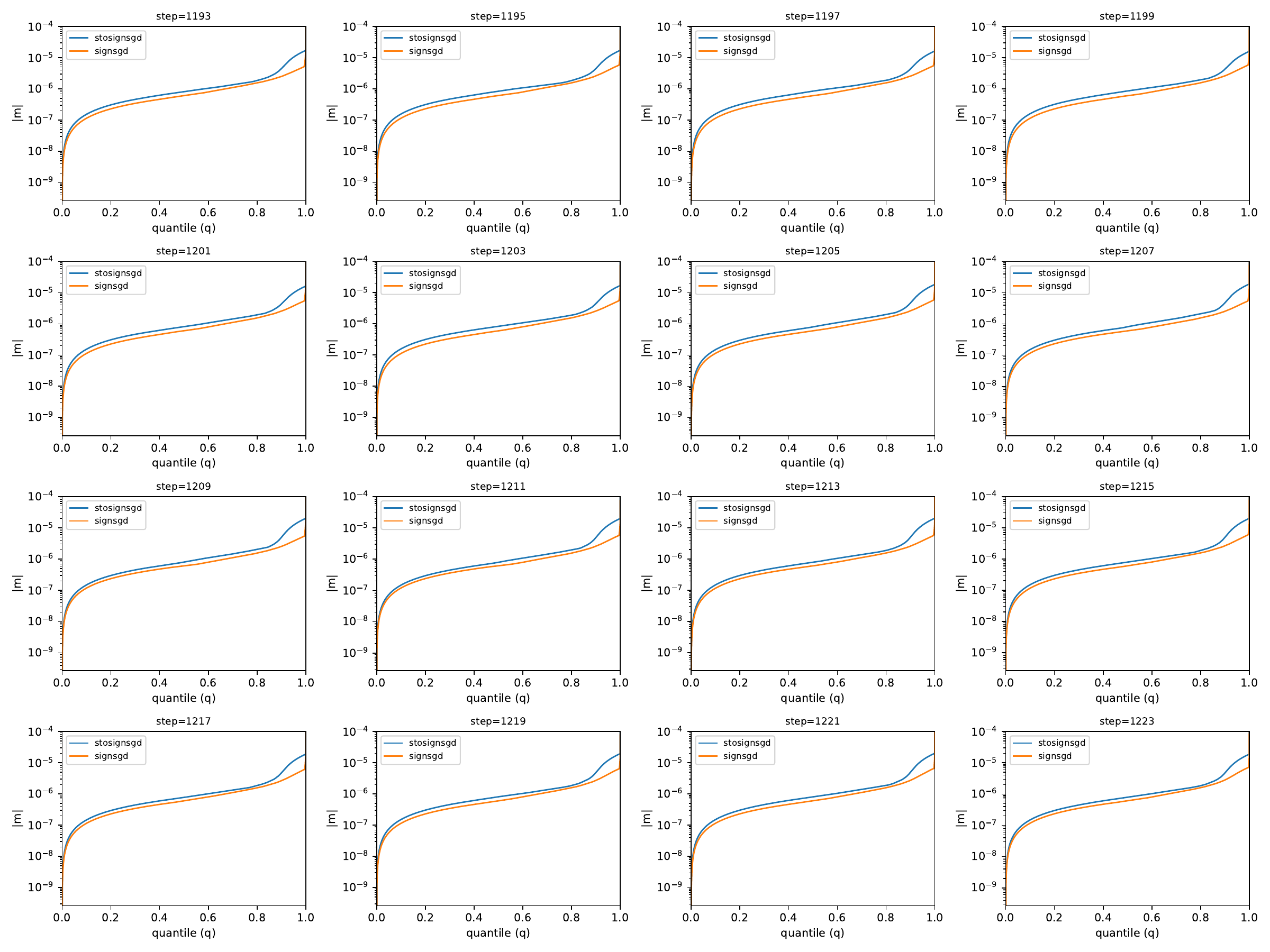}
    \vspace{-20pt}
    \caption{Momentum ($\abs{\m_t}$) distribution at the end of training.\label{fig:momentum}}
\end{figure}

To verify this dynamic, we follow the setup detailed in \cref{sec:dessect-sign-noise} and analyze the momentum distributions during the GPT-2 training trajectory on the Alpaca dataset. We isolate 16 timestamps near the conclusion of training (totaling 1,224 steps) and compute the coordinate-wise absolute values of the momentum vector, $|\m_t|$. Treating each dimension as an individual data point, we construct a value-quantile plot. Comparing StoSignSGD against the standard SignSGD baseline, the results in \cref{fig:momentum} consistently demonstrate that StoSignSGD yields a momentum distribution with significantly larger magnitudes, particularly in the highest echelon ($>0.9$ quantile). This pronounced surge in high-quantile momentum directly indicates that StoSignSGD more effectively locks onto and amplifies the consistent gradient directions associated with ubiquitous data patterns. This geometric behavior perfectly corroborates our empirical findings: by confidently amplifying the most reliable gradient signals, StoSignSGD naturally excels at isolating and reinforcing the foundational rules required for core mathematical reasoning. Ultimately, this momentum-driven amplification serves as a tangible metric for the model's structural learning. Because StoSignSGD successfully shields these consistent, high-frequency signals from being disrupted by batch-specific stochasticity, it locks onto fundamental features much faster than standard methods. This geometric stability provides a direct explanation for our empirical observations: by efficiently isolating and reinforcing these core rules, StoSignSGD achieves superior sample efficiency on tasks that strictly require robust foundational knowledge, such as the mathematical reasoning benchmarks evaluated in GSM8k and MATH.

\subsection{Low-precision Pretraining in FP8\label{sec:fp8-cont}}

\paragraph{Experimental Setup}
Our codebase is adapted from~\citet{qiu2026why}, where we follow the nanoGPT~\citep{nanogpt} implementation (124M parameters) of GPT-2 architecture on the OpenWebText dataset~\citep{Gokaslan2019OpenWeb}, where we use GeLU activations~\citep{hendrycks2016gelu}, disable bias, and set zero Dropout~\citep{JMLR:v15:srivastava14dropout}. The training and validation sets comprise approximately 9B and 4.4B tokens, respectively, and have been preprocessed using the GPT-2 tokenizer. The general hyperparameter configurations are listed in \cref{tab:general-hp-gpt2}. We adopt fewer learning-rate warmup steps and a smaller minimum learning rate than~\citet{hoffmann2022chinchilla} since we use sign-based methods. We train on a total of 13.1B tokens, which corresponds to a 5.3x Chinchilla-optimal ratio~\citep{hoffmann2022chinchilla}. This relatively long training run is usually enough to capture the dynamics between different optimization algorithms~\citep{an2025asgo}. For the optimizer configurations in \cref{tab:optimizer-hp-gpt2}, we set the ratio of the minimum learning rate to the peak learning rate to 1\%, and perform a grid search over $\cbrac{1e-3,5e-4,1e-4,5e-5,1e-5}$. We search $\beta_1$ over $\cbrac{0.9,0.95,0.99}$ and $\beta_2$ over $\cbrac{0.95,0.99,0.999}$. For the FP8 configurations, we use \texttt{torchao}~\citep{or2025torchao} to replace all \texttt{nn.Linear} layers with \texttt{Float8Linear} layers, where scaling is handled automatically. Master weights and non-Linear operations stay in relatively high precision (we use BF16 following common practice~\citep{an2025asgo,qiu2026why}). We then convert all optimizer states into FP8 following~\citet{NEURIPS2019hybrid_fp8,noune20228bit,micikevicius2022fp8}. The gradients are first cast into FP8\_E5M2 before the optimizer sees them, and the remaining momentum buffers are then cast into FP8\_E4M3.

\begin{table}[t]
\centering
\caption{General hyperparameters for pretraining GPT-2\label{tab:general-hp-gpt2}}
\centering
\begin{tabular}{lc}
\toprule
\textbf{Hyperparameter} & Value \\
\midrule
Gradient clipping & 1.0 \\
Learning rate schedule & cosine \\
Warmup steps & 2000 \\
Total steps   & 50000 \\
Sequence length & 1024 \\
Micro-batch size & 16 \\
Number of GPUs & 8 \\
Gradient accumulation steps & 16 \\
Per step tokens & 262144 \\
Total tokens & 13.1B \\
Evaluation interval & 1000 \\
\bottomrule
\end{tabular}
\end{table}

\begin{table}[t]
\centering
\caption{Optimizer hyperparameters for pretraining GPT-2. Muon uses Nesterov momentum with coefficient $0.95$ and, following the AdamW ablation, is additionally swept over a larger weight decay ($\lambda=0.5$) and a larger batch size; it nonetheless diverges to NaN in all settings.\label{tab:optimizer-hp-gpt2}\tablefootnote{We present the configurations for the divergent AdamW and Muon because these hyperparameters determine when they encounter numerical crash. Smaller learning rates will \emph{defer} the occurrence of NaN, but can not rescue it.}}
\centering
\begin{tabular}{lcccccc}
\toprule
\textbf{Optimizer} & $\eta_\text{peak}$ &$\eta_\text{min}$ & $\beta_1$ & $\beta_2$ & $\lambda$ & seed\\
\midrule
AdamW   & 1e-4 & 1e-6 & 0.9 & 0.95 & 0.1 & 61\\
Muon   & 1e-4 & 1e-6 & 0.95 & -- & 0.1 & 61\\
Lion   & 1e-4 & 1e-6 & 0.9 & 0.95 & 0.1 & 61\\
SignSGD & 1e-4 & 1e-6 & 0.9 & -- & 0.1  & 61\\
StoSignSGD & 1e-4 & 1e-6 & 0.9 & -- & 0.1 & 61\\
\bottomrule
\end{tabular}
\end{table}

\begin{table}[htbp]
\centering
\caption{AdamW optimizer state collapse under FP8\_E4M3 quantization at 103rd step. }
\label{tab:adamw_collapse}
\begin{tabular}{lcccc}
\toprule
\textbf{Layer Name} & \textbf{$\v_t$ Zeros (\%)} & \textbf{$\m_t$ Zeros (\%)} & \textbf{Max $|\g_t|$} & \textbf{Max $\g_t\odot\g_t$} \\
\midrule
Token Embedding     & 100.0 & 100.0 & 0.0023 & $5.29 \times 10^{-6}$ \\
Position Embedding  & 100.0 & 100.0 & 0.0004 & $1.60 \times 10^{-7}$ \\
L0: LayerNorm 1     & 100.0 & 100.0 & 0.0002 & $4.00 \times 10^{-8}$ \\
L0: Attn QKV        & 100.0 & 100.0 & 0.0009 & $8.10 \times 10^{-7}$ \\
L0: Attn Proj       & 100.0 & 100.0 & 0.0018 & $3.24 \times 10^{-6}$ \\
L0: LayerNorm 2     & 100.0 & 100.0 & 0.0001 & $1.00 \times 10^{-8}$ \\
L0: MLP FC1         & 100.0 & 100.0 & 0.0003 & $9.00 \times 10^{-8}$ \\
L0: MLP Proj        & 100.0 & 100.0 & 0.0024 & $5.76 \times 10^{-6}$ \\
L1: LayerNorm 1     & 100.0 & 100.0 & 0.0002 & $4.00 \times 10^{-8}$ \\
L1: Attn QKV        & 100.0 & 100.0 & 0.0005 & $2.50 \times 10^{-7}$ \\
L1: Attn Proj       & 100.0 & 100.0 & 0.0039 & $1.52 \times 10^{-5}$ \\
\bottomrule
\end{tabular}
\end{table}

\paragraph{Why AdamW fails in FP8 regimes?}
We locate the failure sample at iteration 103 and inspect the gradients across various layers in the model, where the visualization sits in \cref{fig:adamw_fp8_collapse}. \cref{tab:adamw_collapse} further presents the detailed optimizer states at the 103rd step. For all observed layers, $100\%$ of the variance ($\v_t$) and momentum ($\m_t$) buffer values underflowed to exactly zero. Because the squared gradients $\g_t\odot\g_t$ are strictly smaller than the E4M3 minimum representable subnormal ($\approx 1.95 \times 10^{-3}$), the updates are flushed to zero, causing catastrophic division-by-zero during the parameter update of AdamW. In contrast to the second-order momentum update of AdamW, $\v_t=\beta_2\v_{t-1}+(1-\beta_2)\g_t\odot\g_t$, StoSignSGD only extracts the sign and tracks the maximum absolute gradient magnitude. Since it avoids squaring, it also avoids this underflow cliff. Finally, we note that tuning down the learning rate only \emph{slows down} the loss spike rather than \emph{fixing} it. Setting $\eta_\text{peak}=5e-4$ advances the failure to appear at around step 50. Further tuning down the learning rate is meaningless since the current sign-aligned $\eta_\text{peak},\eta_\text{min}$ is already unfavorable for AdamW~\citep{liu2025muon,wen2026fantastic,liang2026cautious}. Beyond tuning the learning rate, we also further ablate on batch size and weight decay. Specifically, we additionally add a sweep over $\lambda\in\cbrac{0.1,0.5}$ and $\mathtt{global\_batch\_size}\in\cbrac{256,4096}$, while the rest configurations stay the same as in \cref{tab:general-hp-gpt2,tab:optimizer-hp-gpt2}. The divergence case remains similar, and there is no way to ``rescue'' AdamW unless extra stabilization tricks are applied.

\subsection{Low-precision Pretraining in FP4\label{sec:fp4-cont}}

\paragraph{Experimental Setup}
We adopt the OLMo2 architecture and training stack~\citep{olmo2024olmo2} and study three model scales, OLMo2-70M, -150M, and -370M, whose configurations are summarized in \cref{tab:fp4-arch}. All models are decoder-only Transformers with rotary position embeddings, SwiGLU feed-forward layers, and RMSNorm applied after each sublayer (the OLMo2 \emph{norm-after} block), with untied input/output embeddings, a sequence length of $4096$, and the OLMo2 (Dolma2) tokenizer with a vocabulary of $100{,}278$ (padded to $100{,}352$). Their total parameter counts, including embeddings, are \textbf{136M}, \textbf{267M}, and \textbf{474M}, respectively. We pretrain on the official OLMo2 pretraining mixture, \texttt{OLMo-mix-1124}~\citep{olmo2024olmo2},\footnote{Publicly available at \url{https://huggingface.co/datasets/allenai/olmo-mix-1124}.} a Dolma-based blend of DCLM, StarCoder, proof-pile-2, peS2o, and Wikipedia, preprocessed with the Dolma2 tokenizer.

Our FP4 pipeline follows the NVFP4 recipe of TetraJet-v2~\citep{chen2026tetrajetv2}. All \texttt{nn.Linear} layers are trained with NVFP4 GEMMs using double-block quantization (an inner block of $16$ with FP8 E4M3 scales and an outer block of $128$ with an FP32 scale), a random Hadamard transform in the backward pass, and stochastic rounding. As in the FP8 study, we further push both the gradients and the optimizer states into FP4, the most aggressive (\emph{radical}) configuration. Gradients are clipped and then rounded to NVFP4 (with a gradient scale of $1024$ and stochastic rounding) before the optimizer consumes them, and every persistent optimizer buffer is stored as an NVFP4 pack, dequantized to FP32 for the update math and repacked afterwards. Master weights are kept in higher precision on the AMP-BF16, FSDP training path~\citep{micikevicius2018mixed}, and no optimizer-specific stabilization (e.g., a variance floor for AdamW) is applied. This setup is much harsher than reported FP4-pretraining practice~\citep{castro2025quartet_fp4,abecassis2025pretraining_fp4}, which keeps gradients and optimizer states in higher precision.

\begin{table}[t]
\centering
\caption{Model architectures for FP4 pretraining (OLMo2 series).\label{tab:fp4-arch}}
\begin{tabular}{lccc}
\toprule
\textbf{Hyperparameter} & \textbf{OLMo2-70M} & \textbf{OLMo2-150M} & \textbf{OLMo2-370M} \\
\midrule
Total parameters       & 136M & 267M & 474M \\
Hidden size $d_\text{model}$ & 512 & 768 & 1024 \\
Number of layers       & 8 & 12 & 16 \\
Number of heads        & 8 & 12 & 16 \\
FFN hidden size        & 2048 & 3072 & 8192 \\
Activation             & \multicolumn{3}{c}{SwiGLU} \\
Normalization          & \multicolumn{3}{c}{RMSNorm (norm-after)} \\
Position embedding     & \multicolumn{3}{c}{RoPE ($\theta=5\times10^{5}$)} \\
Sequence length        & \multicolumn{3}{c}{4096} \\
Vocabulary size        & \multicolumn{3}{c}{100{,}278} \\
\bottomrule
\end{tabular}
\end{table}

\begin{table}[t]
\centering
\caption{General training hyperparameters for radical FP4 pretraining.\label{tab:general-hp-olmo}}
\begin{tabular}{lccc}
\toprule
\textbf{Hyperparameter} & \textbf{OLMo2-70M} & \textbf{OLMo2-150M} & \textbf{OLMo2-370M} \\
\midrule
Global batch size (seq.)     & 1024 & 1024 & 1024 \\
Device micro-batch size      & 4 & 8 & 8 \\
Sequence length              & 4096 & 4096 & 4096 \\
Tokens per step              & 4.19M & 4.19M & 4.19M \\
Training steps               & 1{,}000 & 2{,}200 & 5{,}000 \\
Total training tokens        & 4.19B & 9.23B & 20.97B \\
Warmup tokens ($\approx 2\%$) & 83.9M & 184.5M & 419.4M \\
LR schedule                  & \multicolumn{3}{c}{cosine (token units)} \\
Min./peak LR ratio $\alpha_f$ & \multicolumn{3}{c}{0.1} \\
Gradient clipping            & \multicolumn{3}{c}{1.0} \\
Master-weight precision      & \multicolumn{3}{c}{AMP BF16} \\
FP4 gradient scale           & \multicolumn{3}{c}{1024} \\
Stochastic rounding          & \multicolumn{3}{c}{true} \\
\bottomrule
\end{tabular}
\end{table}

The general training hyperparameters are listed in \cref{tab:general-hp-olmo}. We use a global batch size of $1024$ sequences (i.e., $\approx 4.19$M tokens per step), gradient clipping at $1.0$, and a cosine learning-rate schedule (in token units) that warms up over the first $2\%$ of training and decays to $10\%$ of the peak value ($\alpha_f=0.1$). The three scales are trained for $1{,}000$, $2{,}200$, and $5{,}000$ steps, which correspond to $4.19$B, $9.23$B, and $20.97$B tokens and go well beyond the Chinchilla-optimal budget~\citep{hoffmann2022chinchilla}. We compare StoSignSGD against Lion~\citep{chen2023symbolic}, AdamW~\citep{loshchilov2019adamw} and Muon~\citep{jordan2024muon}. For StoSignSGD we use the practical StoSignSGDv2 variant with a damped max-buffer (\cref{alg:stosignsgdv2}), whose damping factor $\beta_2<1$ fits these long-horizon low-precision runs. We fix $\beta_1=0.9$ and weight decay $\lambda=0.1$, and grid-search the peak learning rate over $\cbrac{8\text{e-}4, 9\text{e-}4, 1\text{e-}3, 2\text{e-}3}$ and $\beta_2$ over $\cbrac{0.9, 0.99, 0.995, 0.999}$. The selected configurations are reported in \cref{tab:optimizer-hp-olmo}.\footnote{We intentionally deviate from the default $(\beta_1,\beta_2)=(0.9,0.99)$ used in the codebase of~\citet{chen2026tetrajetv2}. We find that equal momentum coefficients consistently work better than the unequal default pair, though we do not yet have a clear explanation.} Downstream performance is evaluated zero-shot with the OLMo2 in-loop evaluation suite~\citep{olmo2024olmo2,eval-harness} following common practice~\citep{pan2025unbiased,chen2026tetrajetv2}. We report accuracy on seven standard commonsense and reasoning benchmarks, namely ARC-Easy and ARC-Challenge~\citep{clark2018arc}, OpenBookQA~\citep{mihaylov2018openbookqa}, HellaSwag~\citep{zellers2019hellaswag}, PIQA~\citep{bisk2020piqa}, Social IQa~\citep{sap2019socialiqa}, and BoolQ~\citep{clark2019boolq}, as summarized in \cref{tab:fp4}. Training loss and perplexity are reported in \cref{tab:fp4-train}. At every scale, AdamW diverges to NaN within the first few hundred steps because its second-moment buffer $\v_t$ underflows to zero under FP4, the same mechanism diagnosed for FP8 in \cref{fig:adamw_fp8_collapse}. The same catastrophic divergence also happens for Muon, and a smaller learning rate only slows down the occurrence of NaN but never avoids it. In contrast, StoSignSGD and Lion train stably, and StoSignSGD reaches the lowest training loss and perplexity and the highest average downstream accuracy at all three scales, with the gap over Lion largest at the biggest model.

\paragraph{Why does Muon also fail in low precision?}
Beyond AdamW, we also evaluate Muon~\citep{jordan2024muon,liu2025muon}, which orthogonalizes the momentum update via a Newton--Schulz iteration and falls back to AuxAdam for the remaining non-matrix parameters (with the Adam learning rate set to $0.1\times$ the Muon rate). For FP8, we use essentially the AdamW configuration but with Nesterov momentum $0.95$, and we also test a larger weight decay and a larger batch size. For FP4 we follow the AdamW grid and additionally try an extremely small $\eta_\text{peak}=1e-4$. This last setting is worth checking because Muon usually prefers a learning rate about $10\times$ larger than AdamW, so it stresses the opposite extreme. Across every model size the result is the same as AdamW. Muon diverges to NaN within the first hundred or so steps, and larger learning rates diverge faster. We attribute this failure to two reasons. First, Muon is coupled with Adam for non-matrix parameters, so it inherits the same variance-buffer underflow that destabilizes AdamW under FP4 and FP8. Second, the Newton--Schulz orthogonalization is very sensitive to numerical precision. When the optimizer states, and hence the momentum matrix being orthogonalized, are already stored in low precision, the iteration cannot recover well-orthogonalized update directions, and the error builds up and quickly drives training to NaN. StoSignSGD avoids both problems, since it neither keeps a squared-gradient buffer nor needs any orthogonalization.

\begin{table}[t]
\centering
\caption{Selected optimizer hyperparameters for radical FP4 pretraining. AdamW and Muon diverge to
NaN at every scale regardless of the learning rate; we list the learning rate matched to the
sign-based methods for completeness. Muon additionally sweeps an extremely small $\eta_\text{peak}=1$e-$4$
and uses AuxAdam (Adam learning rate set to $0.1\times$ the Muon rate) for non-matrix parameters.\label{tab:optimizer-hp-olmo}}
\begin{tabular}{llcccc}
\toprule
\textbf{Model} & \textbf{Optimizer} & $\eta_\text{peak}$ & $\beta_1$ & $\beta_2$ & $\lambda$ \\
\midrule
\multirow{4}{*}{OLMo2-70M}
 & AdamW      & 1e-3 & 0.9 & 0.95 & 0.1 \\
 & Muon       & 1e-2 & 0.95 & -- & 0.1 \\
 & Lion       & 1e-3 & 0.9 & 0.9  & 0.1 \\
 & \cellcolor{yellow!20}StoSignSGD & \cellcolor{yellow!20}1e-3 & \cellcolor{yellow!20}0.9 & \cellcolor{yellow!20}0.995 & \cellcolor{yellow!20}0.1 \\
\midrule
\multirow{4}{*}{OLMo2-150M}
 & AdamW      & 1e-3 & 0.9 & 0.95 & 0.1 \\
 & Muon       & 1e-2 & 0.95 & -- & 0.1 \\
 & Lion       & 1e-3 & 0.9 & 0.9  & 0.1 \\
 & \cellcolor{yellow!20}StoSignSGD & \cellcolor{yellow!20}8e-4 & \cellcolor{yellow!20}0.9 & \cellcolor{yellow!20}0.995 & \cellcolor{yellow!20}0.1 \\
\midrule
\multirow{4}{*}{OLMo2-370M}
 & AdamW      & 8e-4 & 0.9 & 0.95 & 0.1 \\
 & Muon       & 1e-3 & 0.95 & -- & 0.1 \\
 & Lion       & 8e-4 & 0.9 & 0.9  & 0.1 \\
 & \cellcolor{yellow!20}StoSignSGD & \cellcolor{yellow!20}8e-4 & \cellcolor{yellow!20}0.9 & \cellcolor{yellow!20}0.99 & \cellcolor{yellow!20}0.1 \\
\bottomrule
\end{tabular}
\end{table}

\begin{table}[ht]
    \centering
    \caption{Training loss and perplexity for radical FP4 pretraining.
    NaN denotes a failed or unavailable run.\label{tab:fp4-train}}
    \vspace{-5pt}
    \begin{tabular}{llcc}
        \toprule
        \textbf{Model}
        & \textbf{Method}
        & \textbf{Train Loss}
        & \textbf{Train PPL} \\
        \midrule

        \multirow{4}{*}{OLMo2-70M}
        & AdamW
        & NaN
        & NaN \\

        & Muon
        & NaN
        & NaN \\

        & Lion
        & 4.254
        & 70.36 \\

        & \cellcolor{yellow!20}\textbf{StoSignSGD}
        & \cellcolor{yellow!20}\textbf{4.211}
        & \cellcolor{yellow!20}\textbf{67.44} \\

        \midrule

        \multirow{4}{*}{OLMo2-150M}
        & AdamW
        & NaN
        & NaN \\

        & Muon
        & NaN
        & NaN \\

        & Lion
        & 3.671
        & 39.30 \\

        & \cellcolor{yellow!20}\textbf{StoSignSGD}
        & \cellcolor{yellow!20}\textbf{3.652}
        & \cellcolor{yellow!20}\textbf{38.54} \\

        \midrule

        \multirow{4}{*}{OLMo2-370M}
        & AdamW
        & NaN
        & NaN \\

        & Muon
        & NaN
        & NaN \\

        & Lion
        & 3.181
        & 24.07 \\

        & \cellcolor{yellow!20}\textbf{StoSignSGD}
        & \cellcolor{yellow!20}\textbf{3.177}
        & \cellcolor{yellow!20}\textbf{23.98} \\

        \bottomrule
    \end{tabular}
    \vspace{-10pt}
\end{table}

\subsection{Dessecting Sign Structural Stochasticity\label{sec:dessect-sign-noise}}

In this section, we present the omitted details in \cref{sec:demystify-structural}. 

\subsubsection{Unbiased Property of Sign Conversion}

First, we provide the following theoretical guarantee for the unbiased sign conversion framework.

\begin{prop}\label{prop:unbiased-sign-conversion}
    For \cref{alg:sign-conversion}, if $\bsigma_t\succeq\m_t$, then the \textcolor{blue}{sign-converted update} equals the \textcolor{gray}{general optimizer update} in the sense that
    \begin{align*}
        \E\sqbrac{\textcolor{blue}{\sign{\m_t+\bsigma_t\odot\n_t}}}=\textcolor{gray}{\m_t/\bsigma_t}.
    \end{align*}
\end{prop}

\begin{proof}
    Once we have $\bsigma_t\succeq\m_t$, then \cref{prop:unbiased-sign-conversion} is a direct instantiation of \cref{prop:unbiased-sign}.
\end{proof}
The condition $\bsigma_t\succeq\m_t$ can be easily satisfied by most optimizers in practice. Our empirical demonstration in \cref{fig:rms-norm} also confirms this point.

\subsubsection{Experimental Details}

To evaluate the optimizers detailed in \cref{sec:supporting-algorithms}, we establish two distinct experimental setups: instruction following and mathematical reasoning. Their configurations are elaborated below.

\paragraph{Instruction Following} 
For the results reported in \cref{tab:alpaca}, we finetune a diverse suite of small-scale LLMs (0.1B to 2.7B parameters) on the Alpaca dataset~\citep{taori2023alpaca}. The selected models include Qwen2.5-0.5B~\citep{yang2024qwen2.5}, TinyLLaMA~\citep{zhang2024tinyllama}, Llama-3.2-1B~\citep{meta2024llama3.2_1b}, GPT-2~\citep{radford2019language}, Gemma-3-1B~\citep{kamath2025gemma3}, and Phi-2~\citep{javaheripi2023phi2}. Following the training protocols recommended by~\citet{taori2023alpaca,liu2025adagrad}, we train each model for 3 epochs. We set the maximum sequence length to 512, with a global batch size of 256 and gradient accumulation steps set to 1.

\paragraph{Mathematical Reasoning} 
For this task, we scale up to three popular 7B+ parameter LLMs: Qwen2.5-7B~\citep{yang2024qwen2.5}, Llama-3.1-8B~\citep{grattafiori2024llama}, and Mistral-7B-v0.1~\citep{jiang2023mistral}. These models are finetuned on the training sets of GSM8k~\citep{cobbe2021gsm8k} and MathQA~\citep{amini-etal-2019-mathqa}, and subsequently evaluated on their respective test sets. In line with standard practice~\citep{pan2024lisa,pan-etal-2025-scalebio}, the models are finetuned for 1 epoch using a training batch size of 128 and an evaluation batch size of 64.

All remaining hyperparameters across both experimental setups adhere to the default configurations provided by~\citet{diao2024lmflow,eval-harness}. For optimizer-specific tuning, we employ a coordinate descent approach to systematically sweep the learning rates~\citep{wen2026fantastic}. To ensure a fair comparison, the momentum parameters ($\beta_1$ and $\beta_2$) for the sign-converted methods (StoSignSGD, SignAdamW, and SignAdaMax) are configured to strictly match the optimal values found for their respective base optimizers (IE-StoSignSGD, AdamW, and AdaMax). Finally, a uniform weight decay of 0.1 is applied across all evaluated optimizers.

\subsubsection{Evolution of the Signal-to-Noise Ratio}
To empirically validate our variance analysis, we track the signal-to-noise ratio, defined as $\r_t := \m_t/\bsigma_t$, during the finetuning of GPT-2 on the Alpaca dataset (utilizing the exact configurations outlined in \cref{tab:alpaca}). We monitor this quantity across two distinct optimizer parameter groups, comprising 15,539,808 and 15,168 parameters, respectively. Throughout the optimization trajectory, we compute the quadratic mean (i.e., the RMS-norm, $\norm{\r_t}_\text{RMS}$), alongside the arithmetic, geometric, and harmonic means of $\r_t$.

\begin{figure}[htbp]
    \centering
    \vspace{-10pt}
    \includegraphics[width=\linewidth]{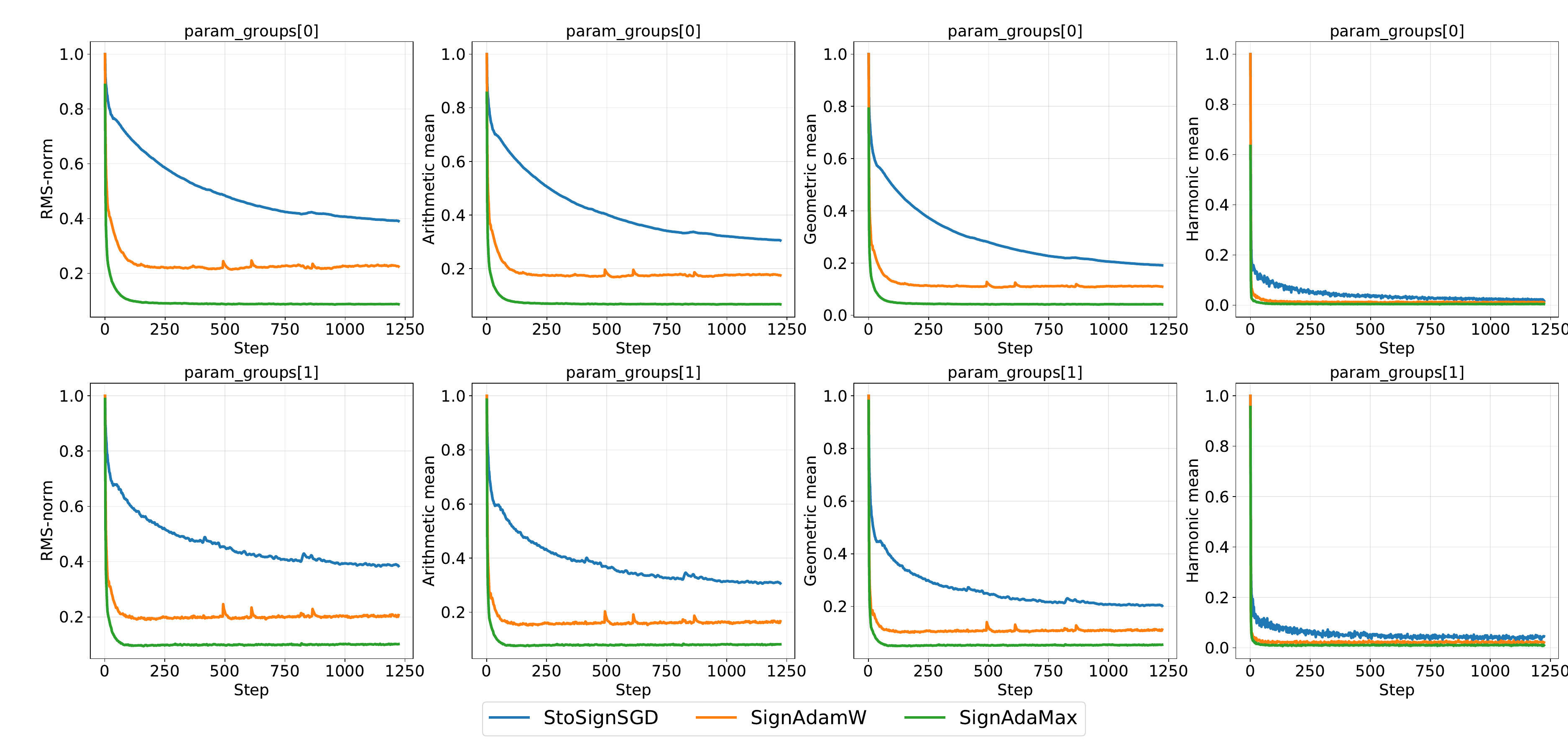}
    \vspace{-20pt}
    \caption{The intensity of structural noise among unbiased sign conversion optimizers.}
    \label{fig:rms-norm}
\end{figure}

As illustrated in \cref{fig:rms-norm}, the magnitude of $\r_t$ sustained by StoSignSGD is significantly larger than that of SignAdamW and SignAdaMax. Recall from \cref{alg:sign-conversion} and \cref{prop:unbiased-sign-conversion} that StoSignSGD, SignAdamW, and SignAdaMax serve as the unbiased sign-converted counterparts to IE-StoSignSGD, AdamW, and AdaMax, respectively. Consequently, the update RMS-norm ($\norm{\r_t}_\text{RMS}$) of these base optimizers directly dictates the structural $\r_t$ utilized by their sign-based derivatives. This relationship mathematically underpins the variance control mechanisms discussed previously in \cref{sec:ablation} and accounts for the empirical measurements reported in the final column of \cref{tab:tricks}. One would also note that our experiments indicate an approximate 0.2 update RMS-norm for AdamW, which aligns with the theoretical value~\citep{pmlr-v235-kosson24a} and empirical observation~\citep{liu2025muon} perfectly well.

\subsection{Computer Vision Experiments}

For completeness, we conduct a preliminary evaluation on standard computer vision tasks to validate the broader effectiveness of StoSignSGD compared to SignSGD and AdamW. Specifically, we perform multi-class image classification using the ResNet-18 architecture~\citep{he2016resnet} on the CIFAR-10 dataset~\citep{krizhevsky2009learning}, adhering to the experimental protocols established by~\citet{yuan25mars,jiang2025improved}.

During the hyperparameter tuning phase, we perform a grid search over the momentum parameter $\beta_1 \in \cbrac{0.99, 0.95, 0.9}$ and the learning rate $\eta \in \cbrac{0.5, 0.25, 0.1, 0.05, 0.025, 0.01}\times10^{-2}$. For AdamW specifically, we also sweep $\beta_2 \in \cbrac{0.95, 0.99, 0.999}$. All models are trained for 200 epochs with a global batch size of 128 and a uniform weight decay of 0.0001. Additionally, we employ a MultiStepLR learning rate scheduler~\citep{NEURIPS2019PYTORCH}, decaying the learning rate to 10\% of its initial value at epoch 50, and to 1\% at epoch 125.

The evolution of the test loss and test accuracy is illustrated in \cref{fig:cifar}. The empirical trajectories clearly demonstrate that StoSignSGD converges more rapidly than both SignSGD and AdamW. While all three methods achieve comparable final test accuracies---with StoSignSGD maintaining a slight edge---StoSignSGD achieves a significantly lower final test loss. This notable reduction in test loss highlights its superior generalization capabilities and stronger resilience against overfitting. Ultimately, these findings confirm that the variance control mechanisms of StoSignSGD translate effectively to visual domains, demonstrating its versatility as a general-purpose optimizer beyond natural language processing.

\begin{figure}[htbp] 
\centering
\includegraphics[width=0.49\linewidth]{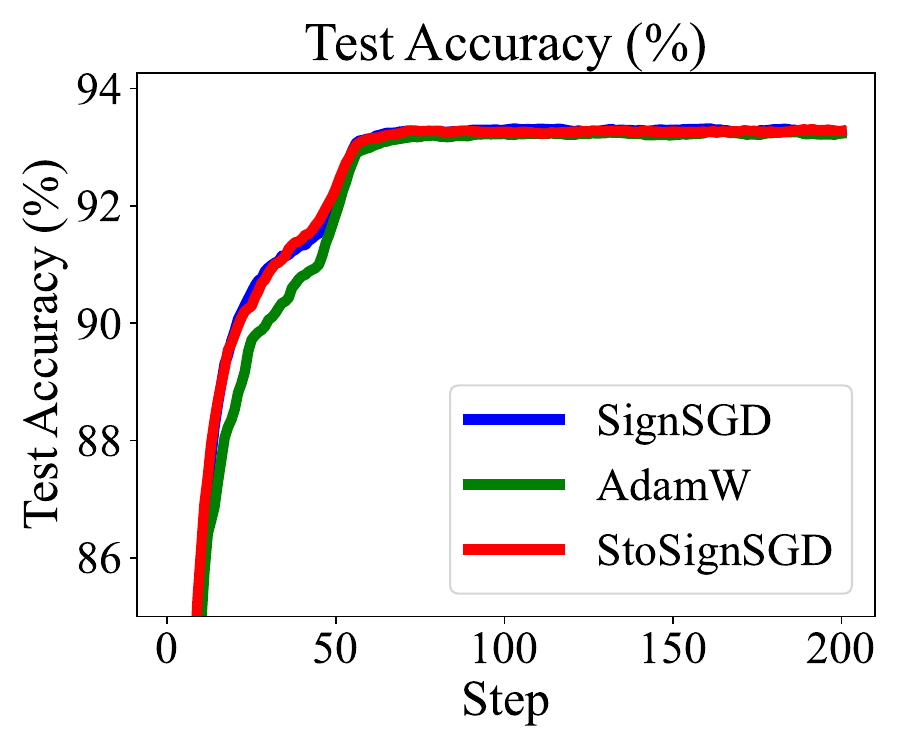}\hfill
\includegraphics[width=0.49\linewidth]{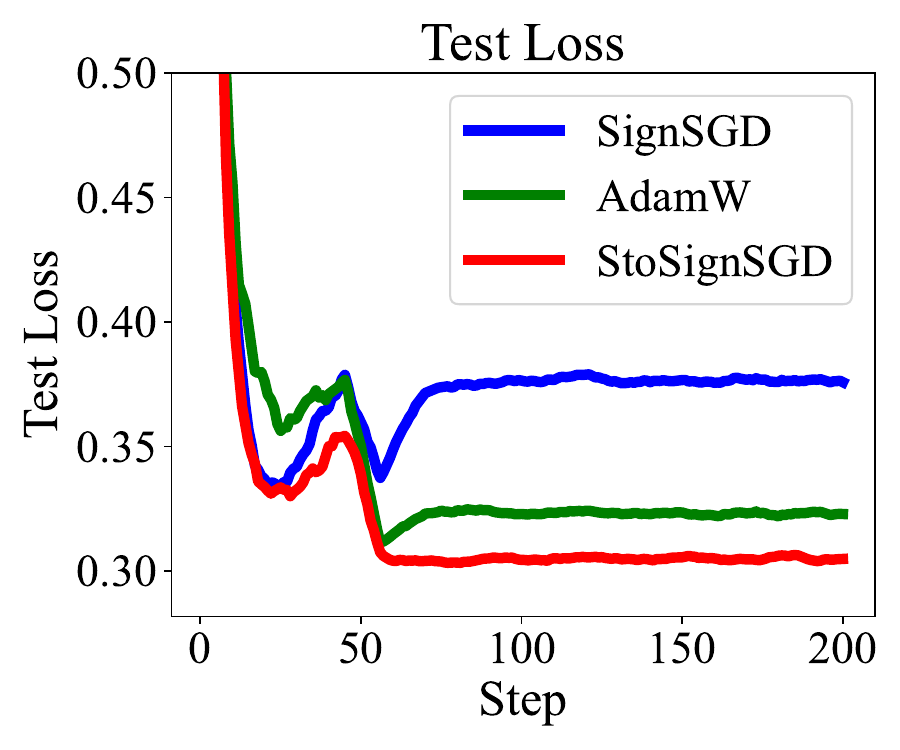}
\vspace{-10pt}
\caption{Test results for ResNet18 on CIFAR-10 dataset.\label{fig:cifar}}
\vspace{-10pt}
\end{figure}

\subsection{Supporting Algorithms\label{sec:supporting-algorithms}}

This section provides the practical implementations for the algorithms in \cref{tab:tricks}, i.e., SignSGD (\cref{alg:practical-signsgd}), AdamW (\cref{alg:adamw}), AdaMax (\cref{alg:adamax}), SignAdamW (\cref{alg:signadamw}), SignAdaMax (\cref{alg:signadamax}), and IE-StoSignSGD (\cref{alg:iestosignsgd}). Similar to the practical implementation of StoSignSGD in \cref{alg:practical-stosignsgd}, we implement SignSGD as \cref{alg:practical-signsgd}, adding decoupled weight decay on top of Signum~\citep{bernstein2018signsgd}, which is SignSGD with momentum, as it has been proven to be highly beneficial for sign-based methods~\citep{sun2023momentum,yu2026signheavytails}. To avoid confusion with relevant literature~\citep{bernstein2018signsgd,sun2023momentum,jiang2025improved,yu2026signheavytails}, we choose the name SignSGD. For AdamW and AdaMax, our implementation follows common practice~\citep{NEURIPS2019PYTORCH}. For SignAdamW and SignAdaMax, they build upon AdamW and AdaMax and are generated by the sign conversion framework in \cref{alg:sign-conversion}.

\begin{algorithm}[htbp]
    \caption{Practical Implementation of SignSGD}
    \label{alg:practical-signsgd}
    \begin{algorithmic}[1]
    \STATE {\bfseries Input:} Start point $\x_1\in\R^d$, momentum $\beta_1\in[0,1)$, learning rate $\cbrac{\eta_t}_{t=1}^T$, weight decay $\lambda\ge0$.\\
    \FOR{$t=1$ {\bfseries to} $T$}
    \STATE Get stochastic gradient $\g_t$
    \STATE Update momentum $\m_t=\beta_1\m_{t-1}+(1-\beta_1)\g_t$\hfill\COMMENT{\textcolor{gray}{$\m_1=\g_1$}}
    \STATE Compute $\x_{t+1} = \x_t - \eta_t \sign{\m_t}-\eta_t\lambda\x_t$
    \ENDFOR
\end{algorithmic}
\end{algorithm}

\begin{algorithm}[htbp]
    \caption{AdamW~\citep{kingma15adam,loshchilov2019adamw}}
    \label{alg:adamw}
    \begin{algorithmic}[1]
    \STATE {\bfseries Input:} Initialization $\x_1\in\R^d$, learning rate $\cbrac{\eta_t}_{t=1}^T$, weight decay $\lambda\ge0$, $\beta_1,\beta_2\in[0,1)$, $\epsilon\ge0$.\\
    \STATE {\bfseries Initialize:} $\m_0=\0,\ \v_0=\0$.
    \FOR{$t=1$ {\bfseries to} $T$}
        \STATE Get stochastic gradient $\g_t$
        \STATE $\m_t = \beta_1 \m_{t-1} + (1-\beta_1)\g_t$
        \STATE $\v_t = \beta_2 \v_{t-1} + (1-\beta_2)\g_t\odot\g_t$
        \STATE Bias correction: $\widehat{\m}_t = \m_t/(1-\beta_1^t)$, $\widehat{\v}_t = \v_t/(1-\beta_2^t)$
        \STATE \textcolor{gray}{\itshape \# General Optimizer Update}
        \STATE \textcolor{gray}{$\x_{t+1} = \x_t - \eta_t \frac{\widehat{\m}_t}{\sqrt{\widehat{\v}_t}+\epsilon} - \eta_t\lambda\x_t$}
    \ENDFOR
\end{algorithmic}
\end{algorithm}

\begin{algorithm}[htbp]
    \caption{AdaMax with Decoupled Weight Decay~\citep{kingma15adam,loshchilov2019adamw}}
    \label{alg:adamax}
    \begin{algorithmic}[1]
    \STATE {\bfseries Input:} Initialization $\x_1\in\R^d$, learning rate $\cbrac{\eta_t}_{t=1}^T$, weight decay $\lambda\ge0$, $\beta_1,\beta_2\in[0,1)$, $\epsilon\ge0$.\\
    \STATE {\bfseries Initialize:} $\m_0=\0,\ \u_0=\0$.
    \FOR{$t=1$ {\bfseries to} $T$}
        \STATE Get stochastic gradient $\g_t$
        \STATE $\m_t = \beta_1 \m_{t-1} + (1-\beta_1)\g_t$
        \STATE $\u_t = \max\cbrac{\beta_2 \u_{t-1},\, |\g_t|}$
        \STATE Bias correction: $\widehat{\m}_t = \m_t/(1-\beta_1^t)$
        \STATE \textcolor{gray}{\itshape \# General Optimizer Update}
        \STATE \textcolor{gray}{$\x_{t+1} = \x_t - \eta_t \frac{\widehat{\m}_t}{\u_t+\epsilon} - \eta_t\lambda\x_t$}
    \ENDFOR
\end{algorithmic}
\end{algorithm}

\begin{algorithm}[htbp]
    \caption{SignAdamW}
    \label{alg:signadamw}
    \begin{algorithmic}[1]
    \STATE {\bfseries Input:} Initialization $\x_1\in\R^d$, learning rate $\cbrac{\eta_t}_{t=1}^T$, weight decay $\lambda\ge0$, $\beta_1,\beta_2\in[0,1)$, $\epsilon\ge0$.\\
    \STATE {\bfseries Initialize:} $\m_0=\0,\ \v_0=\0$.
    \FOR{$t=1$ {\bfseries to} $T$}
        \STATE Get stochastic gradient $\g_t$
        \STATE $\m_t =\beta_1 \m_{t-1} + (1-\beta_1)\g_t$
        \STATE $\v_t =\beta_2 \v_{t-1} + (1-\beta_2)\g_t\odot\g_t$
        \STATE Bias correction: $\widehat{\m}_t = \m_t/(1-\beta_1^t)$, $\widehat{\v}_t = \v_t/(1-\beta_2^t)$
        \STATE \textcolor{blue}{\itshape \# Sign Conversion Update}
        \STATE \textcolor{blue}{Sample Uniform noise $\n_t\sim\unif\brac{[-1,1]^d}$}
        \STATE \textcolor{blue}{$\x_{t+1} = \x_t - \eta_t\sign{\widehat{\m}_t +\brac{\sqrt{\widehat{\v}_t}+\epsilon}\odot \n_t} - \eta_t\lambda\x_t$}
    \ENDFOR
\end{algorithmic}
\end{algorithm}

\begin{algorithm}[htbp]
    \caption{SignAdaMax}
    \label{alg:signadamax}
    \begin{algorithmic}[1]
    \STATE {\bfseries Input:} Initialization $\x_1\in\R^d$, learning rate $\cbrac{\eta_t}_{t=1}^T$, weight decay $\lambda\ge0$, $\beta_1,\beta_2\in[0,1)$, $\epsilon>0$.\\
    \STATE {\bfseries Initialize:} $\m_0=\0,\ \u_0=\0$.
    \FOR{$t=1$ {\bfseries to} $T$}
        \STATE Get stochastic gradient $\g_t$
        \STATE $\m_t = \beta_1 \m_{t-1} + (1-\beta_1)\g_t$
        \STATE $\u_t = \max\cbrac{\beta_2 \u_{t-1},\, |\g_t|}$ 
        \STATE Bias correction: $\widehat{\m}_t = \m_t/(1-\beta_1^t)$
        \STATE \textcolor{blue}{\itshape \# Sign Conversion Update}
        \STATE \textcolor{blue}{Sample Uniform noise $\n_t\sim\unif\brac{[-1,1]^d}$}
        \STATE \textcolor{blue}{$\x_{t+1} = \x_t - \eta_t\sign{\widehat{\m}_t +\brac{\u_t+\epsilon}\odot \n_t} - \eta_t\lambda\x_t$}
    \ENDFOR
\end{algorithmic}
\end{algorithm}

\begin{algorithm}[htbp]
    \caption{In-expectation StoSignSGD (IE-StoSignSGD)}
    \label{alg:iestosignsgd}
    \begin{algorithmic}[1]
    \STATE {\bfseries Input:} Initialization $\x_1\in\R^d$, momentum $\beta_1\in[0,1),\beta_2\in(0,1]$, learning rate $\cbrac{\eta_t}_{t=1}^T$, weight decay $\lambda\ge0$.\\
    \FOR{$t=1$ {\bfseries to} $T$}
    \STATE Get stochastic gradient $\g_t$
    \STATE Update momentum $\m_t=\beta_1\m_{t-1}+(1-\beta_1)\g_t$\hfill\COMMENT{\textcolor{gray}{$\m_1=\g_1$}}
    \STATE Update max-buffer $\Gb_t=\max\cbrac{\beta_2\Gb_{t-1},\abs{\m_t}}$\hfill\COMMENT{\textcolor{gray}{$\Gb_1=\abs{\m_1}$}}
    \STATE \textcolor{gray}{\itshape \# General Optimizer Update}
    \STATE \textcolor{gray}{Compute $\x_{t+1} = \x_t - \eta_t \m_t/\Gb_t-\eta_t\lambda\x_t$}
    \ENDFOR
\end{algorithmic}
\end{algorithm}

\subsubsection{StoSignSGDv2}

A closer inspection of our proposed methods reveals a structural discrepancy between the practical StoSignSGD (\cref{alg:practical-stosignsgd}) and its in-expectation counterpart, IE-StoSignSGD (\cref{alg:iestosignsgd}). Specifically, IE-StoSignSGD incorporates a damping factor $\beta_2$ when updating the maximum buffer $\Gb_t$. Inspired by AdaMax (cf. line 6 in \cref{alg:adamax}), this damping mechanism is highly advantageous during extended training trajectories because it prevents the tracked infinity norm of the historical gradients from becoming permanently locked to early, large-magnitude updates, thereby reducing algorithmic conservativity over time. To fully leverage this advantage within our sign-based framework, we introduce StoSignSGDv2 (\cref{alg:stosignsgdv2}), which serves as the exact, unbiased sign conversion of \cref{alg:iestosignsgd}.

\begin{algorithm}[htbp]
    \caption{StoSignSGD with Damping Max (StoSignSGDv2)}
    \label{alg:stosignsgdv2}
    \begin{algorithmic}[1]
    \STATE {\bfseries Input:} Initialization $\x_1\in\R^d$, momentum $\beta_1\in[0,1),\beta_2\in(0,1]$, learning rate $\cbrac{\eta_t}_{t=1}^T$, weight decay $\lambda\ge0$.\\
    \FOR{$t=1$ {\bfseries to} $T$}
    \STATE Get stochastic gradient $\g_t$
    \STATE Update momentum $\m_t=\beta_1\m_{t-1}+(1-\beta_1)\g_t$\hfill\COMMENT{\textcolor{gray}{$\m_1=\g_1$}}
    \STATE Update max-buffer $\Gb_t=\max\cbrac{\beta_2\Gb_{t-1},\abs{\m_t}}$\hfill\COMMENT{\textcolor{gray}{$\Gb_1=\abs{\m_1}$}}
    \STATE \textcolor{blue}{\itshape \# Sign Conversion Update}
    \STATE \textcolor{blue}{Sample Uniform noise $\n_t\sim\unif\brac{[-1,1]^d}$}
    \STATE \textcolor{blue}{$\x_{t+1} = \x_t - \eta_t \sign{\m_t+\Gb_t\odot\n_t}-\eta_t\lambda\x_t$}
    \ENDFOR
\end{algorithmic}
\end{algorithm}

A critical hyperparameter governing StoSignSGDv2 is the damping factor $\beta_2$. From a theoretical perspective, there is a strict trade-off: while a $\beta_2 < 1$ effectively reduces update variance in late-stage training (when gradient magnitudes naturally decay), it violates the monotonic non-decreasing property of the preconditioner ($\Gb_t \succeq \Gb_{t-1}$). As highlighted by~\citet{reddi2018convergence}, this monotonicity is essential for establishing rigorous convergence guarantees in adaptive optimization. To reconcile this theoretical requirement with the desire for empirical flexibility, we find it necessary to set $\beta_2$ extremely close to 1 (e.g., $\beta_2 = 0.99995$ is what we adopt for the following pretraining runs).

We evaluate StoSignSGDv2 across both LLM finetuning and pretraining regimes. While its performance in standard finetuning closely mirrors that of the original StoSignSGD, StoSignSGDv2 demonstrates distinct superiority in low-precision environments. Replicating the FP8 pretraining experiments from \cref{sec:fp8,sec:fp8-cont}, the results in \cref{tab:fp8-v2,fig:gpt2fp8v2} show that StoSignSGDv2 achieves the lowest training loss and perplexity, outperforming the original StoSignSGD, Lion, and standard baselines. This highlights the practical potential of the damped max-buffer design for highly constrained, large-scale training regimes.

\begin{table}
    \centering   
    \caption{FP8 pretraining results for GPT-2 with StoSignSGDv2. Extended from \cref{tab:fp8}\label{tab:fp8-v2}}

    \begin{tabular}{lcccc}
        \toprule
        \multirow{2}{*}{\textbf{Optimizer}} & \multicolumn{2}{c}{\textbf{Loss}} & \multicolumn{2}{c}{\textbf{Perplexity}} \\
        \cmidrule(lr){2-3} \cmidrule(lr){4-5}
        & \textbf{Train} & \textbf{Val} & \textbf{Train} & \textbf{Val} \\
        \midrule
        AdamW      & NaN   & NaN   & NaN   & NaN   \\
        Muon       & NaN   & NaN   & NaN   & NaN   \\
        SignSGD    & 3.539 & 3.546 & 34.43 & 34.67 \\
        Lion       & 3.452 & 3.460 & 31.55 & 31.82 \\
        StoSignSGD & 3.410 & 3.418 & 30.26 & 30.51 \\
        StoSignSGDv2 & \textbf{3.394} & \textbf{3.405} & \textbf{29.78} & \textbf{30.12} \\
        \bottomrule
    \end{tabular}
    \vspace{-10pt} 
\end{table}

\begin{figure}[htbp] 
\centering
\includegraphics[width=\linewidth]{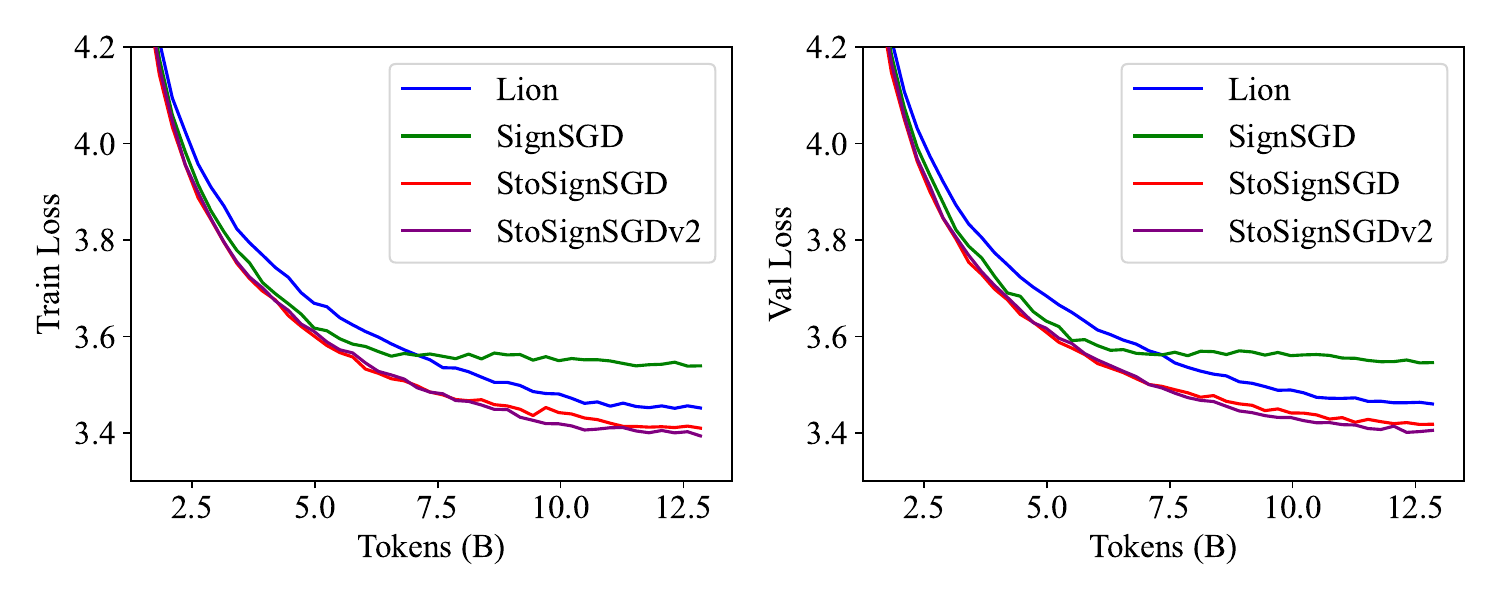}\hfill
\vspace{-10pt}
\caption{FP8 pretraining loss curves for GPT-2 with StoSignSGDv2. Extended from \cref{fig:gpt2fp8}.\label{fig:gpt2fp8v2}}
\vspace{-10pt}
\end{figure}

\section{Analysis for StoSignSGD}

In this section, we give the convergence analysis of the theoretical versions of StoSignSGD in \cref{alg:online-stosignsgd,alg:nonconvex-stosignsgd}. The practical version in \cref{alg:practical-stosignsgd}, incorporating momentum, can share the same derivations as \cref{alg:online-stosignsgd,alg:nonconvex-stosignsgd} according to, for example, the techniques in~\citet[Lemma~1]{alacaoglu2020newregretanalysis}. However, it is highly challenging to prove the benefits of momentum (see\mbox{~\citet{kidambi2018on,panICLR2024momentum}} and references therein), especially in our non-smooth settings\footnote{Here, we refer to benefits in the worst case analysis. A data-dependent analysis could still showcase the advantages of momentum updates.}. Therefore, we only provide the analysis without momentum.

\subsection{Stochastic Convex Optimization\label{sec:sco}}

In this section, we consider the settings of stochastic convex optimization (SCO) as a supplement to \cref{sec:convex-theory}. StoSignSGD is initialized at $\x_1\in\X$ and proceeds for $t\in[T]$ as
\begin{equation}\label{eq:convex-stosignsgd}
    \Gb_t=\max\cbrac{\Gb_{t-1},\abs{\g_t}},\quad\x_{t+1}=\Pi_\X^{\Gb_t}\sqbrac{\x_{t}-\eta_t\cdot\SS_{\Gb_t}(\g_t)},
\end{equation}
where $\g_t$ is a stochastic subgradient satisfying $\E_{\O_t}[\g_t]\in\partial f(\x_t)$, and $\Gb_t=(G_{t,1},\cdots,G_{t,d})\in\R^d$ tracks the infinity norm of past gradients. The procedure in~\eqref{eq:convex-stosignsgd} can be viewed as a offline deployment of \cref{alg:online-stosignsgd}. Below, we present the theoretical guarantee of StoSignSGD for SCO.

\begin{thm}\label{thm:convex-upper-bound}
    Under \cref{ass:convexity,ass:coord-lipschitz}, suppose StoSignSGD in~\eqref{eq:convex-stosignsgd} runs $T$ steps. Let $\xb_T=\frac{1}{T}\sum_{t=1}^T\x_t$, and $\eta_t=\frac{D_\infty}{\sqrt{2t}}$, it holds that 
    \begin{align*}
        \E\sqbrac{f(\xb_T)-f^*}\le\frac{\sqrt{2}D_\infty\norm{\Lb}_1}{\sqrt{T}}.
    \end{align*}
\end{thm}

The rate is known to be minimax optimal in $T$~\citep{nemirovskij1983problem}. Under the common coordinate-wise Lipschitz constant $L_i=L$, it also matches the information-theoretic lower bound~\citep{agarwal2012information}.  

The proof of \cref{thm:convex-upper-bound} lies in \cref{app:sec:proof-convex-upper-bound}, which serves as a basis of the analysis behind \cref{thm:oco-bound}. Concretely, the bound in \cref{thm:convex-upper-bound} could be easily transferred into a pseudo-regret bound for OCO, where the elaborations are in \cref{app:sec:sco2oco,app:sec:pseudo2expected}

\subsection{Proof of Theorem~\ref{thm:convex-upper-bound}\label{app:sec:proof-convex-upper-bound}}

We adopt an AdaGrad style of analysis~\citep{McMahanS10adagrad,duchi2011adaptive,liu2025adagrad} to prove \cref{thm:convex-upper-bound}. Consider the update in~\eqref{eq:convex-stosignsgd}, by \cref{lem:projection} we have
\begin{align*}
    &\E\sqbrac{\norm{\x_{t+1}-\x_*}^2_{\Gb_t/\eta_t}}\\\le&\E\sqbrac{\norm{\x_{t}-\eta_t\cdot\SS_{\Gb_t}(\g_t)-\x_*}^2_{\Gb_t/\eta_t}}\\=&\E\sqbrac{\norm{\x_{t}-\x_*}^2_{\Gb_t/\eta_t}-2\inner{\diag{\Gb_t}\SS_{\Gb_t}(\g_t)}{\x_{t}-\x_*}+\eta_t^2\norm{\SS_{\Gb_t}(\g_t)}^2_{\Gb_t/\eta_t}}\\=&\E\sqbrac{\norm{\x_{t}-\x_*}^2_{\Gb_t/\eta_t}}-\E\sqbrac{2\inner{\E_{\O_t,\SS_t}\sqbrac{\diag{\Gb_t}\SS_{\Gb_t}(\g_t)}}{\x_{t}-\x_*}}+\E\sqbrac{\eta_t\norm{\Gb_t}_1}\\=&\E\sqbrac{\norm{\x_{t}-\x_*}^2_{\Gb_{t-1}/\eta_{t-1}}+\norm{\x_{t}-\x_*}^2_{\Gb_t/\eta_t-\Gb_{t-1}/\eta_{t-1}}+\eta_t\norm{\Gb_t}_1}-\E\sqbrac{2\inner{\E_{\O_t}[\g_t]}{\x_{t}-\x_*}}\\\le&\E\sqbrac{\norm{\x_{t}-\x_*}^2_{\Gb_{t-1}/\eta_{t-1}}+D_\infty^2\brac{\frac{\norm{\Gb_t}_1}{\eta_t}-\frac{\norm{\Gb_{t-1}}_1}{\eta_{t-1}}}+\eta_t\norm{\Gb_t}_1}-\E\sqbrac{2\inner{\partial f(\x_t)}{\x_{t}-\x_*}},
\end{align*}
where the third equality is due to \cref{prop:unbiased-sign}; the last step makes use of \cref{ass:convexity} along with the fact that
\begin{align*}
    &\forall\eta_t\le\eta_{t-1},\ \frac{\Gb_t}{\eta_t}-\frac{\Gb_{t-1}}{\eta_{t-1}}\succeq\frac{\Gb_t-\Gb_{t-1}}{\eta_t}=\frac{1}{\eta_t}\begin{pmatrix}
        \norm{\g_{1:t,1}}_\infty-\norm{\g_{1:t-1,i}}_\infty\\\cdots\\\norm{\g_{1:t,d}}_\infty-\norm{\g_{1:t-1,d}}_\infty
    \end{pmatrix}\succeq\0.\\\Longrightarrow&\norm{\frac{\Gb_t}{\eta_t}-\frac{\Gb_{t-1}}{\eta_{t-1}}}_1=\frac{\norm{\Gb_t}_1}{\eta_t}-\frac{\norm{\Gb_{t-1}}_1}{\eta_{t-1}}.
\end{align*}
Define $\eta_0=\eta_1,\Gb_0=\Gb_1$\footnote{This is without loss of generality since the telescoping is only valid for $2\le t\le T$.}. Rearranging the above inequality and summing from $1$ to $T$ yields the following regret bound:
\begin{equation}\label{eq:convex-telescoping}
    \begin{aligned}
        &2\E\sqbrac{\sum_{t=1}^T\inner{\partial f(\x_t)}{\x_{t}-\x_*}}\\\le& \E\sqbrac{\sum_{t=1}^T\brac{\norm{\x_t-\x_*}^2_{\Gb_{t-1}/\eta_{t-1}}-\norm{\x_{t+1}-\x_*}^2_{\Gb_t/\eta_t}}}\\&+\E\sqbrac{D_\infty^2\sum_{t=1}^T\brac{\frac{\norm{\Gb_t}_1}{\eta_t}-\frac{\norm{\Gb_{t-1}}_1}{\eta_{t-1}}}}+\E\sqbrac{\sum_{t=1}^T\eta_t\norm{\Gb_t}_1}\\\le&\norm{\x_1-\x_*}^2_{\Gb_1/\eta_1}+D_\infty^2\E\sqbrac{\frac{\norm{\Gb_T}_1}{\eta_T}-\frac{\norm{\Gb_1}_1}{\eta_1}}+\sum_{t=1}^T\eta_t\E\sqbrac{\norm{\Gb_t}_1}\\\le&\norm{\x_1-\x_*}^2_\infty\frac{\norm{\Gb_1}_1}{\eta_1}-\frac{D_\infty^2\norm{\Gb_1}_1}{\eta_1}+\frac{D_\infty^2}{\eta_T}\E\sqbrac{\norm{\Gb_T}_1}+\sum_{t=1}^T\eta_t\E\sqbrac{\norm{\Gb_t}_1}\\\le&\frac{D_\infty^2}{\eta_T}\E\sqbrac{\norm{\Gb_T}_1}+\sum_{t=1}^T\eta_t\E\sqbrac{\norm{\Gb_T}_1}=D_\infty\sqrt{2T}\E\sqbrac{\norm{\Gb_T}_1}+D_\infty\sum_{t=1}^T\frac{1}{\sqrt{2t}}\cdot\E\sqbrac{\norm{\Gb_T}_1}\\\le&\brac{2\sqrt{2T}-\frac{1}{\sqrt{2}}}D_\infty\norm{\Lb}_1,
    \end{aligned}
\end{equation}
where we used \cref{lem:sqrt-sum} and \cref{ass:coord-lipschitz} in the last step. By convexity, we complete the proof using the well-known online-to-batch conversion~\citep{cesa2006prediction}:
\begin{align*}
    \E\sqbrac{f(\xb_T)-f^*}\le\E\sqbrac{\frac{1}{T}\sum_{t=1}^T\brac{f(\x_t)-f^*}}\le\E\sqbrac{\sum_{t=1}^T\inner{\partial f(\x_t)}{\x_{t}-\x_*}}\le\frac{\sqrt{2}D_\infty\norm{\Lb}_1}{\sqrt{T}}. 
\end{align*}

\subsection{Proof of Theorem~\ref{thm:oco-bound}\label{app:sec:proof-online-upper-bound}}

In this section, we analyze \cref{alg:online-stosignsgd} in the context of online convex optimization (OCO). We first prove a weaker pseudo-regret bound in \cref{app:sec:sco2oco}, building on the framework of stochastic convex optimization (SCO) in \cref{app:sec:proof-convex-upper-bound}. Then, we demonstrate a stronger regret bound in \cref{app:sec:pseudo2expected}, which will play an indispensable role in the non-convex setting.

\subsubsection{From SCO to OCO\label{app:sec:sco2oco}}

\begin{thm}\label{thm:oco-bound-pseudo}
    Under \cref{ass:convexity,ass:coord-lipschitz}, set $\eta_t=\frac{D_\infty}{\sqrt{2t}}$. For any fixed $\x\in\X$, \cref{alg:online-stosignsgd} ensures 
    \begin{align*}
        \E\sqbrac{\textnormal{Regret}_T(\x)}\le\E\sqbrac{\sum_{t=1}^T\inner{\g_t}{\x_t-\x}}\le\sqrt{2}D_\infty\norm{\Lb}_1\sqrt{T}.
    \end{align*}
\end{thm}

\begin{proof}
    Inspecting \cref{app:sec:proof-convex-upper-bound}, one should realize that \cref{thm:oco-bound} could be proven almost identically to \cref{thm:convex-upper-bound}. For completeness, we still present its proof. Following \cref{app:sec:proof-convex-upper-bound}, for any fixed $\x\in\X$, we have
\begin{align*}
    &\E\sqbrac{\norm{\x_{t+1}-\x}^2_{\Gb_t/\eta_t}}\overset{\text{Lemma~\ref{lem:projection}}}{\le}\E\sqbrac{\norm{\x_{t}-\eta_t\cdot\SS_{\Gb_t}(\g_t)-\x}^2_{\Gb_t/\eta_t}}\\=&\E\sqbrac{\norm{\x_{t}-\x}^2_{\Gb_t/\eta_t}-2\inner{\diag{\Gb_t}\SS_{\Gb_t}(\g_t)}{\x_{t}-\x}+\eta_t^2\norm{\SS_{\Gb_t}(\g_t)}^2_{\Gb_t/\eta_t}}\\=&\E\sqbrac{\norm{\x_{t}-\x}^2_{\Gb_t/\eta_t}}-\E\sqbrac{2\inner{\E_{\SS_t}\sqbrac{\Gb_t\odot\SS_{\Gb_t}(\g_t)}}{\x_{t}-\x}}+\E\sqbrac{\eta_t\norm{\Gb_t}_1}\\=&\E\sqbrac{\norm{\x_{t}-\x}^2_{\Gb_{t-1}/\eta_{t-1}}+\norm{\x_{t}-\x}^2_{\Gb_t/\eta_t-\Gb_{t-1}/\eta_{t-1}}+\eta_t\norm{\Gb_t}_1}-\E\sqbrac{2\inner{\g_t}{\x_{t}-\x}}\\\le&\E\sqbrac{\norm{\x_{t}-\x}^2_{\Gb_{t-1}/\eta_{t-1}}+D_\infty^2\brac{\frac{\norm{\Gb_t}_1}{\eta_t}-\frac{\norm{\Gb_{t-1}}_1}{\eta_{t-1}}}+\eta_t\norm{\Gb_t}_1}-\E\sqbrac{2\inner{\g_t}{\x_{t}-\x}}.
\end{align*}
Then, by the same telescoping step as depicted in~\eqref{eq:convex-telescoping}, we obtain the following regret bound for linear losses:
\begin{align*}
    \E\sqbrac{\sum_{t=1}^T\inner{\g_t}{\x_t-\x}}\le\sqrt{2T}D_\infty\norm{\Lb}_1.
\end{align*}
By convexity of $\cbrac{f_t}_{t\in[T]}$, the regret bound for $\cbrac{f_t}_{t\in[T]}$ follows directly from above.
\end{proof}

\subsubsection{From Pseudo-Regret to Expected Regret\label{app:sec:pseudo2expected}}

The theoretical guarantee in \cref{thm:oco-bound-pseudo} is actually a pseudo-regret bound which only holds for a \emph{fixed} comparator $\x\in\X$~\citep{bubeck2012regret}. In this section, we prove a stronger regret bound called expected regret. In pseudo-regret, the ``best'' action is only defined w.r.t.~the expectation and is not random. The randomness only plays a role once the ``deterministic'' comparator $\x\in\X $ is fixed. In expected regret, the ``best'' action depends on the randomness and is itself a random variable, because the maximum is inside the expectation. For a complete overview of these notions, one may refer to~\citep{bubeck2012regret,lattimore2020bandit}.

\begin{thm}[Full version of \cref{thm:oco-bound}]\label{thm:oco-bound-expected}
    Under \cref{ass:convexity,ass:coord-lipschitz}, set $\eta_t=\frac{D_\infty}{\sqrt{2t}}$. \cref{alg:online-stosignsgd} ensures 
    \begin{align*}
        \E\sqbrac{\max_{\x\in\X}\textnormal{Regret}_T(\x)}\le\E\sqbrac{\max_{\x\in\X}\sum_{t=1}^T\inner{\g_t}{\x_t-\x}}\le(2+\sqrt{2})D_\infty\norm{\Lb}_1\sqrt{T}.
    \end{align*}
    If we set $\eta_t\equiv\frac{D_\infty}{\sqrt{T}}$, then \cref{alg:online-stosignsgd} ensures 
    \begin{align*}
        \E\sqbrac{\max_{\x\in\X}\textnormal{Regret}_T(\x)}\le\E\sqbrac{\max_{\x\in\X}\sum_{t=1}^T\inner{\g_t}{\x_t-\x}}\le3D_\infty\norm{\Lb}_1\sqrt{T}.
    \end{align*}
\end{thm}

\begin{proof}
    It suffices to prove the second inequality, i.e., the regret bound for linear losses. Denote by $\xt_T:=\argmax_{\x\in\X}\sum_{t=1}^T\inner{\g_t}{\x_t-\x}$. By the similar derivations as in \cref{app:sec:proof-convex-upper-bound,app:sec:sco2oco}, we have
    \begin{align*}
    \E\sqbrac{\norm{\x_{t+1}-\xt_T}^2_{\Gb_t/\eta_t}}\le&\E\sqbrac{\norm{\x_{t}-\xt_T}^2_{\Gb_{t-1}/\eta_{t-1}}+D_\infty^2\brac{\frac{\norm{\Gb_t}_1}{\eta_t}-\frac{\norm{\Gb_{t-1}}_1}{\eta_{t-1}}}+\eta_t\norm{\Gb_t}_1}\\&-\E\sqbrac{2\inner{\Gb_t\odot\SS_{\Gb_t}(\g_t)}{\x_{t}-\xt_T}}.
\end{align*}
Rearranging and telescoping analogously to~\eqref{eq:convex-telescoping}, we obtain
\begin{align*}
    \E\sqbrac{\sum_{t=1}^T\inner{\g_t}{\x_{t}-\xt_T}}\le\E\sqbrac{\sum_{t=1}^T\underbrace{\inner{\g_t-\Gb_t\odot\SS_{\Gb_t}(\g_t)}{\x_{t}-\xt_T}}_{:=V_t}}+\frac{D_\infty^2\norm{\Lb}_1}{2\eta_T}+\sum_{t=1}^T\frac{\eta_t\norm{\Lb}_1}{2}.
\end{align*}
This is where we address the key relation $\max\E[\cdot]\le\E[\max\cdot]$. In \cref{app:sec:proof-convex-upper-bound,app:sec:sco2oco}, the $V_t$ term simply equals to $0$ in expectation, since $\inner{\g_t-\Gb_t\odot\SS_{\Gb_t}(\g_t)}{\x_{t}-\x}$ is a martingale difference sequence for any \emph{fixed} $\x\in\X$. Obviously, this is not the case here, as $\xt_T$ depends on the entire trajectory of the algorithm. 
To decouple this dependency elegantly, we borrow the ``ghost-iterate'' technique~\citep{nemirovski2009robust,NeurIPS:2023:Zhang,yu24egdro,bai2025group,zhang2026gdro} to resolve this implicit non-oblivious online learning problem~\citep{cesa2006prediction}. 
Define $\e_t:=\g_t-\Gb_t\odot\SS_{\Gb_t}(\g_t)$. Suppose there is an \emph{imaginary} online StoSignSGD algorithm which receives $\e_t$ as subgradients:
\begin{align}\label{eq:ghost-iterate}
    \y_{t+1}=\Pi_\X^{\EB_t}\sqbrac{\y_t-\gamma_t\cdot\SS_{\EB_t}(\e_t)},\text{ where }\y_1=\x_1, \EB_t=(\norm{\e_{1:t,1}}_\infty,\cdots,\norm{\e_{1:t,d}}_\infty)^\top,
\end{align}
where $\gamma_t>0$ is a step size to be determined later. For the update defined above, we can invoke \cref{thm:oco-bound-pseudo} to obtain
\begin{align*}
    \E\sqbrac{\sum_{t=1}^T\inner{\e_t}{\y_t-\xt_T}}\overset{\mathtt{(a)}}{\le} \brac{\frac{D_\infty^2}{2\gamma_T}+\sum_{t=1}^T\frac{\gamma_t}{2}}\E\sqbrac{\norm{\EB_T}_1}\overset{\mathtt{(b)}}{\le}2D_\infty\norm{\Lb}_1\sqrt{T}.
\end{align*}
Note that step $\mathtt{(a)}$ is valid since $\xt_T$ does not depend on the trajectory of~\eqref{eq:ghost-iterate}. While $\mathtt{(b)}$ is due to
\begin{align*}
    \E\sqbrac{\norm{\EB_T}_1}=\E\sqbrac{\sum_{i=1}^d\norm{\e_{1:t,i}}_\infty}\le\E\sqbrac{\sum_{i=1}^d\brac{\norm{\g_{1:t,i}}_\infty+\Gb_{t,i}}}=2\E\sqbrac{\norm{\Gb_t}_1}\le2\norm{\Lb}_1,
\end{align*}
and the choice $\gamma_t\equiv D_\infty/\sqrt{T}$. We proceed to analyze the original expected regret:
\begin{align*}
    \E\sqbrac{\sum_{t=1}^TV_t}=\E\sqbrac{\sum_{t=1}^T\underbrace{\inner{\e_t}{\x_t-\y_t}}_{:=M_t}}+\E\sqbrac{\sum_{t=1}^T\inner{\e_t}{\y_t-\xt_T}}.
\end{align*}
Let's be precise about the information available at each step. The filtration $\F_{t-1}$ represents the history of all random events up to the beginning of round $t$. This includes the initial points $\x_1,\y_1$ and the sequences of gradients, iterates, and any other random variables from rounds $1,\cdots,t-1$. Thus, we know that $\x_t-\y_t$ is $\F_{t-1}$-measurable. So we deduce that $\E[M_t|\F_{t-1}]=\E\sqbrac{\inner{\E[\e_t|\F_{t-1}]}{\x_t-\y_t}}=0$, i.e., $\cbrac{M_t}_{t\in[T]}$ is a martingale difference sequence. This immediately implies that $\E[\sum_{t=1}^TM_t]=0$. Hence, we conclude that
\begin{align*}
    \E\sqbrac{\sum_{t=1}^TV_t}=\E\sqbrac{\sum_{t=1}^TM_t}+\E\sqbrac{\sum_{t=1}^T\inner{\e_t}{\y_t-\xt_T}}=\E\sqbrac{\sum_{t=1}^T\inner{\e_t}{\y_t-\xt_T}}\le2D_\infty\norm{\Lb}_1\sqrt{T}.
\end{align*}
Finally, we arrive at
\begin{align*}
    \E\sqbrac{\sum_{t=1}^T\inner{\g_t}{\x_{t}-\xt_T}}\le&\E\sqbrac{\sum_{t=1}^TV_t}+\frac{D_\infty^2\norm{\Lb}_1}{2\eta_T}+\sum_{t=1}^T\frac{\eta_t\norm{\Lb}_1}{2}\\\le&2D_\infty\norm{\Lb}_1\sqrt{T}+\norm{\Lb}_1\brac{\frac{D_\infty^2}{2\eta_T}+\sum_{t=1}^T\frac{\eta_t}{2}}.
\end{align*}
Plugging in $\eta_t=D_\infty/\sqrt{2t}$ and $\eta_t\equiv D_\infty/\sqrt{T}$ respectively completes the proof.
\end{proof}

\subsection{Proof of Theorem~\ref{thm:lower-bound}\label{app:sec:proof-lower-bound}}

A well-known technique in the online learning community is that running different online learning algorithms on all coordinates yields regret bounds w.r.t.~the $\ell_1$-norm of linear losses~\citep{McMahanS10adagrad,duchi2011adaptive}. The regret for this procedure equals to the sum of regrets of the online learning algorithms for each coordinate. We use this idea to present a lower bound. First, we present a general regret lower bound for online linear optimization (OLO).

\begin{lem}[Theorem~5.1 in~\citet{orabona2019intro}]\label{lem:olo-lower-bound}
    Let $\X\in\R^d$ be any non-empty closed convex set. Let $D = \sup_{\x,\x^\prime\in\X}\norm{\x-\x^\prime}_2$ be the diameter of $\X$. Let $\A$ be any (possibly randomized) algorithm for OLO on $\X$. Let $T$ be any non-negative integer. Then, there exists a sequence of vectors $\cbrac{\g_t}_{t\in[T]}$ with $\norm{\g_t}_2 \le L$ and $\x\in\X$ such that the regret of algorithm $\A$ satisfies
    \begin{align*}
        \textnormal{Regret}_T(\x)=\sum_{t=1}^T\inner{\g_t}{\x_t-\x}\ge\frac{LD\sqrt{2T}}{4}.
    \end{align*}
\end{lem}

Our high-level idea is to construct a separable function $f(\x):=\sum_{i=1}^df_i(x_i)$ where each $f_i:\R\mapsto\R$. This allows us to decompose the total regret of the algorithm into a sum of regrets incurred independently on each coordinate. We then apply \cref{lem:olo-lower-bound} to each of these terms and sum the results to obtain the final lower bound. Since the class of linear functions is a subset of convex functions, a lower bound for OLO is also a valid lower bound for Online Convex Optimization.

\begin{thm}[Regret Lower Bound]\label{thm:oco-lower-bound}
    Let $\X \subseteq \R^d$ be a convex set contained within a hypercube of side length $D_\infty$, i.e., $\sup_{\x, \x' \in \X} \norm{\x - \x'}_\infty \le D_\infty$. Let $\A$ be any (possibly randomized) online algorithm that generates iterates $\x_t \in \X$. For any $T \ge 1$ and any vector $\Lb = (L_1, \dots, L_d)^\top \in \R^d_+$, there exists a sequence of loss vectors $\{\g_t\}_{t \in [T]}$ with $|g_{t,i}| \le L_i$ for all $t, i$, and a competitor $\x \in \X$ such that the expected regret of algorithm $\A$ satisfies:
    \begin{align*}
        \textnormal{Regret}_T(\x)=\sum_{t=1}^T \inner{\g_t}{\x_t - \x} \ge \frac{D_\infty \norm{\Lb}_1\sqrt{2T}}{4} .
    \end{align*}
\end{thm}

\begin{proof}

Without loss of generality, let the domain be the hypercube $\X = [-D_\infty/2, D_\infty/2]^d$. For each coordinate $i \in [d]$, \cref{lem:olo-lower-bound} guarantees the existence of a ``hard'' sequence of one-dimensional loss functions. Specifically, for the one-dimensional domain $[-D_\infty/2, D_\infty/2]$ (which has diameter $D = D_\infty$) and a Lipschitz loss bound $L=L_i$, there exists a sequence of scalars $\{g_{t,i}\}_{t \in [T]}$ with $|g_{t,i}| \le L_i$ and a point $x_i^* \in [-D_\infty/2, D_\infty/2]$ such that any $1$-dimensional online algorithm incurs a regret of at least
\begin{equation}\label{eq:olo-coordinate-lower-bound}
    \textnormal{Regret}_{T,i}(x_i^*)=\sum_{t=1}^T g_{t,i} (x_{t,i} - x_i^*)\ge\frac{L_i D_\infty \sqrt{2T}}{4}
\end{equation}
Next, we construct our $d$-dimensional problem using these hard sequences. We define the sequence of loss vectors as $\g_t = (g_{t,1}, g_{t,2}, \dots, g_{t,d})^\top$ for all $t \in [T]$. Our competitor vector is simply the combination of the $1$-dimensional competitor points: $\x^* = (x_1^*, x_2^*, \dots, x_d^*)^\top \in \X$. The total regret of any algorithm $\A$ against our constructed competitor $\x^*$ can be written as:
\begin{align*}
    \textnormal{Regret}_T(\x^*) = &\sum_{t=1}^T \inner{\g_t}{\x_t - \x^*}=\sum_{t=1}^T \sum_{i=1}^d g_{t,i} (x_{t,i} - x_i^*)=\sum_{i=1}^d \sum_{t=1}^T g_{t,i} (x_{t,i} - x_i^*)\\=&\sum_{i=1}^d \textnormal{Regret}_{T,i}(x_i^*)\overset{\eqref{eq:olo-coordinate-lower-bound}}{\ge}\sum_{i=1}^d \frac{L_i D_\infty \sqrt{2T}}{4}= \frac{D_\infty \norm{\Lb}_1\sqrt{2T}}{4}.
\end{align*}
Therefore, we have finished the proof.
\end{proof}

Since this lower bound holds for a deterministic sequence of losses, it must also hold in expectation for any randomized algorithm. Also note that when $\abs{g_{t,i}}\le L_i$ holds, then \cref{ass:coord-lipschitz} also holds. Hence, \cref{thm:lower-bound} is a direct consequence of \cref{thm:oco-lower-bound}. This result confirms that the upper bound in \cref{thm:oco-bound} is tight in its dependence on $T$, $D_\infty$, and the $\ell_1$-norm of the coordinate-wise subgradient bounds, establishing the optimality of the algorithm up to constant factors.

\subsection{Proof of Theorem~\ref{thm:nonconvex-bound}\label{app:sec:exponential}}

By \cref{lem:random-taylor-expansion}, we have
\begin{align*}
    \E\sqbrac{f(\x_t)-f(\x_{t-1})}=&\E\sqbrac{\inner{\nabla f(\x_t)}{\bDelta_t}}\\=&\E\sqbrac{\inner{\g_t}{\bDelta_t}}+\E\sqbrac{\inner{\E_{\O_t}\sqbrac{\nabla f(\x_t)-\g_t}}{\bDelta_t}}=\E\sqbrac{\inner{\g_t}{\bDelta_t}}.
\end{align*}
Summing from $1$ to $T$ we obtain
\begin{align}\label{eq:grad+regret}
    \E\sqbrac{f(\x_T)-f(\x_0)}=\E\sqbrac{\sum_{t=1}^T\inner{\g_t}{\u_t}}+\E\sqbrac{\sum_{t=1}^T\inner{\g_t}{\bDelta_t-\u_t}}.
\end{align}
We proceed to bound the terms in~\eqref{eq:grad+regret} separately. Define
\begin{equation}
\begin{aligned}\label{eq:def-uk}
    &\forall k\in[K],\forall t:(k-1)N<t\le kN,\quad\uh^k:=\sum_{n=1}^N\nabla f(\x_{(k-1)N+n})\\&\u^k:=\u_{(k-1)N+1}=\cdots=\u_{kN}=-D_\infty\uh^k/\abs{\uh^k}.
\end{aligned}
\end{equation}
Then we write
\begin{align}
    &\E\sqbrac{\sum_{t=1}^T\inner{\g_t}{\u_t}}=\E\sqbrac{\sum_{k=1}^K\sum_{n=1}^N\inner{\g_{(k-1)N+n}}{\u^k}}\notag\\=&\sum_{k=1}^K\E\sqbrac{\inner{\sum_{n=1}^N\nabla f(\x_{(k-1)N+n})}{\u^k}}+\sum_{k=1}^K\E\sqbrac{\sum_{n=1}^N\inner{\underbrace{\nabla f(\x_{(k-1)N+n})-\g_{(k-1)N+n}}_{:=\eps_n^k}}{\u^k}}\notag\\\le&\sum_{k=1}^K\E\sqbrac{\inner{\uh^k}{\frac{-D_\infty\uh^k}{\abs{\uh^k}}}}+\sum_{k=1}^K\E\sqbrac{\norm{\sum_{n=1}^N\eps_n^k}_1\cdot\norm{\u^k}_\infty}\label{eq:bound-grad-inner-prod}\\\le&-\sum_{k=1}^KD_\infty\E\sqbrac{\norm{\uh^k}_1}+\sum_{k=1}^KD_\infty\sqrt{N}\sum_{i=1}^d\sigma_i\notag\\\le&-\sum_{k=1}^KD_\infty\E\sqbrac{\norm{\sum_{n=1}^N\nabla f(\x_{(k-1)N+n})}_1}+KD_\infty\sqrt{N}\norm{\Lb}_1\notag,
\end{align}
where the first inequality is due to Cauchy-Schwarz inequality; the second inequality follows from \cref{lem:coordinate-variance} and $\norm{\u^k}_\infty=D_\infty$; the last inequality holds under \cref{ass:coord-lipschitz}. Next, we control the linear regret term in~\eqref{eq:grad+regret}. Similarly, we first decompose it into the $K$-shifting regret~\citep{cutkosky2023optimal} of online stochastic sign gradient descent, and then apply the expected regret bound in \cref{thm:oco-bound} to obtain the final result\footnote{This is the place where one get to see further the differences between the bounds in \cref{app:sec:sco2oco,app:sec:pseudo2expected}. Since $\u^k$ depends on the whole history, we must invoke \cref{thm:oco-bound-expected} rather than \cref{thm:oco-bound-pseudo}.}.
\begin{equation}\label{eq:bound-regret-T}
    \begin{aligned}
        &\E\sqbrac{\sum_{t=1}^T\inner{\g_t}{\bDelta_t-\u_t}}=\sum_{k=1}^K\E\sqbrac{\sum_{n=1}^N\inner{\g_{(k-1)N+n}}{\bDelta_{(k-1)N+n}-\u^k}}\\\le&\sum_{k=1}^K\brac{2+\sqrt{2}}\cdot(2D_\infty)\norm{\Lb}_1\sqrt{N}=\brac{4+2\sqrt{2}}KD_\infty\norm{\Lb}_1\sqrt{N}.
    \end{aligned}
\end{equation}
Combining the bounds above, we obtain the following relation after rearranging:
\begin{equation}\label{eq:pre-bound}
    \begin{aligned}
        \E\sqbrac{\frac{1}{K}\sum_{k=1}^K\norm{\sum_{n=1}^N\frac{1}{N}\nabla f(\x_{(k-1)N+n})}_1}\le&\frac{f(\x_0)-\E[f(\x_T)]}{D_\infty KN}+\frac{\brac{5+2\sqrt{2}}\norm{\Lb}_1}{\sqrt{N}}\\\le&\frac{\Delta_f}{D_\infty KN}+\frac{\brac{5+2\sqrt{2}}\norm{\Lb}_1}{\sqrt{N}}.
    \end{aligned}
\end{equation}
Lastly, it suffices to connect the LHS of~\eqref{eq:pre-bound} to the stationary measure defined in \cref{def:stationary-point}. For any $k\in[K]$, we upper bound the quantity $\norm{\nabla f(\xb^k)}_{1,\infty}^{[\delta]}$ by the uniform distribution on the set $S_k=\cbrac{\x_{(k-1)N+1},\cdots,\x_{kN}}$, i.e., $\inf_{P\in\P(\R^d)}$ is relaxed by the realization $P=\unif(S_k)$. Note that \cref{alg:nonconvex-stosignsgd} calls \cref{alg:online-stosignsgd} to ensure that $\norm{\bDelta_t}_\infty\le D_\infty$. Hence, for any $\x_{n_1},\x_{n_2}\in S_k, (k-1)N+1\le n_1<n_2\le kN$, we have
\begin{equation}\label{eq:Sk-diam}
    \begin{aligned}
        \E\sqbrac{\norm{\x_{n_1}-\x_{n_2}}_\infty^2}\le&(n_2-n_1)\sum_{n=n_1+1}^{n_2}\E\sqbrac{\norm{\x_n-\x_{n-1}}_\infty^2}\\=&(n_2-n_1)\sum_{n=n_1+1}^{n_2}\E\sqbrac{\E_s[s_n^2]\norm{\bDelta_n}_\infty^2}\\\overset{\text{\cref{lem:exp-moments}}}{=}&(n_2-n_1)\sum_{n=n_1+1}^{n_2}\E\sqbrac{2\norm{\bDelta_n}_\infty^2}\le2(n_2-n_1)^2D_\infty^2,
    \end{aligned}
\end{equation}
where the first step utilizes Jensen's inequality. Based on the above derivations, we immediately obtain the following relation for any $\x_n\in S_k$:
\begin{equation}\label{eq:xb^k-center}
    \begin{aligned}
        \E\sqbrac{\norm{\x_n-\xb^k}_\infty^2}\le \frac{1}{N}\sum_{\x_{n^\prime}\in S_k}\E\sqbrac{\norm{\x_n-\x_{n^\prime}}_\infty^2}\overset{\eqref{eq:Sk-diam}}{\le}2N^2D_\infty^2.
    \end{aligned}
\end{equation}
Thus, we can compute the $(\delta,\epsilon)$-$\ell_{1,\infty}$-stationary point:
\begin{align*}
    \E\sqbrac{\norm{\nabla f(\xb^k)}_{1,\infty}^{[\delta]}}=&\E\sqbrac{\inf_{P\in\P(\R^d),\E_{\y\sim P}[\y]=\xb^k}\cbrac{\norm{\E\sqbrac{\nabla f(\y)}}_1+\delta\E\sqbrac{\norm{\y-\xb^k}_\infty^2}}}\\\le&\E\sqbrac{\norm{\E_{\y\sim\unif(S_k)}\sqbrac{\nabla f(\y)}}_1+\delta\E_{\y\sim\unif(S_k)}\sqbrac{\norm{\y-\xb^k}_\infty^2}}\\=&\E\sqbrac{\norm{\sum_{n=1}^N\frac{1}{N}\nabla f(\x_{(k-1)N+n})}_1}+\delta\E\sqbrac{\sum_{n=1}^N\frac{1}{N}\norm{\x_{(k-1)N+n}-\xb^k}_\infty^2}\\\overset{\eqref{eq:xb^k-center}}{\le}&\E\sqbrac{\norm{\sum_{n=1}^N\frac{1}{N}\nabla f(\x_{(k-1)N+n})}_1}+2\delta N^2D_\infty^2.
\end{align*}
Since $\xb\sim\unif\brac{\cbrac{\xb^1,\cdots,\xb^k}}$, we have
\begin{align*}
    \E\sqbrac{\norm{\nabla f(\xb)}_{1,\infty}^{[\delta]}}=&\frac{1}{K}\sum_{k=1}^K\E\sqbrac{\norm{\nabla f(\xb^k)}_{1,\infty}^{[\delta]}}\\\le&\frac{1}{K}\sum_{k=1}^K\E\sqbrac{\norm{\sum_{n=1}^N\frac{1}{N}\nabla f(\x_{(k-1)N+n})}_1}+2\delta N^2D_\infty^2\\\overset{\eqref{eq:pre-bound}}{\le}&\frac{\Delta_f}{D_\infty KN}+\frac{\brac{5+2\sqrt{2}}\norm{\Lb}_1}{\sqrt{N}}+2\delta N^2D_\infty^2\\\le& \frac{2}{7}\epsilon+\frac{4}{7}\epsilon+\frac{1}{7}\epsilon=\epsilon,
\end{align*}
where the last step follows from
\begin{align*}
    N=\ceil{\frac{49\brac{33+20\sqrt{2}}\norm{\Lb}_1^2}{16\epsilon^2}},\ K=\ceil{\frac{7\Delta_f\sqrt{14\delta}}{2\epsilon^{3/2}}},\ D_\infty=\frac{1}{N}\sqrt{\frac{\epsilon}{14\delta}}.
\end{align*}
Then the complexity can be bounded by
\begin{align*}
    T=NK\le4\cdot\frac{49\brac{33+20\sqrt{2}}\norm{\Lb}_1^2}{16\epsilon^2}\cdot\frac{7\Delta_f\sqrt{14\delta}}{2\epsilon^{3/2}}\le 39326\Delta_f\norm{\Lb}_1^2\delta^{1/2}\epsilon^{-7/2}.
\end{align*}

\subsection{Technical Lemmas}

\begin{lem}[Folklore]\label{lem:sqrt-sum}
$\forall T\ge1,\ \sum_{t=1}^T\frac{1}{\sqrt{t}}\le2\sqrt{T}-1.$
\end{lem}

\begin{lem}[Folklore]\label{lem:exp-moments}
The moments of $s\sim\expdist(\lambda),\lambda>0$ are given by $\E_s[s^n]=n!/\lambda^n,\forall n\in\N$.
\end{lem}

\begin{lem}[Lemma~8 in~\citet{hazan2007logarithmic}]\label{lem:projection}
    Let $\X\in\R^d$ be a convex set, $\x\in\R^d$ and $\y=\Pi_\X^\Ab[\x]$ be the generalized projection of $\x$ onto $\X$ according to a symmetric, positive semidefinite matrix $\Ab\succeq\0$. Then for any point $\z\in\X$ it holds that
    \begin{align*}
        \norm{\y-\z}_\Ab^2\le \norm{\x-\z}_\Ab^2.
    \end{align*}
\end{lem}

\begin{lem}[Lemma~3.2 in~\citet{zhang2024random}]\label{lem:random-taylor-expansion}
    Under \cref{ass:well-behavedness}, let $s\sim\expdist(\lambda)$ for some $\lambda>0$, then 
    \begin{align*}
        \E_s\sqbrac{f(\x+s\bDelta)-f(\x)}=\frac{1}{\lambda}\E_s\sqbrac{\inner{\nabla f(\x+s\bDelta)}{\bDelta}}.
    \end{align*}
\end{lem}

\begin{lem}[Coordinate-wise variance bound]\label{lem:coordinate-variance}
    Define $\eps_t:=\g_t-\nabla f(\x_t)$ with coordinate-wise finite variance: $\E_t[\epsilon_{t,i}^2]\le\sigma_i^2,\forall t\in[T],\forall i\in[d]$, then it holds that
    \begin{align*}
        \E\sqbrac{\norm{\sum_{t=1}^T\eps_t}_1}\le\sqrt{T}\sum_{i=1}^d\sigma_i.
    \end{align*}
\end{lem}

\begin{proof}
    Decompose the objective into the sum of all coordinates:
    \begin{align*}
         \E\sqbrac{\norm{\sum_{t=1}^T\eps_t}_1}=\sum_{i=1}^d\E\sqbrac{\abs{\sum_{t=1}^T\epsilon_{t,i}}}\le\sum_{i=1}^d\sqrt{\E\sqbrac{\sum_{t=1}^T\epsilon_{t,i}}^2}.
    \end{align*}
    Then we bound each coordinate using standard techniques:
    \begin{align*}
        \E\sqbrac{\sum_{t=1}^T\epsilon_{t,i}}^2=&\E\sqbrac{\sum_{t=1}^T\epsilon_{t,i}^2}+\E\sqbrac{\sum_{1\le s<t\le T}2\epsilon_{s,i}\epsilon_{t,i}} \\\le& \E\sqbrac{\sum_{t=1}^T\E_t\sqbrac{\epsilon_{t,i}^2}}+\E\sqbrac{\sum_{1\le s<t\le T}2\E_t\sqbrac{\epsilon_{s,i}\epsilon_{t,i}}}\le T\sigma_i^2.       
    \end{align*}
    Combining the above two relations finishes the proof.
\end{proof}

\section{StoSignSGD Also Tracks Goldstein Stationary Points\label{sec:goldstein}}

In \cref{sec:non-convex-theory}, we propose a non-smooth, non-convex version of StoSignSGD that finds $(\delta,\epsilon)$-$\ell_{p,q}$-stationary points in \cref{def:stationary-point} at state-of-the-art rates. In this section, we propose a variant of the stationary measure based on Goldstein stationarity~\citep{goldstein1977optimization} and present a variant of StoSignSGD that efficiently identifies it. The formal definition is given below.

\begin{defn}[$(\delta,\epsilon)$-$\ell_{p,q}$-Goldstein stationary point]\label{def:stationary-point-as}
Given $\epsilon,\delta> 0,1\le p,q\le\infty$ and an almost-everywhere differentiable function $f$, $\x$ is an $(\delta,\epsilon)$-$\ell_{p,q}$-Goldstein stationary point of $f$ if
    \begin{align*}
        \norm{\nabla f(\x)}_{p,q}^{[\delta]_{\textnormal{Gs}}}:=\inf_{S\subseteq B_q(\x,\delta),\frac{1}{\abs{S}}\sum_{\y\in S}\y=\x}\norm{\frac{1}{\abs{S}}\sum_{\y\in S}\nabla f(\y)}_p\le\epsilon.
    \end{align*}
\end{defn}

\citet[Definition~12]{cutkosky2023optimal} instantiates both $p=q=2$ and $p=1,q=\infty$ in \cref{def:stationary-point-as}. Analogously to \cref{def:stationary-point}, \cref{def:stationary-point-as} is more general and capable of fitting diverse problem geometries compared to the Euclidean notion in previous literature~\citep[Definition~12]{cutkosky2023optimal}. We also discuss the relationship between \cref{def:stationary-point,def:stationary-point-as} in \cref{sec:stationary-points}, following the similar recipe as~\citet{zhang2026random}.

Next, we introduce a variant of \cref{alg:nonconvex-stosignsgd} that replaces the exponential random scaling by uniform random scaling, and further queries stochastic gradients on the random line segment between $\x_t$ and $\x_{t-1}$. The complete procedure is outlined in \cref{alg:nonconvex-stosignsgd-uniform}. The above alterations stem from the original online-to-nonconvex conversion framework in~\citet[Algorithm~1]{cutkosky2023optimal}. \cref{alg:nonconvex-stosignsgd-uniform} can be viewed as replacing their online gradient descent~\citep{zinkevich2003oco} base algorithm with online StoSignSGD in \cref{alg:online-stosignsgd}. The theoretical guarantee of \cref{alg:nonconvex-stosignsgd-uniform} is given below, whose proof can be found in \cref{sec:proof-goldstein}.

\begin{thm}\label{thm:nonconvex-bound-unif}
    Under \cref{ass:coord-lipschitz,ass:non-convexity,ass:well-behavedness} and further assume that $f(\x_0)-f^*\le\Delta_f$. Set
    \begin{align*}
    K=\ceil{\frac{3\Delta_f}{\epsilon\delta}},\ N=\ceil{\frac{81\norm{\Lb}_1^2}{4\epsilon^{2}}},\  D_\infty=\frac{\delta}{N},\ \eta_t=\frac{2D_\infty}{\sqrt{N}}.
\end{align*}
Then, \cref{alg:nonconvex-stosignsgd-uniform} finds $(\delta,\epsilon)$-$\ell_{1,\infty}$-Goldstein stationary points within
\begin{align*}
    T=KN\le1323\Delta_f\norm{\Lb}_1^2\delta^{-1}\epsilon^{-3}
\end{align*}
stochastic gradient evaluations.
\end{thm}

The complexity bound matches~\citet[Theorem~13]{cutkosky2023optimal}. Following the discussions in~\citet[\S~4]{cutkosky2023optimal}, the result in \cref{thm:nonconvex-bound-unif} can be better than the naive $\ell_{2,2}$ bound based on SGD.

\begin{algorithm}[t]
   \caption{Non-smooth Non-convex StoSignSGD with Uniform Random Scaling}
   \label{alg:nonconvex-stosignsgd-uniform}
\begin{algorithmic}[1]
   \STATE {\bfseries Input:} Initialization $\x_0\in\R^d$, $K,N\in\N,D_\infty>0$, $\bDelta_0\in B_\infty(\0,D_\infty)\subseteq\R^d$, $\Gb_0=\0\in\R^d$.\\
   \STATE Set $T=KN$ and $\eta_0=0$.
   \FOR{$t=1$ {\bfseries to} $T$}
   \STATE Compute $\bDelta_t=\Pi_{B_\infty(\0,D_\infty)}^{\Gb_{t-1}}\sqbrac{\Delta_{t-1}-\eta_{t-1}\cdot\SS_{\Gb_{t-1}}(\g_{t-1})}$.
   \STATE Set $\x_t=\x_{t-1}+\bDelta_t$.
   \STATE Sample $s_t\sim\unif([0,1])$.   
   \STATE Set $\w_t=\x_{t-1}+s_t\bDelta_t$.
   \STATE Generate stochastic gradient $\g_t$ from oracle $\O$ such that $\E_{\O_t}[\g_t]=\nabla f(\w_t)$.
   \STATE Update $\Gb_t=\max\cbrac{\Gb_{t-1},\abs{\g_t}}$.
   \ENDFOR
   \STATE Set $\wb^k=\frac{1}{N}\sum_{n=1}^N\w_{(k-1)N+n}$ for all $k\in[K]$.
   \STATE {\bfseries Output:} Return $\wb\sim\unif\brac{\cbrac{\wb^1,\cdots,\wb^K}}$.
\end{algorithmic}
\end{algorithm}

\subsection{Proof of Theorem~\ref{thm:nonconvex-bound-unif}\label{sec:proof-goldstein}}

The analysis for \cref{alg:nonconvex-stosignsgd-uniform} follows from the online-to-nonconvex conversion framework introduced by~\citet{cutkosky2023optimal}. First, we construct a telescoping sum of function values analogously to~\citet[Theorem~7]{cutkosky2023optimal}. Under \cref{ass:well-behavedness}, we have
\begin{align*}
    \E\sqbrac{f(\x_t)-f(\x_{t-1})}=&\E\sqbrac{\int_0^1\inner{\nabla f(\x_{t-1}+s(\x_t-\x_{t-1}))}{\x_t-\x_{t-1}}\diff{s}}\\=&\E\sqbrac{\inner{\int_0^1\nabla f(\x_{t-1}+s\bDelta_t)\diff{s}}{\bDelta_t}}\\=&\E\sqbrac{\inner{\g_t}{\bDelta_t}}+\E\sqbrac{\inner{\int_0^1\nabla f(\x_{t-1}+s\bDelta_t)\diff{s}-\g_t}{\bDelta_t}}\\=&\E\sqbrac{\inner{\g_t}{\bDelta_t}}+\E\sqbrac{\inner{\E_{s_t,\O_t}[\g_t]-\int_0^1\nabla f(\x_{t-1}+s\bDelta_t)\diff{s}}{-\bDelta_t}}\\=&\E\sqbrac{\inner{\g_t}{\u_t}}+\E\sqbrac{\inner{\g_t}{\bDelta_t-\u_t}}\\&+\E\sqbrac{\inner{\E_{s_t}[\nabla f(\x_{t-1}+s_t\bDelta_t)]-\int_0^1\nabla f(\x_{t-1}+s\bDelta_t)\diff{s}}{-\bDelta_t}}\\=&\E\sqbrac{\inner{\g_t}{\u_t}}+\E\sqbrac{\inner{\g_t}{\bDelta_t-\u_t}},
\end{align*}
where the last step necessitates the randomized scaling trick~\citep{zhang2020complexity}, i.e., evaluating stochastic gradients at a random point on the line segment between adjacent iterates. Similar to~\eqref{eq:grad+regret}, we have
\begin{align*}
    -\Delta_f\le\E\sqbrac{f(\x_T)-f(\x_0)}=\E\sqbrac{\sum_{t=1}^T\inner{\g_t}{\u_t}}+\E\sqbrac{\sum_{t=1}^T\inner{\g_t}{\bDelta_t-\u_t}}.
\end{align*}
Different from~\eqref{eq:def-uk}, we define
\begin{align}
    &\forall k\in[K],\forall t:(k-1)N<t\le kN,\quad\uh^k:=\sum_{n=1}^N\nabla f(\w_{(k-1)N+n})\\&\u^k:=\u_{(k-1)N+1}=\cdots=\u_{kN}=-D_\infty\uh^k/\abs{\uh^k}.
\end{align}
Then we perform the nearly identical procedure as in~\eqref{eq:bound-grad-inner-prod}:
\begin{align*}
    &\E\sqbrac{\sum_{t=1}^T\inner{\g_t}{\u_t}}=\E\sqbrac{\sum_{k=1}^K\sum_{n=1}^N\inner{\g_{(k-1)N+n}}{\u^k}}\\=&\sum_{k=1}^K\E\sqbrac{\inner{\sum_{n=1}^N\nabla f(\w_{(k-1)N+n})}{\u^k}}+\sum_{k=1}^K\E\sqbrac{\sum_{n=1}^N\inner{\underbrace{\nabla f(\w_{(k-1)N+n})-\g_{(k-1)N+n}}_{:=\eps_n^k}}{\u^k}}\\\le&\sum_{k=1}^K\E\sqbrac{\inner{\uh^k}{\frac{-D_\infty\uh^k}{\abs{\uh^k}}}}+\sum_{k=1}^K\E\sqbrac{\norm{\sum_{n=1}^N\eps_n^k}_1\cdot\norm{\u^k}_\infty}\\\le&-\sum_{k=1}^KD_\infty\E\sqbrac{\norm{\uh^k}_1}+\sum_{k=1}^KD_\infty\sqrt{N}\sum_{i=1}^d\sigma_i\\\le&-\sum_{k=1}^KD_\infty\E\sqbrac{\norm{\sum_{n=1}^N\nabla f(\w_{(k-1)N+n})}_1}+KD_\infty\sqrt{N}\norm{\Lb}_1,
\end{align*}
The regret bound in~\eqref{eq:bound-regret-T} remains almost intact:
\begin{align*}
    \E\sqbrac{\sum_{t=1}^T\inner{\g_t}{\bDelta_t-\u_t}}\le\sum_{k=1}^K3\cdot(2D_\infty)\norm{\Lb}_1\sqrt{N}=6K\sqrt{N}D_\infty\norm{\Lb}_1.
\end{align*}
The only difference lies in the constant $3$, which improves upon the $2+\sqrt{2}$ used in \cref{alg:nonconvex-stosignsgd-uniform}. That's simply because we leverages a constant learning rate of $\eta_t=2D_\infty/\sqrt{N}$.
Analogously to~\eqref{eq:pre-bound}, we obtain
\begin{align}\label{eq:pre-bound-as}
    \E\sqbrac{\frac{1}{K}\sum_{k=1}^K\norm{\sum_{n=1}^N\frac{1}{N}\nabla f(\w_{(k-1)N+n})}_1}\le\frac{\Delta_f}{D_\infty KN}+\frac{7\norm{\Lb}_1}{\sqrt{N}}.
\end{align}
Similar to \cref{app:sec:exponential}, the remaining procedure is to find a realization of the stationary measure defined in \cref{def:stationary-point-as}. For any $k\in[K]$, we upper bound the quantity $\norm{\nabla f(\wb^k)}_{1,\infty}^{[\delta]_{\textrm{Gs}}}$ by realizing the uniform distribution supported on the set $S_k=\cbrac{\w_{(k-1)N+1},\cdots,\w_{kN}}$, i.e., $\inf_{S\subseteq B_\infty(\wb^k,\delta),\frac{1}{\abs{S}}\sum_{\y\in S}\y=\wb^k}$ is relaxed by the realization $S=S_k$. Thus, it's necessary to show that $S_k\subseteq B_\infty(\wb^k,\delta)$. For any $\w_{n_1},\w_{n_2}\in S_k, (k-1)N+1\le n_1<n_2\le kN$, we have
\begin{equation}\label{eq:Sk-diam-as}
    \begin{aligned}
        \norm{\w_{n_1}-\w_{n_2}}_\infty\le&\norm{\w_{n_1}-\x_{n_1}}_\infty+\norm{\x_{n_1}-\x_{n_2-1}}_\infty+\norm{\x_{n_2}-\w_{n_2-1}}_\infty\\
        \le&\norm{(s_{n_1}-1)\bDelta_{n_1}}_\infty+\sum_{n=n_1+1}^{n_2-1}\norm{\x_n-\x_{n-1}}_\infty+\norm{-s_{n_2}\bDelta_{n_2}}_\infty\\\le&\norm{\bDelta_{n_1}}_\infty+\sum_{n=n_1+1}^{n_2-1}\norm{\bDelta_n}_\infty+\norm{\bDelta_{n_2}}_\infty\\\le&D_\infty+(n_2-1-(n_1+1)+1)D_\infty+D_\infty\le D_\infty N,
    \end{aligned}
\end{equation}
where the third inequality is due to $s_n\sim\unif([0,1])$ and the last step follows from $\norm{\bDelta_n}_\infty\le D_\infty$. Then, in the same view of~\eqref{eq:xb^k-center}, we conclude that
\begin{align*}
    \forall\w_n\in S_k,\norm{\w_n-\wb^k}_\infty\le\frac{1}{N}\sum_{\w_{n^\prime}\in S_k}\norm{\w_n-\w_{n^\prime}}_\infty\overset{\eqref{eq:Sk-diam-as}}{\le}ND_\infty\Longrightarrow S_k\subseteq B_\infty(\wb^k,ND_\infty).
\end{align*}
By choosing $D_\infty=\delta/N$, we derive that
\begin{align*}
        \norm{\nabla f(\wb^k)}_{1,\infty}^{[\delta]_{\textrm{Gs}}}=\inf_{S\subseteq B_\infty(\wb^k,\delta),\frac{1}{\abs{S}}\sum_{\y\in S}\y=\wb^k}\norm{\frac{1}{\abs{S}}\sum_{\y\in S}\nabla f(\y)}_1\le\norm{\frac{1}{N}\sum_{n=1}^N\nabla f(\w_{(k-1)N+n})}_1.
\end{align*}
Since $\wb\sim\unif\brac{\cbrac{\wb^1,\cdots,\wb^k}}$, we have
\begin{align*}
    \E\sqbrac{\norm{\nabla f(\wb)}_{1,\infty}^{[\delta]_{\textrm{Gs}}}}=&\frac{1}{K}\sum_{k=1}^K\E\sqbrac{\norm{\nabla f(\wb^k)}_{1,\infty}^{[\delta]_{\textrm{Gs}}}}\\\le&\frac{1}{K}\sum_{k=1}^K\E\sqbrac{\norm{\sum_{n=1}^N\frac{1}{N}\nabla f(\x_{(k-1)N+n})}_1}\\\overset{\eqref{eq:pre-bound-as}}{\le}&\frac{\Delta_f}{\delta K}+\frac{7\norm{\Lb}_1}{\sqrt{N}}\\\le& \frac{1}{3}\epsilon+\frac{2}{3}\epsilon=\epsilon,
\end{align*}
where the last step follows from
\begin{align*}
    N=\ceil{\frac{441}{4}\norm{\Lb}_1^2\epsilon^{-2}},\ K=\ceil{\frac{3\Delta_f}{\epsilon\delta}}.
\end{align*}

\subsection{Conversion Between Stationary Points\label{sec:stationary-points}}

To better clarify the relationship between different stationary measures, we introduce the following slightly stronger version of \cref{def:stationary-point-as}.

\begin{defn}[$(\delta,\epsilon)$-$\ell_{p,q}$-Goldstein stationary point (Version 2)]\label{def:stationary-point-as-v2}
Given $\epsilon,\delta>0$, $1\le p,q\le\infty$, and an almost-everywhere differentiable function $f$, $\x$ is an $(\delta,\epsilon)$-$\ell_{p,q}$-Goldstein stationary point (Version 2) of $f$ if
\begin{align*}
    \norm{\nabla f(\x)}_{p,q}^{[\delta]_{\textnormal{Gs-v2}}}
    :=
    \inf_{S\subseteq B_q(\x,\delta),\,P\in\P(S),\,\E_{\y\sim P}\sqbrac{\y}=\x}
    \norm{\E_{\y\sim P}\sqbrac{\nabla f(\y)}}_p
    \le \epsilon.
\end{align*}
\end{defn}

Compared with \cref{def:stationary-point-as}, \cref{def:stationary-point-as-v2} relaxes the uniformly weighted finite set in \cref{def:stationary-point-as} to an arbitrary distribution $P\in\P(S)$. Therefore, \cref{def:stationary-point-as-v2} is a stronger notion in the sense that its defining infimum is taken over a larger class of feasible objects, and hence can only be smaller. 

The reason we use \cref{def:stationary-point-as}, rather than \cref{def:stationary-point-as-v2}, in \cref{alg:nonconvex-stosignsgd-uniform,thm:nonconvex-bound-unif} is that the analysis in \cref{sec:proof-goldstein} only constructs uniformly distributed random variables, namely distributions that are finite and uniformly weighted on their supports. Thus, the proof of \cref{thm:nonconvex-bound-unif} already suffices to establish the weaker stationarity notion in \cref{def:stationary-point-as}.

The purpose of introducing \cref{def:stationary-point-as-v2} is to bridge \cref{def:stationary-point} and \cref{def:stationary-point-as}. The next lemma shows that, under an additional Lipschitz assumption, \cref{def:stationary-point} reduces to a nearby instance of \cref{def:stationary-point-as-v2}.

\begin{lem}[Reduction of \cref{def:stationary-point} to \cref{def:stationary-point-as-v2}, adapted from Lemma~2.4 in~\citet{zhang2026random}]\label{lem:reduction}
Fix $1\le p,q,r\le\infty$, and let $r_*$ be the dual norm of $r$. Assume that $f$ is differentiable and $G$-Lipschitz with respect to $\norm{\cdot}_r$, i.e., $\abs{f(\x)-f(\y)}\le G\norm{\x-\y}_r$ for all $\x,\y\in\R^d$. If $\norm{\nabla f(\x)}_{p,q}^{[\lambda]}\le\epsilon$, then for every $\rho>0$, there exists a point $\xt\in\R^d$ such that
\begin{align*}
    \norm{\xt-\x}_q\le \rho
    \quad\text{and}\quad
    \norm{\nabla f(\xt)}_{p,q}^{[2\rho]_{\textnormal{Gs-v2}}}
    \le
    \epsilon\brac{1+\frac{2C_{p,r_*}G}{\lambda\rho^2}},
\end{align*}
where $C_{p,r_*}:=\sup_{\z\neq\0}\norm{\z}_p/\norm{\z}_{r_*}$.
\end{lem}

\begin{proof}
Fix $\varepsilon>0$. By $\norm{\nabla f(\x)}_{p,q}^{[\lambda]}\le\epsilon$, there exists $P\in\P(\R^d)$ with $\E_{\y\sim P}[\y]=\x$ such that
\begin{align}
    \norm{\E\sqbrac{\nabla f(\y)}}_p+\lambda\E\sqbrac{\norm{\y-\x}_q^2}\le \epsilon+\varepsilon. \label{eq:lipschitz-reduction-start}
\end{align}
Define
\begin{align*}
    \wt:=\x+\min\cbrac{1,\frac{\rho}{\norm{\y-\x}_q}}(\y-\x),
    \qquad
    \xt:=\E[\wt],
\end{align*}
and let $\widetilde P\in\P(\R^d)$ denote the distribution of $\wt$. Since $\norm{\wt-\x}_q\le \rho$ almost surely, Jensen's inequality gives $\norm{\xt-\x}_q\le \rho$. Also,
\begin{align*}
    \norm{\wt-\xt}_q\le \norm{\wt-\x}_q+\norm{\x-\xt}_q\le 2\rho,
\end{align*}
hence $\supp(\widetilde P)\subseteq B_q(\xt,2\rho)$ and $\E_{\y\sim\widetilde P}[\y]=\xt$.

Now let $A:=\cbrac{\norm{\y-\x}_q>\rho}$. On $A^c$ we have $\wt=\y$, so
\begin{align*}
    \norm{\E\sqbrac{\nabla f(\wt)}}_p
    &\le \norm{\E\sqbrac{\nabla f(\y)}}_p+\E\sqbrac{\norm{\nabla f(\wt)-\nabla f(\y)}_p\indicator{A}}.
\end{align*}
Since $f$ is $G$-Lipschitz with respect to $\norm{\cdot}_r$, we have $\norm{\nabla f(\z)}_{r_*}\le G$ almost everywhere, and therefore
\begin{align*}
    \norm{\nabla f(\wt)-\nabla f(\y)}_p
    \le \norm{\nabla f(\wt)}_p+\norm{\nabla f(\y)}_p
    \le 2C_{p,r_*}G
\end{align*}
almost everywhere. Combining this with \eqref{eq:lipschitz-reduction-start} and Markov's inequality,
\begin{align*}
    \norm{\E\sqbrac{\nabla f(\wt)}}_p
    &\le \epsilon+\varepsilon+2C_{p,r_*}G\Pr(A) \\
    &\le \epsilon+\varepsilon+2C_{p,r_*}G\frac{\E\sqbrac{\norm{\y-\x}_q^2}}{\rho^2} \\
    &\le (\epsilon+\varepsilon)\brac{1+\frac{2C_{p,r_*}G}{\lambda\rho^2}}.
\end{align*}
Since $\widetilde P$ is feasible in \cref{def:stationary-point-as-v2}, we obtain
\begin{align*}
    \norm{\nabla f(\xt)}_{p,q}^{[2\rho]_{\textnormal{Gs-v2}}}
    \le (\epsilon+\varepsilon)\brac{1+\frac{2C_{p,r_*}G}{\lambda\rho^2}}.
\end{align*}
Letting $\varepsilon\downarrow0$ completes the proof.
\end{proof}

As an immediate consequence, we obtain the following specialization in the geometry used throughout this paper, namely the $\ell_{1,\infty}$ setting. By choosing $\rho=\delta/2$ in \cref{lem:reduction}, the support radius $2\rho$ becomes exactly $\delta$, which matches the radius parameter in \cref{def:stationary-point-as-v2}.

\begin{cor}\label{cor:reduction-l1-linf-delta}
Assume that $f$ is differentiable and $G$-Lipschitz with respect to $\norm{\cdot}_\infty$. If $\norm{\nabla f(\x)}_{1,\infty}^{[\lambda]}\le\epsilon$, then for every $\delta>0$, there exists a point $\xt\in\R^d$ such that
\begin{align*}
    \norm{\xt-\x}_\infty\le \delta/2
    \quad\text{and}\quad
    \norm{\nabla f(\xt)}_{1,\infty}^{[\delta]_{\textnormal{Gs-v2}}}
    \le
    \epsilon\brac{1+\frac{8G}{\lambda\delta^2}}.
\end{align*}
\end{cor}

\begin{proof}
Apply \cref{lem:reduction} with $p=1$, $q=\infty$, and $r=\infty$. Then $r_*=1$ and $C_{1,1}=1$. Setting $\rho=\delta/2$ completes the proof.
\end{proof}

\end{document}